\documentclass{article}
\usepackage{ijcai26}
\makeatletter
\ifx\latex@outputpage\@undefined\else
  \let\@outputpage\latex@outputpage
\fi
\let\citep\cite
\newcommand{\citet}[1]{\citeauthor{#1}~\shortcite{#1}}
\makeatother

\usepackage{times}
\usepackage{soul}
\usepackage{url}
\usepackage[hidelinks]{hyperref}
\usepackage[utf8]{inputenc}
\usepackage[small]{caption}
\usepackage{graphicx}
\graphicspath{{./}{../}{../../}}
\usepackage{amsmath}
\usepackage{amsthm}
\usepackage{booktabs}
\usepackage{algorithm}
\usepackage{algpseudocode}
\usepackage[switch]{lineno}

\usepackage{mathtools}
\usepackage{dsfont} 
\usepackage{xspace} 
\usepackage{subcaption} 
\usepackage{caption}
\usepackage{xcolor}
\usepackage{framed}

\definecolor{shadecolor}{RGB}{248,248,248}
\usepackage{amsfonts}
\usepackage{proof-at-the-end}
\usepackage{thmtools,thm-restate}
\usepackage{aligned-overset} 
\usepackage{enumitem}
\setlist[enumerate]{wide=-3pt, widest=99,leftmargin=\parindent, labelsep=*} 

\theoremstyle{plain} 
\newtheorem{theorem}{Theorem}
\newtheorem*{theorem*}{Theorem}
\newtheorem{lemma}[theorem]{Lemma}

\theoremstyle{definition} 

\theoremstyle{remark} 
 
\newcounter{casecount}
\AtBeginEnvironment{proof}{\setcounter{casecount}{0}}

\usepackage{romannum}

\renewcommand{\epsilon}{\varepsilon}

\newcommand{\F}{\ensuremath{\mathcal{F}}\xspace}

\newcommand{\E}[1]{\ensuremath{\mathrm{E}\mathord{\left(#1\right)}}}

\usepackage{graphicx}

\usepackage{booktabs}
\newcommand{\best}[1]{{\bfseries\boldmath $#1$}}     
\newcommand{\runner}[1]{\underline{$#1$}}            

\newcommand{\appendixtableofcontents}{
  \begingroup
  \let\clearpage\relax
  \let\cleardoublepage\relax
  \let\cleardoublepage\relax
  \tableofcontents
  \endgroup
}

\usepackage{amssymb} 

\newcommand{\St}{\mathcal{S}}                   
\newcommand{\A}{\mathcal{A}}                    
\newcommand{\cP}{\mathcal{P}}                   
\newcommand{\RE}{\mathcal{R}}                   
\newcommand{\Epi}[1]{\mathbb{E}_{\pi}\left[#1\right]} 

\newcommand{\Input}{\item[\textbf{Input:}]}
\newcommand{\Output}{\item[\textbf{Output:}]}
\usepackage{framed} 

\providecommand{\pdfinfo}[1]{}
\title{Provably Sub-Linear Two-Timescale NeuroEvolution with Online Plasticity}

\author{
{\fontsize{12}{14}\selectfont
\textbf{Shishen Lin$^{1,*}$} {\normalfont and} \textbf{Yixin Chen$^{2,}$\footnote{denotes the equal contribution.}}
}
\\
\affiliations
$^1$Department of Statistics, University of Warwick, Coventry, United Kingdom\\
$^2$Biozentrum, University of Basel, Basel, Switzerland
\\
\emails
shishen.lin.1@warwick.ac.uk,
chenyx1119@gmail.com
}

\begin{document}

\maketitle

\begin{abstract}
    NeuroEvolution of Augmenting Topologies (NEAT) is a widely used neuroevolution algorithm for learning neural network architectures and weights for control tasks.
    However, standard offline optimisation searches for connection strengths directly, which can scale poorly in high-dimensional weight spaces and more difficult continuous control problems.
    Hybrid methods that combine neuroevolution with online learning can address this challenge, but their theoretical properties remain underexplored.
    This paper gives the first regret analysis for a general NeuroEvolutionary Online Learning (NEOL) framework, which decouples learning into two timescales: an outer loop for architecture search and an inner loop for online weight adaptation via reward-modulated plasticity.
    Under mild conditions, we prove that NEOL achieves sublinear regret.
    Empirically, under fixed interaction budgets on four standard control benchmarks, a NEAT-based NEOL implementation achieves higher final fitness and lower variance than pure NEAT, and is competitive with strong reinforcement learning (RL) baselines on several tasks.
    The results are supported by Wilcoxon rank-sum tests and ablation studies.
    Overall, the findings show that online plasticity can improve the sample efficiency and robustness of two-timescale neuroevolution.
    Code is available at \url{https://github.com/boobaa2001/NeuroEvolution_Online_Learning_NEOL}.
    
\end{abstract}

\section{Introduction}

NeuroEvolution (NE) employs evolutionary operators, rather than gradient descent, to optimise neural network architectures and parameters~\citep{miikkulainen2025neuroevolution,miikkulainen2024evolving,stanley2019designing,khan2010evolution,yao1999evolving,angeline1994evolutionary}. 
It has been successfully applied to biologically inspired models for lifelong learning~\citep{kudithipudi2022biological}, reinforcement learning (RL) tasks~\citep{xue2024sample,chalumeau2023neuroevolution,coreyes2021evolving,khadka2018evolution,stanley2002efficient}, and domain-specific challenges (i.e., optimisation of land-use planning policies for carbon reduction~\citep{young2025discovering}). 

A key limitation of pure neuroevolution is that weight optimisation in high-dimensional spaces relies on mutation-based perturbations. 
These often provide weak fitness assignment, leading to unstable optimisation dynamics and premature convergence in complex continuous-control problems~\citep{stanley2009hypercube}. 
Furthermore, many neuroevolutionary algorithms are offline: a population is evolved on a task, and the policy is fixed at deployment. 
This approach is often inadequate in settings requiring real-time adaptation under sequential interaction and non-stationary environments~\citep{agogino2000online,bellman1966dynamic,sutton1998reinforcement}.
These challenges motivate hybrid methods that retain evolutionary local search for structure and network, while introducing direct mechanisms for online weight adaptation~\citep{agogino2000online,Stanley2003AdaptiveSynapses,Peng2018neatpg}.

Combining population-based search with lifelong learning is a foundational goal for adaptive intelligence~\citep{schmidhuber1987evolutionary,holland1992adaptation,miikkulainen2025neuroevolution}. 
While early attempts to evolve synaptic plasticity rules directly struggled as genetic search spaces for learning rules expanded~\citep{agogino2000online,Stanley2003AdaptiveSynapses}, the introduction of neuromodulation, using Hebbian updates with reward signals, has enabled networks to solve dynamic control tasks more efficiently than were previously intractable for both fixed-weight and non-modulated plastic architectures \citep{Soltoggio2008NeuromodulatedPlasticity,Soltoggio2018EPANN,Najarro2020HebbianMetaLearning}.
This established a robust two-timescale loop where an outer evolutionary process designs an inner online learner. 

Despite these empirical successes, theoretical analysis of neuroevolution is limited. 
Existing studies focus entirely on the runtime analysis of evolutionary neural architecture search for discrete optimisation tasks~\citep{fischer2023first,lv2024nas,lv2025nas}, building an important block for a rigorous understanding of neuroevolution for discrete search spaces. 
Although they are an important first step, these theoretical analyses cannot directly be applied to neuroevolution with online learning, especially in a continuous search space.
To the best of our knowledge, a formal regret analysis of neuroevolution or its online variants remains unexplored. 
This paper addresses this gap by introducing a general NeuroEvolutionary Online Learning (NEOL) framework. 
As a reasonable starting point, we employ an exponential weight update mechanism, in which the algorithm selects the architecture via softmax selection (we will explain further in Section 3, A6), using a simplified model for the selection process rather than more complicated selection algorithms. 
We do not model the full nonconvex MDP dynamics of MuJoCo (Multi-Joint dynamics with Contact)/Box2D (2D game-like RL environments) or other control tasks, as NEOLs (e.g., NEAT) might be highly intractable due to their design. 
Instead, we isolate the algorithmic roles of:
(i)  outer search over network topology, and
(ii) inner online weight adaptation from interaction feedback via local plasticity (NEOL-style reward-modulated, including Hebb/Oja/BCM updates),
Therefore, this paper aims to answer:

\begin{enumerate}
    \item[(1)] Can we provide a regret guarantee for NeuroEvolutionary Online Learning (NEOL) methods?
    \item[(2)] How can we effectively decompose policy learning into structural evolution and online weight adaptation?
\end{enumerate}

\subsection{Contribution}
\noindent Our contributions are primarily in three folds.
First, we introduce a general NEOL framework and provide a first regret analysis of this framework, showing that under mild conditions and exponential-weights selection over neural architectures, NEOL with reward-modulated plasticity rules (BCM, Hebb, and Oja) achieves sublinear regret $O(\sqrt{T})$, i.e., implying asymptotically optimal average performance.
Second, we propose NEAT-NEOL, a practical implementation that decouples topology (network architecture) evolution and online weight adaptation via plasticity rules.
Third, we conduct extensive experiments on four standard control benchmarks (CartPole, Lunar Lander, Hopper, and Bipedal Walker) under fixed interaction budgets.
We also compare NEAT-NEOL with standard NEAT and strong RL baselines (PPO and SAC), and validate the results through Wilcoxon rank-sum tests and ablation studies.
These findings highlight the potential of online plasticity in improving the performance (best fitness) and robustness of two-timescale NeuroEvolution in interactive environments.

\section{Preliminaries and Notation}
\label{sect_back}

\noindent \textbf{Notation.} 
$|A|$ is the cardinality of a finite set $A$.
We use the standard Euclidean inner product on $\mathbb{R}^{d}$:
$\langle w, u\rangle := \sum_{j=1}^{d} w_j\,u_j, \|w\|_2 := \sqrt{\langle w,w\rangle}$.
We define Euclidean projection of $x \in \mathbb{R}^d$ on space $\mathcal{W}$ by $\Pi_{\mathcal{W}}(x) := \arg\min_{w\in\mathcal{W}}\frac{1}{2}\|w-x\|_2^2$.
For a convex function $f:\mathbb{R}^d\to\mathbb{R}$, the {subdifferential} at $x$ is
$\partial f(x)
:=\left\{g\in\mathbb{R}^d:\;
f(y)\ge f(x)+\langle g,\,y-x\rangle\ , \ \forall\,y\in\mathbb{R}^d\right\}$.
Any $g\in\partial f(x)$ is called a {subgradient} of $f$ at $x$.
An algorithm is no-regret if its regret is sublinear $\mathbb{E}[R_T]=o(T)$, i.e., its average performance converges asymptotically to the optimal policy\footnote {Note that different no-regret bounds can differ in different convergence rates. ``Optimal" here does not mean the convergence rate is the best~\citep{orabona2019modern}.}.
The natural filtration,  $(\mathcal{G}_{n})_{n\geq 0}$ associated with the stochastic process $(X_{n})_{n\geq 0}$ on probability space $(\Omega, \mathcal{F},\Pr)$ is given by $\mathcal{G}_{n}=\sigma (X_{0},\ldots ,X_{n} )$ (sigma-algebra generated by a family of r.v.s..) for $n\geq 0$.

\paragraph{Online Learning via Synaptic Plasticity.}
Online learning follows a sequential protocol: a learner repeatedly acts, observes (reward) feedback, and updates its parameters immediately, to minimise cumulative loss over time \citep{ShalevShwartz2011Online}. 
In biological systems, life cycle adaptation is often attributed to synaptic plasticity, in which synaptic strength varies with changes in local neural activity~\citep{Martin2000SPM,Takeuchi2014SPM}.
Importantly, many plasticity mechanisms are gated by a third factor, such as neuromodulatory signals correlated with reward, yielding three-factor rules that support behaviourally relevant credit assignment without backpropagating gradients through time \citep{FremauxGerstner2016ThreeFactor,Gerstner2018Eligibility,Florian2007ModulatedSTDP}. Motivated by this, we use reward-modulated synaptic updates as the inner-loop learning mechanism in NEOL.
Let $x$ denote presynaptic activity, $y$ postsynaptic activity, $w$ a synaptic weight, $\eta>0$ a step size, $r$ the scalar reward, $\tau \in \mathbb{N}$ denote the time step in the \textsc{Rollout} and $\beta>0$ a reward scaling factor. 
We study three standard local rules.
\paragraph{Reward-modulated Hebbian rule.}
Following Hebb's rule \citep{Hebb1949}, correlated pre- and postsynaptic activities strengthen a synapse.
We use a reward-gated three-factor form whose local correlation defines $e_\tau=x_\tau y_\tau$:
$
\Delta w_\tau = \eta\,\beta\, r_\tau\, e_\tau =  \eta\,\beta\, r_\tau\, x_\tau\, y_\tau .
$
This rule is maximally local and simple, but can be unstable without additional constraints \citep{CaporaleDan2008STDPReview}.

\paragraph{Reward-modulated Oja rule.}
Oja’s rule varied from Hebbian learning with an additional activity-dependent negative feedback normalisation that prevents divergence \citep{Oja1982}. Using the same reward-gating yields
$
e_\tau =  y_\tau\bigl(x_\tau - y_\tau w_\tau\bigr), 
\Delta w_\tau =  \eta\,\beta\, r_\tau\, e_\tau
=  \eta\,\beta\, r_\tau\, y_\tau\bigl(x_\tau - y_\tau w_\tau\bigr).
$
This stabilises the dynamics and yields a principal-component interpretation under stationary inputs.

\paragraph{Reward-modulated BCM rule.}
The BCM rule introduces a sliding threshold $\theta_\tau$ that separates depression from potentiation and supports selectivity with homeostasis \citep{Bienenstock1982BCM}. In reward-modulated form,
$
e_\tau=  y_\tau\bigl(y_\tau-\theta_\tau\bigr)x_\tau,
\Delta w_\tau =  \eta\,\beta\, r_\tau\, e_\tau
=  \eta\,\beta\, r_\tau\, y_\tau\bigl(y_\tau-\theta_\tau\bigr)x_\tau .
$
The threshold $\theta_\tau$ is updated on a slower timescale to track recent postsynaptic activity.

\section{NeuroEvolutionary Online Learning}
\label{sec:neol}
\algnewcommand\Input{\item[\text{Input:}]}
\algnewcommand\Output{\item[\text{Output:}]}

In this section, we present a NeuroEvolutionary Online Learning (NEOL) framework. 
We begin with a high-level overview, then describe the decoupled update strategy for weights and topology in Section~\ref{sec:decoupling}, and finally detail the main algorithm in Section~\ref{sec:neol_alg}. 

\begin{figure}[t]
    \centering
    \includegraphics[width=1.0\linewidth]{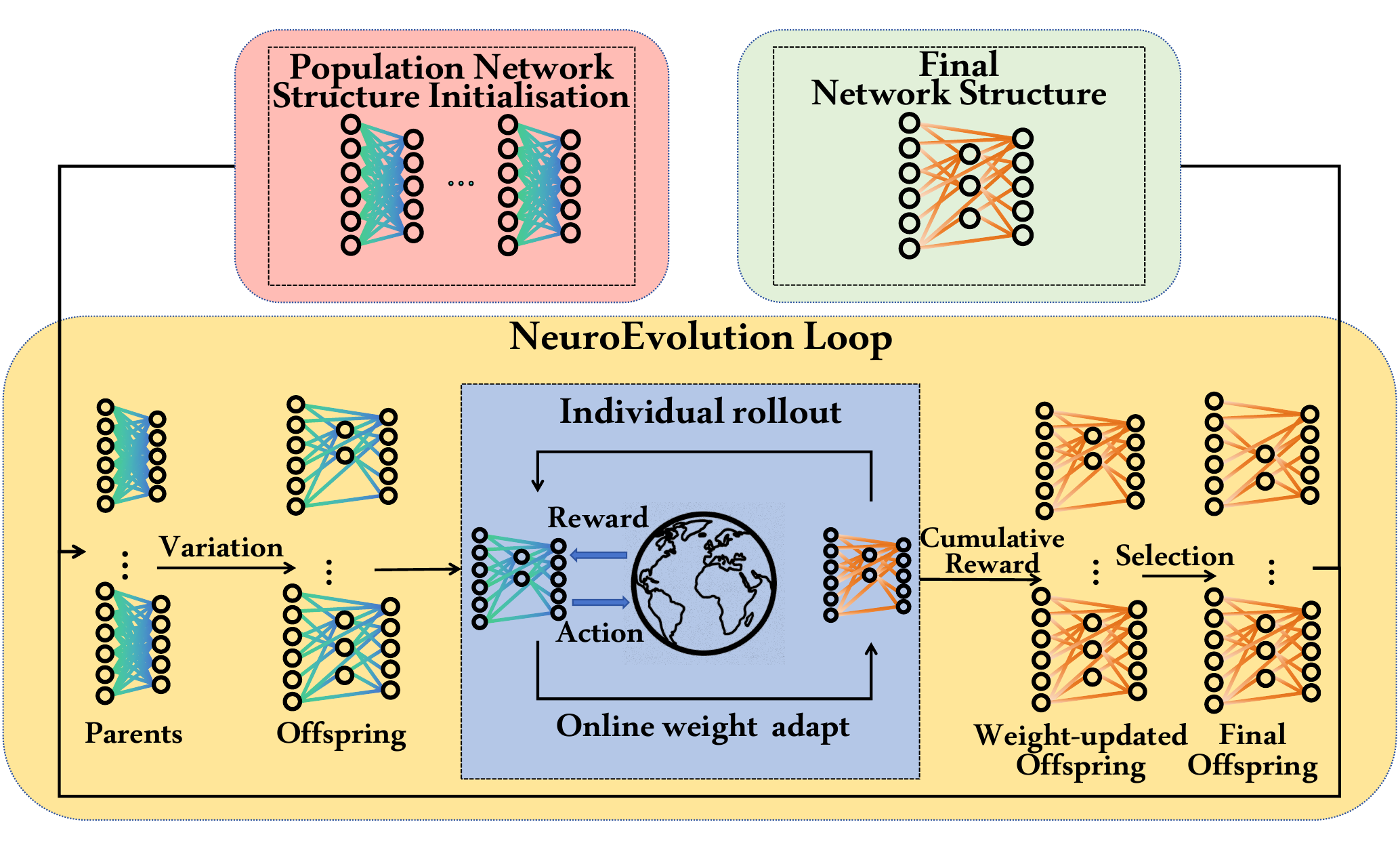}
    \vspace{-0.5cm} 
    \caption{
    Overview of the NEOL framework. 
    The procedure begins by initialising the network structures population. In each generation, the variation operators generate offspring from the current parent generation. 
    Each offspring is evaluated in a separate rollout, during which its weights are adapted online via reward-modulated plasticity (illustrated connectivity lines colour change from blue to orange). 
    The cumulative reward from the rollout is used as the fitness for selection to form the next generation. 
    The best individual in the population determines the final network structure, including its adapted weights and evolved topology.}
    \vspace{-0.5cm} 
    \label{fig:ERLchart}
\end{figure}

\subsection{Decoupling Updates for Weight and Topology}
\label{sec:decoupling}
We maintain a population of network architectures, each encoded as a genome. 
The flowchart in Fig.~\ref{fig:ERLchart} shows a generational loop in which evolution and evaluation are interleaved to progressively improve solutions. 
Unlike standard NEAT, which mutates both topology and weights offline between episodes and evaluates policies with fixed weights during rollouts, NEOL decouples the two update processes. 
During each rollout, synaptic plasticity performs online weight adaptation once the reward feedback is observed and then, based on the current policy, the agent chooses an action to interact with the environment until the inner loop phase ends.
Then, the outer loop uses an evolutionary variation to modify the topology/architecture of networks and select the final offspring based on final fitness, which is the cumulative reward of the agent interacting with the environment.
Such a separation reserves the structural innovation and adapts to real-time feedback, thereby improving the performance of the algorithm (in the sense of fitness) and  stabilising the search process (see Section~\ref{sec:experiments}).

\subsection{General Outer-Inner Framework for NEOL}
\label{sec:neol_alg}

This section provides more detailed pseudocode for NEOL.
Algorithm~\ref{alg:NEOL_general_decoupled} consists of a two-timescale learning phase.
At the outer loop ($t=1,\dots,T$), the algorithm selects an architecture $h_t\in\mathcal{H}$ using an outer
selector state $\Theta_t$ (which might be a probability distribution over candidates, or a population of architectures in evolutionary
algorithms such as NeuroEvolution of Augmenting Topologies (NEAT, a neuroevolution method that evolves both neural network topology and connection weights in an offline manner~\citep{Stanley2002neat}.). 
The selected architecture is then evaluated by an inner \textsc{Rollout} learning procedure
(Algorithm~\ref{alg:NEOL_inner_decoupled}), which adapts the weights online through reward-modulated plasticity and returns
an estimated fitness $\widehat{J}_t$. 
``Env" denotes the task environment, modelled as an episodic MDP \(\mathcal M=(\mathcal S,\mathcal A,P,r,\rho_0)\) with horizon \(T_{\max}\), where \(\mathcal S\) is state space, \(\mathcal A\) is action space, \(P(\cdot\mid s,a)\) is transition kernel, \(r(s,a)\) is a reward function, and \(\rho_0\) is the initial-state distribution.
Finally, the outer state $\Theta_t$ is updated based on this evaluation, favouring future selections toward higher-performing architectures.

Algorithm~\ref{alg:NEOL_inner_decoupled} evaluates a fixed architecture $h$ by performing an online adaptation through reward-modulated plasticity. 
For each of the $N$ episodes, the environment is reset, and the network
weights are initialised uniformly at random over candidate architectures $h$ (line~4). 
At every interaction step $\tau\le T_{\max}$, the current policy $\pi_{w^{(h)}_{\tau}}$ produces an action $a$, the environment returns the next state and reward, and a local synaptic trace $u_{\tau}^{(h)}$ is computed from pre/post-synaptic activity (line~9). 
The weights are then updated online via a reward-gated plasticity rule (Hebb/Oja/BCM) with step size $\eta$ and reward scaling $\beta$ (line~11).
Finally, the episode return is accumulated as fitness, and the final output
is the average return $\widehat{J}(h,w^{(h)})$ across episodes together with the adapted weight vector $w^{(h)}$.



\begin{algorithm}[!t]
\caption{General \textsc{Outer}-\textsc{Inner} framework}
\label{alg:NEOL_general_decoupled}
\begin{algorithmic}[1]
\Input Candidate space $\mathcal{H}$ (architectures/genomes); outer rounds $T$.
Outer selector state $\Theta$ (e.g., a distribution $p_t$ on $\mathcal{H}$, or a NEAT population $\mathcal{P}_t$).
Inner loop: learning rule $\mathcal{L}$; step size $\eta$; reward scale $\beta$; episodes $N$; horizon $T_{\max}$; env.
\Output Final architecture and weights $(h_T,w_T^{(h_T)})$.

\State $\Theta_1 \gets \textsc{OuterInit}(\mathcal{H})$
\Comment{e.g., uniform $p_1(h)=1/|\mathcal{H}|$ or random initial population $\mathcal{P}_1$}
\For{$t=1,\dots,T$} \Comment{Outer: architecture selection}
    \State $h_t \gets \textsc{OuterSelect}(\Theta_t)$ 
    \State $(w_t^{(h_t)},\, \widehat{J}_t) \gets \textsc{Rollout}(h_t,\mathcal{L},\eta,\beta,N,T_{\max},\textit{env})$  
    \State ${\ell}_t(h_t) \gets \textsc{Loss}(\widehat{J}_t)$ \Comment{e.g., ${\ell}_t=-\widehat{J}_t$}
    \State $\Theta_{t+1} \gets \textsc{OuterUpdate}(\Theta_t, h_t, {\ell}_t(h_t))$
\EndFor
\State \Return $(h_T, w_T^{(h_T)})$
\end{algorithmic}
\end{algorithm}

\begin{algorithm}[!t]
\caption{\textsc{Rollout}: reward-modulated plasticity}
\label{alg:NEOL_inner_decoupled}
\begin{algorithmic}[1]
\Input Fixed architecture $h$; learning rule $\mathcal{L}$; step size $\eta$; reward scale $\beta$;
episodes $N$; inner horizon $T_{\max}$; env.
\Output Adapted weights $w^{(h)}$; average return $\widehat{J}(h,w^{(h)})$.

\State $\textsc{returns}\gets[]$
\State $w^{(h)} \gets \text{null}$
\For{$e=1,\dots,N$}
    \State Instantiate net from $h$; initialise plastic weights $w^{(h)}_{1}$
    \State $s\gets\textit{env.reset}()$, $R\gets0$
    \For{$\tau=1,\dots,T_{\max}$} 
        \State $a \gets \pi_{w^{(h)}_{\tau}}(s)$
        \State $(s',r,\text{done})\gets \textit{env.step}(a)$
        \State $u_{\tau}^{(h)} \gets \textsc{LocalTrace}(s,a,s')$ 
        \State $w^{(h)}_{\tau+1} \gets \textsc{PlasticUpdate}(w^{(h)}_{\tau}, u_{\tau}^{(h)}, r;\mathcal{L},\eta,\beta)$
        \State $R\gets R+r$, $s\gets s'$
        \If{done} \textbf{break} \EndIf
    \EndFor
    \State $\textsc{returns.append}(R)$
    \State $w^{(h)} \gets w^{(h)}_{\tau+1}$ 
\EndFor
\State $\widehat{J}(h,w^{(h)})\gets \frac{1}{N}\sum_{e=1}^N \textsc{returns}[e]$
\State \Return $(w^{(h)},\, \widehat{J}(h,w^{(h)}))$
\end{algorithmic}
\end{algorithm}

\par This paper presents a simplified theoretical lens for the outer-loop with inner-loop structure in Algorithms~\ref{alg:NEOL_general_decoupled}--\ref{alg:NEOL_inner_decoupled}.
We show that under standard boundedness assumptions, the hybrid admits a sublinear regret bound.
For simplicity of both theoretical and empirical analysis, we consider episode $N=1$ in this paper.

\subsection{Simplified protocol and regret definition}
\label{app:toy_setting}
Let $\mathcal{H}$ be a finite set of candidate genomes/architectures (e.g., a finite pool of NEAT topologies), with $M := |\mathcal{H}| \in \mathbb{N}$.
Each $h\in\mathcal{H}$ induces a parameter space $\mathcal{W}_h\subset\mathbb{R}^{d_h}$.
We consider an online interaction protocol over outer rounds $t=1,\dots,T$.


\begin{enumerate}
  \item At outer round $t$, the environment defines an evaluation instance (e.g., initial state distribution).
  \item The outer loop selects an architecture $h_t\in\mathcal{H}$ according to $p_t\in\Delta(\mathcal{H})$.
  \item Given $h_t$, the inner loop runs for $N$ episodes at most $T_{\max}$ iterations using reward-modulated plasticity,
  generating adaptive weights $w_t^{(h_t)}\in\mathcal{W}_{h_t}$ and a (stochastic) average return
  $J_t\big(h_t,w_t^{(h_t)}\big)\in[0,1]$ (normalised).
  \item Define the (stochastic) loss at round $t$ as $ \ell_t\big(h_t,w_t^{(h_t)}\big) := 1 - J_t\big(h_t,w_t^{(h_t)}\big)\in[0,1]$.
\end{enumerate}

As a simplified mode, in our analysis, we consider $N=1$ and $T_{\max}$ is constant with respect to the outer iteration $T$.


We define the (expected) regret of the hybrid method as
\begin{align*}
 \mathbb{E}[R_T]
:= \mathbb{E}\Bigg[\sum_{t=1}^T \ell_t\big(h_t, w_t^{(h_t)}\big)
- \min_{h\in\mathcal{H}}\min_{w\in\mathcal{W}_h}\sum_{t=1}^T \ell_t(h,w)\Bigg],
\end{align*}
where $\ell_t(h,w)$ denotes the loss that would be incurred on round $t$ if the candidate $(h,w)$that is in round $t$ if deployed on the same evaluation instance, and the expectation is taken from all the randomness in evaluation instances, rollouts, and the algorithm.
Regret quantifies the cost of online learning.
The use of the Regret metric $R(T)$ is motivated by the need to balance structural exploration (finding the right NEOL topology) with parametric exploitation (optimising weights for a given topology). 
A sub-linear regret bound, $R(T)/T \to 0$ as $T \to \infty$, guarantees that the learner’s cumulative reward is asymptotically as good as the best possible architecture-weight pair $(h^*, w^*)$ in the candidate set.


\subsection{Assumptions}
\label{app:toy_assumptions}

Most NEOL algorithms, including NEAT and its variants, are highly intractable. 
It is sensible to make some assumptions before conducting regret analysis.

\noindent{\bf A1 Finite architecture.}
$\mathcal{H}$ is finite with $|\mathcal{H}|=M<\infty$.

\noindent{\bf A2 Bounded rewards/losses.}
Rewards are normalised so that $r_t(h,w)\in[0,1]$ for all $t,h,w$, hence $\ell_t(h,w)\in[0,1]$.

\noindent{\bf A3 Bounded activities and bounded parameter domain.}
For each architecture $h$, the (vectorised) weights lie in a convex and compact set $\mathcal{W}_h$ with diameter at most $D$:
$
\|w-w'\|_2 \le D, \forall w,w'\in\mathcal{W}_h.
$
Moreover, for each $h$ and each round $t$, the network induces bounded {local synaptic traces}
(e.g., pre-/post-synaptic activity products) denoted by $z_t:=x_ty_t\in\mathbb{R}^{d_h}$,
there exists a constant $Z$ such that
$
\|z_t\|_2 \le Z  \text{ for all } t,h.
$
For BCM, we additionally assume a bounded sliding threshold $\theta_t\in[-1,1]$ and bounded post-synaptic activity $y_t\in[-1,1]$.

\noindent{\bf A4 NEOL inner update is a reward-modulated local plasticity rule.}
For each fixed architecture $h$, the within-lifetime update (Algorithm~\ref{alg:NEOL_inner_decoupled}, line 9-10) is a local rule driven by synaptic traces and reward,
implemented with projection to keep weights in $\mathcal{W}_h$:
\begin{equation}
\label{eq:generic_local_update}
w_{t+1}^{(h)} = \Pi_{\mathcal{W}_h}\!\big(w_t^{(h)} + \eta\, u_t^{(h)}\big),
\end{equation}
where $u_t^{(h)}$ depends only on local quantities and the reward.



\noindent
Assumption~4 does not assume backpropagation or policy gradients.
It only assumes that the inner-loop update is a bounded local synaptic rule of the form \eqref{eq:generic_local_update},
which holds for the reward-modulated Hebb/Oja/BCM updates below when weights are clipped/projected.

\begin{align*}
\label{eq:plasticity_updates}
u_t^{(h)} = \,
\begin{cases}
r_tz_t  & \text{Hebb},\\
r_t \left(z_t - \alpha_t w_t^{(h)} \right) & \text{Oja},\\
r_ty_t (y_t-\theta_t)\, x_t & \text{BCM}.
\end{cases}
\end{align*}
Here $\alpha_t\ge 0$ is bounded (e.g., $\alpha_t=y_t^2$ with $y_t\in[-1,1]$), and
$y_t,\theta_t$ are bounded as in Assumption~3. 
The pre-synaptic activity $x_t$ can be absorbed into
$z_t $ if desired.
We say the plasticity rule is reward-modulated if $u_t^{(h)}=r_t \cdot \text{local trace}$, where $r_t$ is the reward.

\noindent{\bf A5 Surrogate domination almost certainly.}
For each architecture $h\in\mathcal{H}$ and each round $t$, define a linear surrogate
$L_t^{(h)}(w):=-\langle w,u_t^{(h)}\rangle$.
We assume there exists a constant $C\ge 1$ such that for all $w\in\mathcal{W}_h$,
\begin{align*}
\ell_t\big(h,w_t^{(h)}\big)-\ell_t(h,w)
\;\le\;
C\, \left( L_t^{(h)}\big(w_t^{(h)}\big)-L_t^{(h)}(w) \right)
\end{align*}
holds almost surely.
An informal intuition for Assumption~5 is:
if $\ell_t(h,\cdot)$ is locally smooth and the reward-modulated plasticity direction $u_t^{(h)}$ is positively aligned with the descent direction of $\ell_t(h,\cdot)$ at $w_t^{(h)}$, then the decrease in the true loss can be controlled by the decrease in the linear surrogate,

\noindent{\bf A6 Selection as a soft and fitness-based reweighting.}
At each outer round $t$, NEOL produces a population of genomes and applies selection based on their fitness values. 
For theoretical tractability, we model this selection as maintaining a probability distribution $p_t$ over a finite candidate set $\mathcal{H}$ of genomes/architectures, and updating it by a monotone fitness-based reweighting rule of exponential-weights form:
\begin{equation}
\label{eq:ew_update}
p_{t+1}(h)\;=\;\frac{p_t(h)\exp\!\big(-\gamma\,\widehat{\ell}_t(h)\big)}
{\sum_{h'\in\mathcal{H}} p_t(h')\exp\!\big(-\gamma\,\widehat{\ell}_t(h')\big)},
\end{equation}
where $\widehat{\ell}_t(h)$ is an empirical estimate of $\ell_t(h,w_t^{(h)})$.
The exponential weight update follows the classical Hedge algorithm~\citep{freund1997decision}.
This abstraction captures that genomes with smaller estimated loss (higher fitness) become more likely to be selected, while ignoring NEAT-specific mechanisms such as speciation, explicit fitness sharing, and structural mutations.
Exponential weights are a standard setting and starting point for such selection dynamics, helping us conduct a clear outer-loop regret analysis.




\noindent{\bf A7 Full-information evaluation over the candidate pool.}
At each outer round $t$, we assume the loss values are revealed for all candidates in the pool $\mathcal{H}$:
the outer update has access to the entire loss vector
$\big(\ell_t(h,w_t^{(h)})\big)_{h\in\mathcal{H}}$ at round $t$.

\section{Regret Analysis of a Simplified NEOL}
First, we consider the inner-loop regret for reward-modulated local updates.
\begin{restatable}{lemma}{Lemmaone}
\label{lem:inner}
Fix any $h\in\mathcal{H}$. Under Assumptions~A2--4, define the (convex) surrogate inner loss: $L_t^{(h)}(w) := -\langle w, u_t^{(h)}\rangle$,
where $u_t^{(h)}$ is the local update direction from \eqref{eq:generic_local_update}.
Then the projected update \eqref{eq:generic_local_update} 
satisfies
\begin{align*}
\sum_{t=1}^T L_t^{(h)}\!\big(w_t^{(h)}\big)
- \min_{w\in\mathcal{W}_h}\sum_{t=1}^T L_t^{(h)}(w)
\le \frac{D^2}{2\eta} + \frac{\eta}{2}\sum_{t=1}^T \|u_t^{(h)}\|_2^2.
\end{align*}
Moreover, if $\|u_t^{(h)}\|_2\le U$ for all $t$ (implied by A2--4), then
\begin{align*}
\sum_{t=1}^T L_t^{(h)}\!\big(w_t^{(h)}\big)
- \min_{w\in\mathcal{W}_h}\sum_{t=1}^T L_t^{(h)}(w)
\;\le\; \frac{D^2}{2\eta} + \frac{\eta U^2 T}{2}.
\end{align*}
Choosing $\eta = D/(U\sqrt{T})$ yields
\[
\sum_{t=1}^T L_t^{(h)}\!\big(w_t^{(h)}\big)
- \min_{w\in\mathcal{W}_h}\sum_{t=1}^T L_t^{(h)}(w)
\;\le\; DU\sqrt{T}.
\]
\end{restatable}

Lemma~\ref{lem:inner} shows that under the bounded assumptions for architecture and bounded local synaptic update, we can bound the inner regret by $DU\sqrt{T}$, where the weight vectors lie in a convex and compact set $\mathcal{W}_h$ with diameter $D$ and the $2$-norm of the bounded local synaptic traces is at most $U$.
This implies that the inner regret can be controlled by the size of the space of weights and local synaptic update in a linear sense. 



Next, we consider the regret of the outer-loop, assuming the architecture uses exponential weights/Hedge.

\begin{restatable}{lemma}{Lemmatwo}
\label{lem:outer}
Let $\{\mathcal{F}_t\}_{t\ge 1}$ be a natural filtration generated by all randomness and loss observations up to the end of round $t-1$.
If $h_t\sim p_t$ is drawn after observing
$\mathcal{F}_t$ and 
under Assumptions~1,2,6,7 with losses in $[0,1]$, the exponential-weights update \eqref{eq:ew_update} satisfies
\begin{align*}
\begin{split}
&
\sum_{t=1}^T \mathbb{E}_{h_t\sim p_t}\!\left[\ell_t\big(h_t,w_t^{(h_t)}\big)\mid \F_t\right]- \min_{h\in\mathcal{H}} \sum_{t=1}^T \ell_t\big(h,w_t^{(h)}\big)  \\
&\quad \le \frac{\log M}{\gamma} + \frac{\gamma T}{8}.
\end{split}
\end{align*}
Choosing $\gamma=\sqrt{8\log M / T}$ yields
\begin{align*}
\begin{split}
&
\sum_{t=1}^T \mathbb{E}_{h_t\sim p_t}\!\left[\ell_t\big(h_t,w_t^{(h_t)}\big) \mid \F_t \right]- \min_{h\in\mathcal{H}} \sum_{t=1}^T \ell_t\big(h,w_t^{(h)}\big)  \\
&\quad \;\le\; \sqrt{\frac{T\log M}{2}}.
\end{split}
\end{align*}

\end{restatable}

Lemma~\ref{lem:outer} shows that we can bound the outer regret conditioning on the current observation by $\sqrt{(T\log M)/2}$, where $M$ is the size of the set of architectures.
This implies that the outer regret can also be controlled by the size of the set of architectures in a square-root sense. 



Finally, using Lemmas~\ref{lem:inner} and \ref{lem:outer}, we obtain:

\begin{restatable}{theorem}{Mainthm}
Under Assumptions~1--7, with $\eta = D/(U\sqrt{T})$ in Lemma~\ref{lem:inner} and $\gamma=\sqrt{8\log M / T}$ in Lemma~\ref{lem:outer},
the expected regret is bounded as
\begin{align*}
 \mathbb{E}[R_T]\;\le\; CDU\sqrt{T} \;+\; \sqrt{\frac{T\log M}{2}}
\;=\; O\!\left(\sqrt{T}\right).
\end{align*}
\end{restatable}

\begin{proof}[Proof Sketch]
Theorem~4 decomposes the regret into an inner-loop term $CDU\sqrt{T}$ (cost of weight adaptation) and an outer-loop term $\sqrt{T\log M/2}$ (cost of architecture search).
Both are sublinear, so their sum is sublinear. 
\end{proof}

In our analysis, we model Algorithms~\ref{alg:NEOL_general_decoupled} and \ref{alg:NEOL_inner_decoupled} with reward-gated synaptic updates as an inner-loop online procedure of the form \eqref{eq:generic_local_update} (A3--4), where the update direction $u_t^{(h)}$ depends only on local traces and the scalar reward.
For theoretical tractability, we consider the outer selection as an exponential-weights update over a finite architecture pool $\mathcal{H}$ (A1 and A7).
In practice, Algorithm~\ref{alg:NEOL_general_decoupled} might evaluate each genome in the population via repeated online rollouts and then apply selection pressure through reproduction/speciation, which we investigate in Section~\ref{sec:experiments}.
Moreover, 
we consider a full-information outer update in which losses for all $h\in\mathcal{H}$ are available (A7).
Under these assumptions, the hybrid outer/inner-loop method admits a sublinear $O(\sqrt{T})$ regret guarantee, supporting this interpretation as an effective combination of evolutionary architecture search and online learning via plasticity updates.

\section{Experiments}\label{sec:experiments}
If we choose the learning rule $\mathcal{L} \in \{\text{Hebb, Oja, BCM}\}$ in Algorithm~\ref{alg:NEOL_general_decoupled}, then the general NEOL framework becomes a NEAT-based NEOL method (we give the pseudocode in the Appendix). 
In this section, we evaluate this method (namely, NEAT-NEOL) and compare it with standard NEAT and two strong reinforcement learning baselines.
We test whether online weight updates from the plasticity rules help, and how the results compare with well-tuned PPO and SAC in practice.

First, we provide a formal formulation of sequential decision making by using the Markov Decision Process (MDP).
Given an MDP defined by a tuple $\langle \St, \A, \cP, \RE, \gamma, T \rangle$ where $\St$ is the state space, $\A$ is the action space, $\RE: \St \times \A \rightarrow \mathbb{R}$ is the reward function and $\cP: \St \times \A \rightarrow \St$ is the transistion function.
In this paper, we consider an online RL setting where the agent can interact with the environment repeatedly until a certain time horizon $T$ by using a policy $\pi: \St \rightarrow \A$. 
Such an agent's policy is usually represented by a neural net.
Then, the goal of the entire learning process is to find an optimal policy $\pi^*$ such that it can maximise the expected rewards:
\begin{align*}
    \pi^* \in \arg\max_{\pi} \Epi{\sum_{t=0}^T  \RE\left(s_t,a_t  \right) \mid s_0,a_0 }.
\end{align*}
As a special setting for NeuroEvolution, this paper directly refers to the cumulative reward as the fitness of the policy at time horizon $T$, i.e., $f(\pi, T):=\sum_{t=0}^T\RE\left(s_t,a_t  \right)$.
\subsection{Experiment setups}


\noindent \textbf{Benchmark Environments.} We evaluate algorithms on four control tasks from the OpenAI gym benchmark suite~\citep{brockman2016openai}, spanning diverse reward structures and action spaces (detailed description deferred to Appendix).

\noindent \textbf{Baselines.}
We compare NEAT-NEOL with two baseline families.
(1) NEAT~\citep{Stanley2002neat}: we use the same outer-loop evolutionary search as in NEAT-NEOL, but remove online synaptic plasticity, i.e., synaptic weights are fixed during each evaluation rollout. 
(2) Proximal Policy Optimisation (PPO)~\citep{schulman2017ppo} and Soft Actor-Critic (SAC)~\citep{haarnoja2018sac}:
PPO serves as a strong on-policy baseline and is used on \texttt{CartPole}; for continuous-control tasks (\texttt{LunarLander}, \texttt{BipedalWalker}, \texttt{Hopper}), we evaluate both PPO and SAC to reflect standard practice for continuous action spaces.



\noindent\textbf{Experimental Configuration for NEAT and NEAT-NEOL.}
We ran NEAT and NEAT-NEOL with population sizes $P \in \{50,100,200,300\}$. 
Each run lasted for $G=500$ generations. In each generation, we evaluated all $P$ individuals on one rollout, so each run contained $P\times G$ rollouts in total (and $P\times G \times T_{max}$, where $T_{max}$ is the rollout length in environment steps). 
For NEAT-NEOL, we tuned the inner-loop learning rate \texttt{lr} by grid search over $\{2.5\times10^{-4},\,2.5\times10^{-3},\,2.5\times10^{-2},\,2.5\times10^{-1}\}$ and $T_{max}=10^3$. 
For each configuration, we used $30$ independent random seeds (See Figures~\ref{fig:NEOL_comparison_combined} and Figures 6-9 in Appendix).

\noindent\textbf{Experimental Configuration for PPO and SAC.}
For PPO and SAC, we trained each run for up to $10^7$ environment steps per seed. 
To make a fair comparison under the same interaction budget as the evolutionary runs, we also report results at $P\times G \times T_{max}=10^7$ environment steps for NEAT and NEOL
(i.e., $P\in \{50, 100\}, G\in \{100, 200\}$ and $T_{max}=10^3$).
We used $30$ random seeds and report boxplots of the best fitness (episode return) across seeds (See Figure~\ref{fig:NEOL_comparison_combined_RL}).




\subsection{NEAT-NEOL vs NEAT vs RL Baselines}
To systematically evaluate the role of online neural plasticity in evolutionary processes, we conducted a comprehensive comparison against NEAT.
This comparison pitted the standard NEAT method against our NEOL framework, which integrates online plasticity rules (Hebb, Oja, and BCM). 
The results clearly demonstrate that the NEOL framework exhibits substantial advantages across multiple test environments.

\begin{figure}[!ht]
    \centering
    \includegraphics[width=0.49\linewidth]{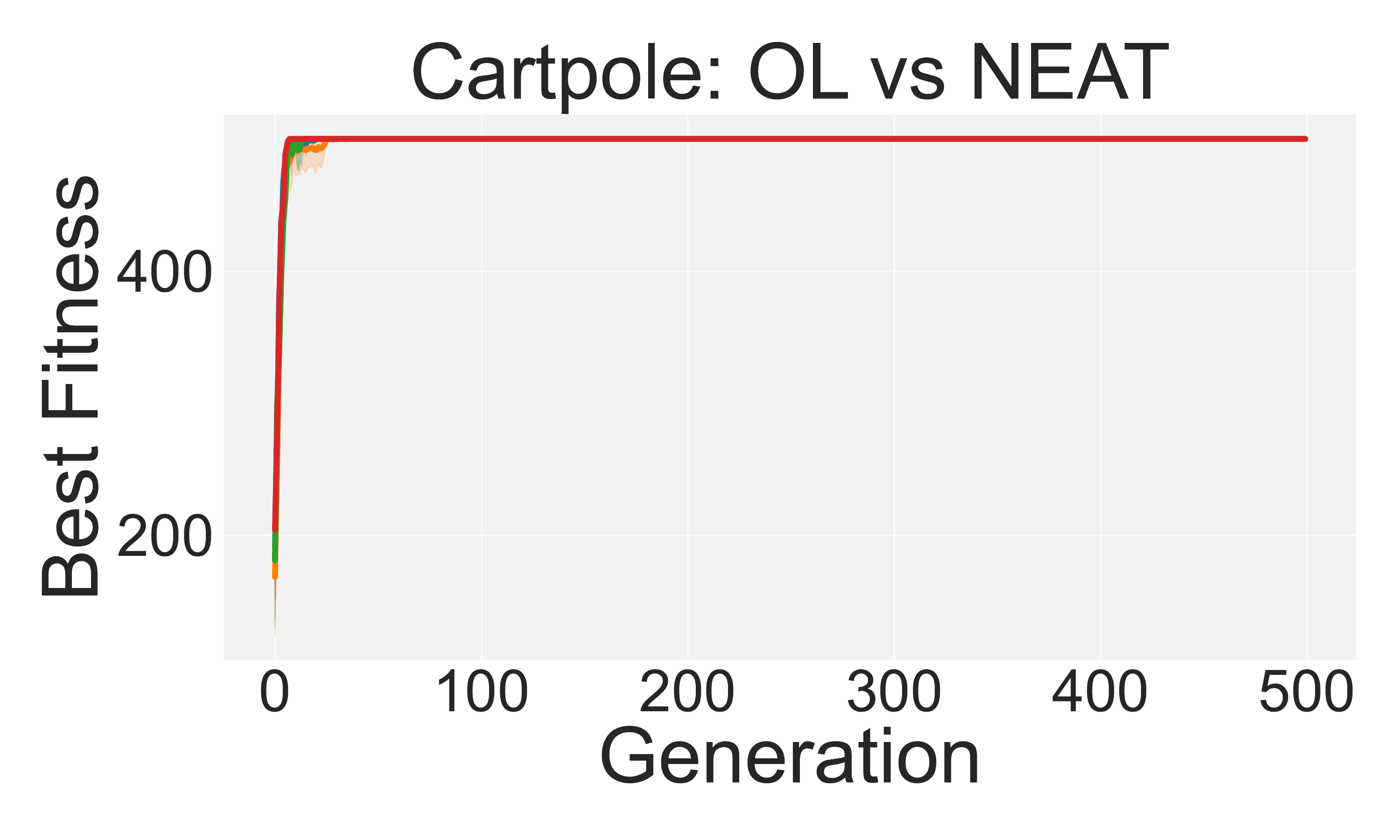}
    \includegraphics[width=0.49\linewidth]{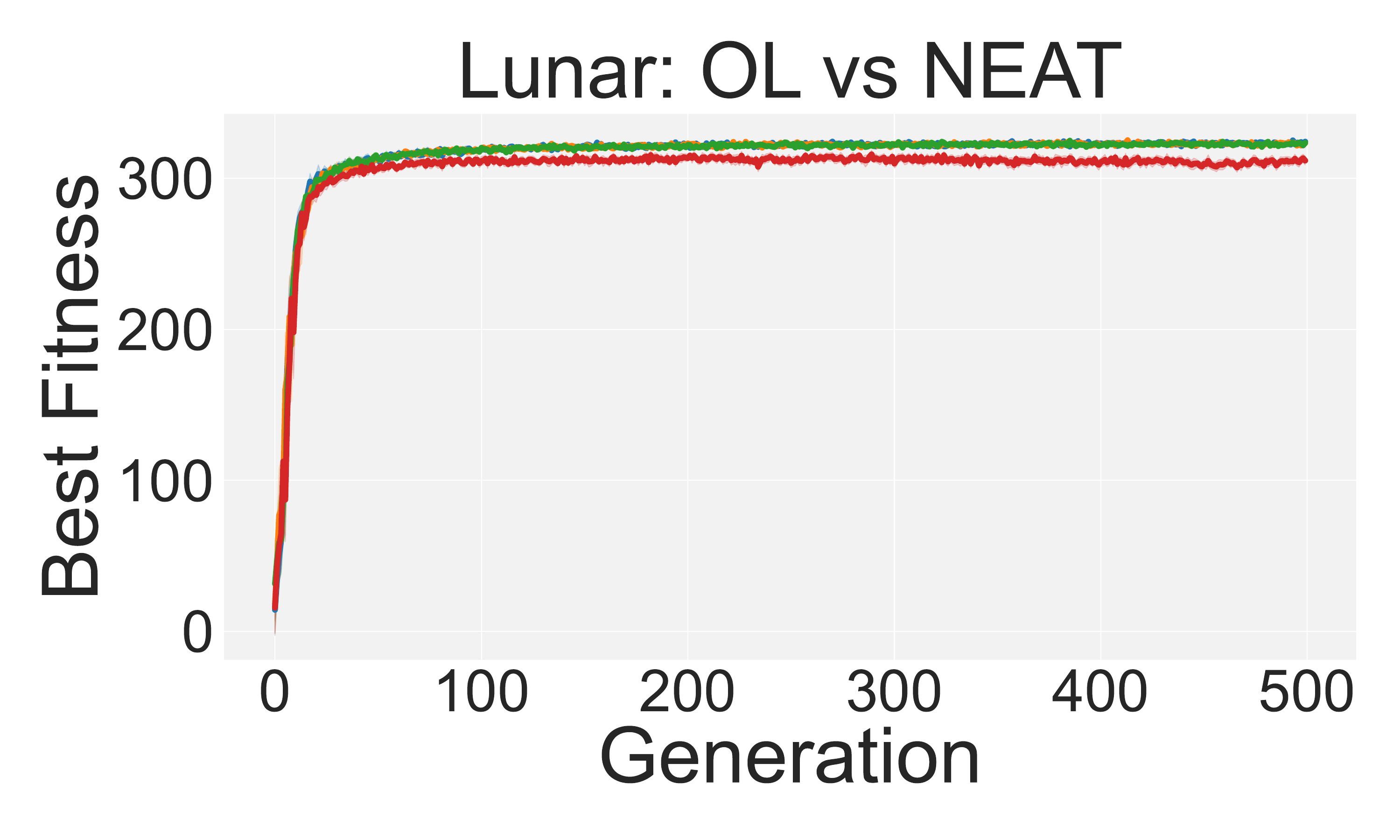}
    \includegraphics[width=0.49\linewidth]{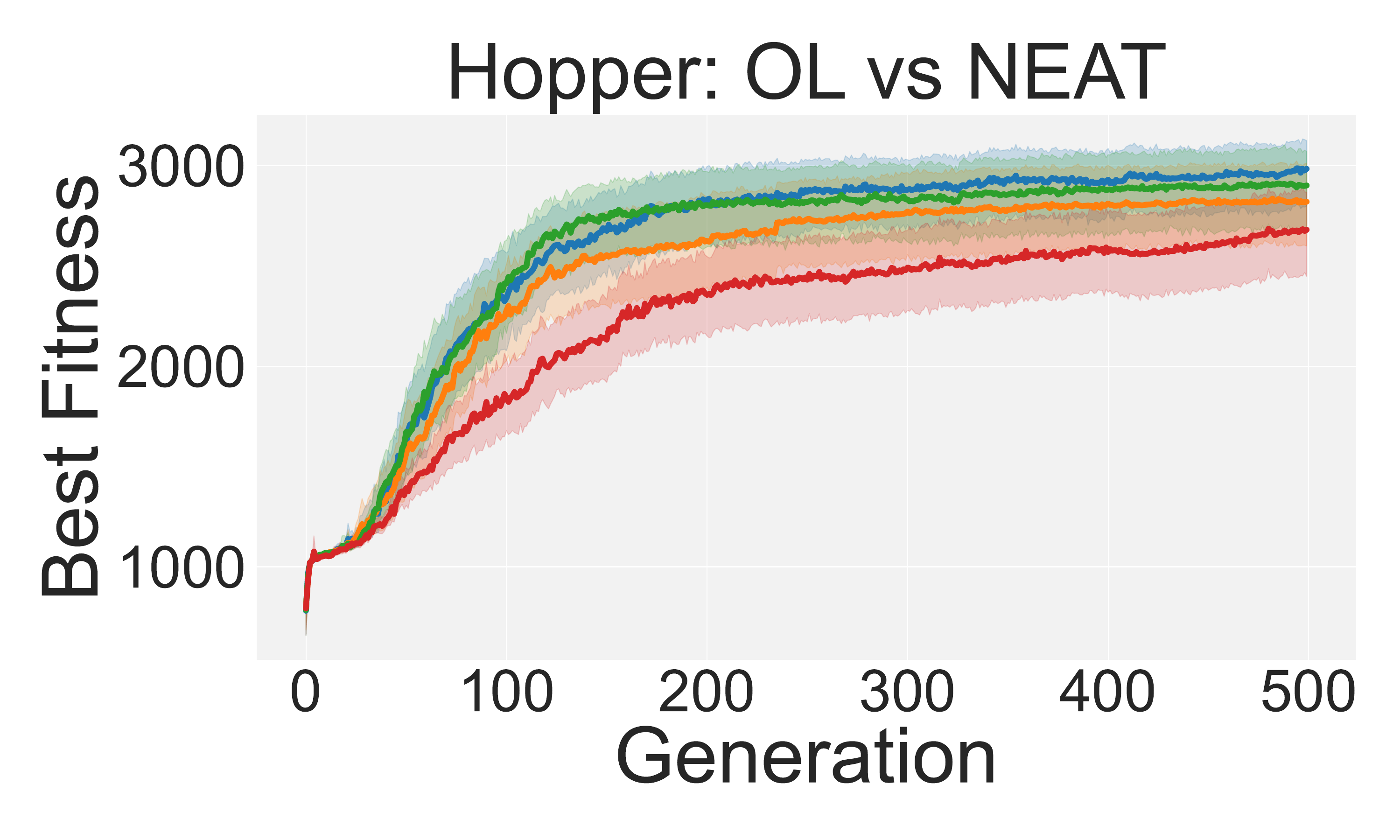}
    \includegraphics[width=0.49\linewidth]{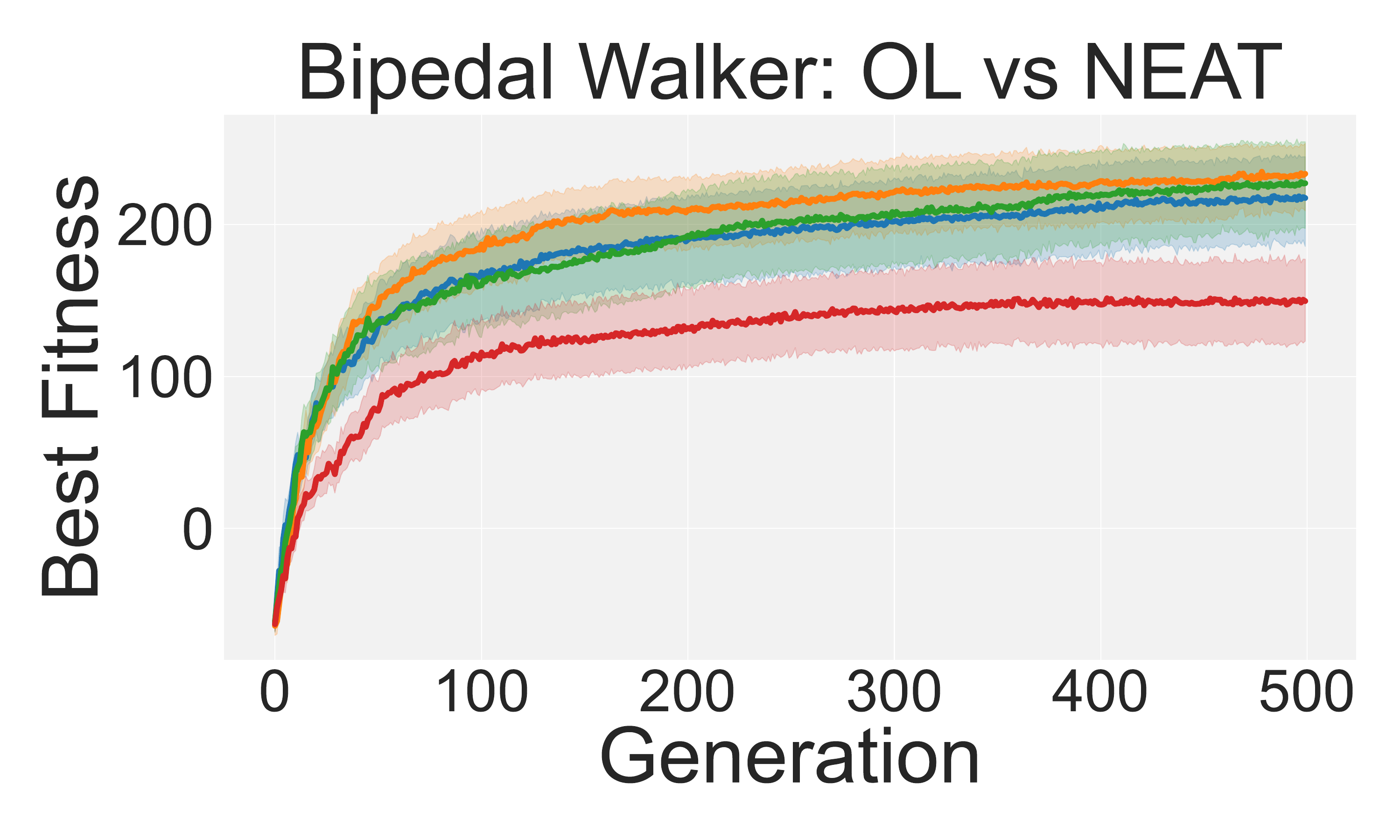}
    
    \includegraphics[width=0.49\linewidth]{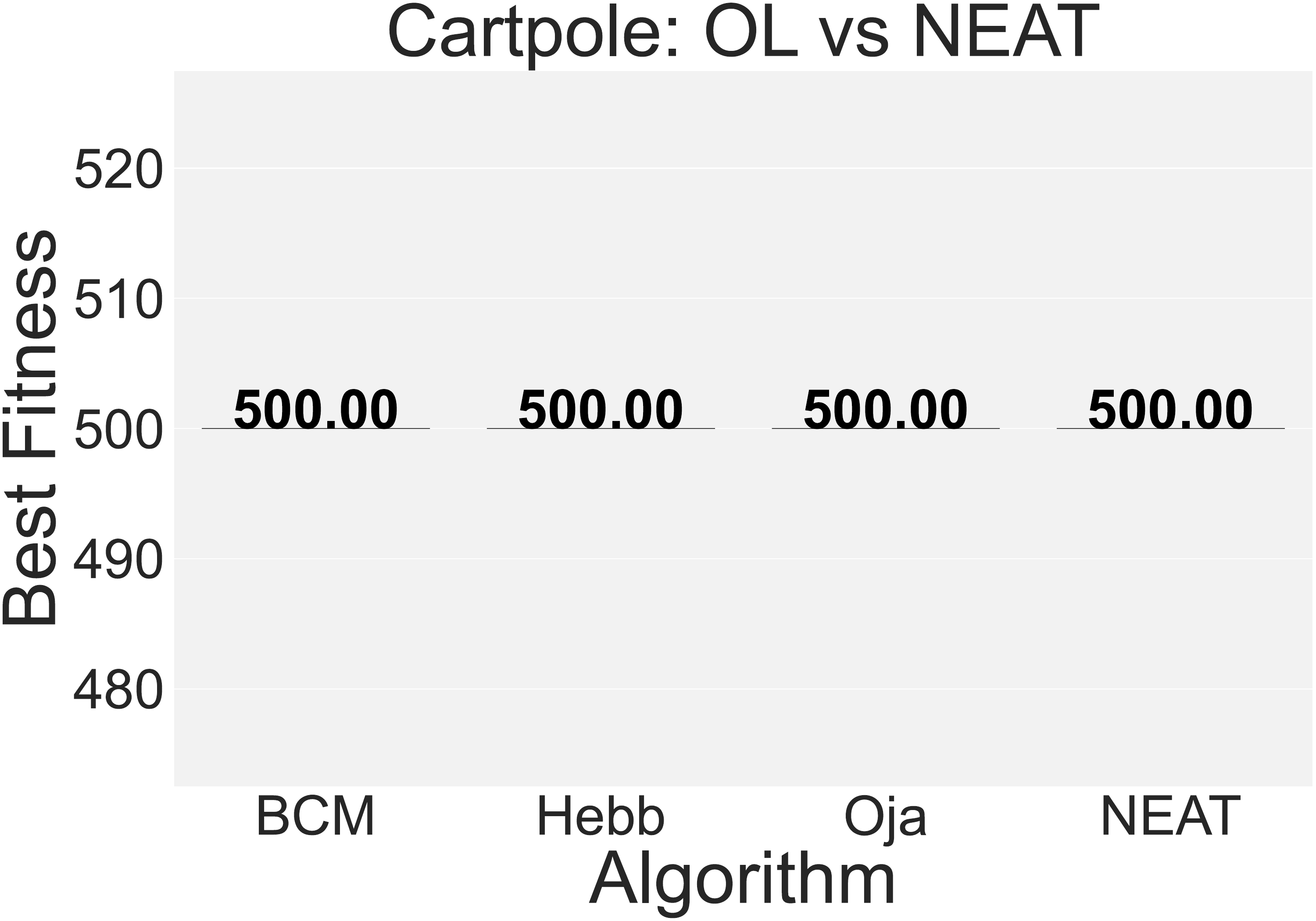}
    \includegraphics[width=0.49\linewidth]{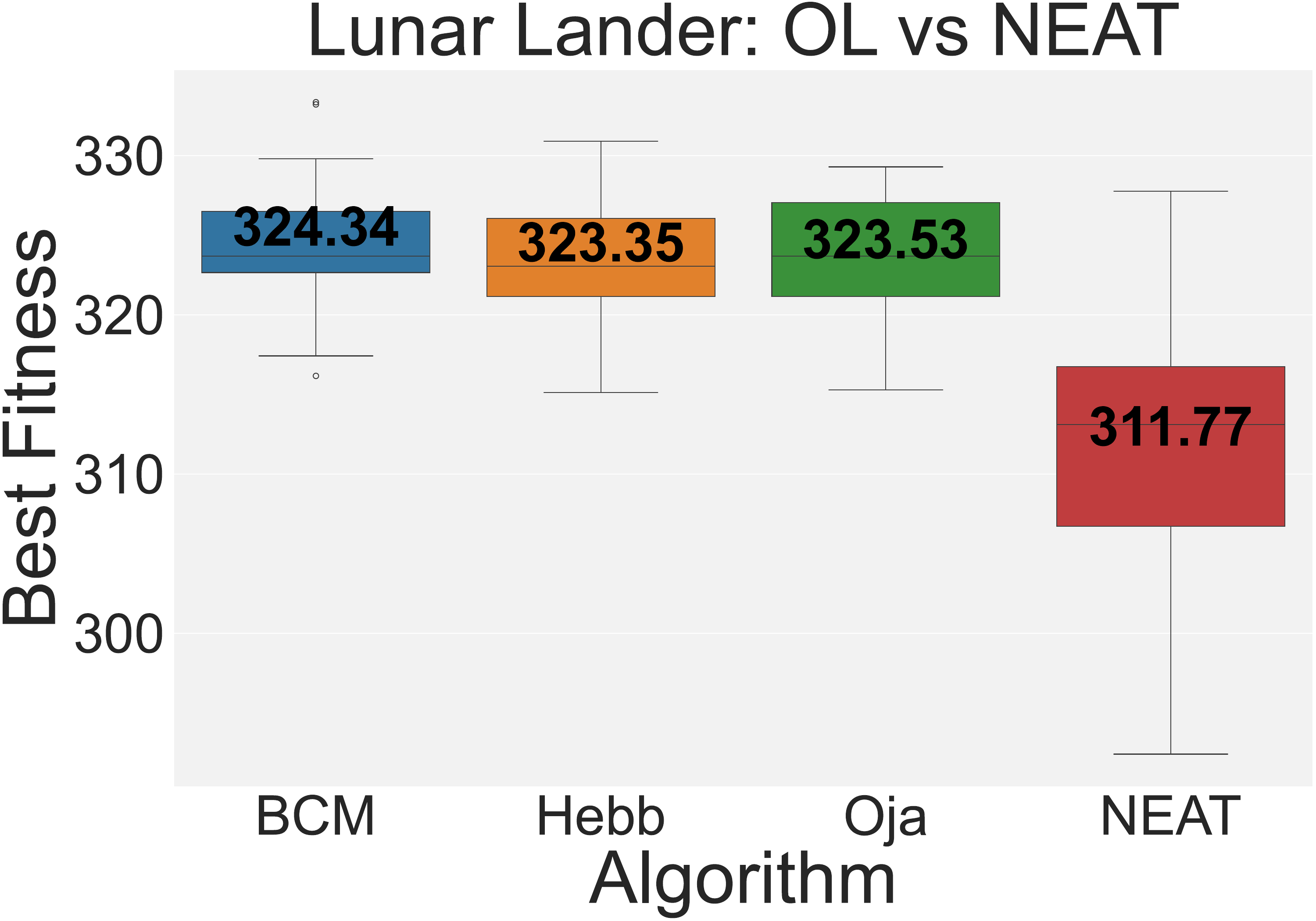}
    \includegraphics[width=0.49\linewidth]{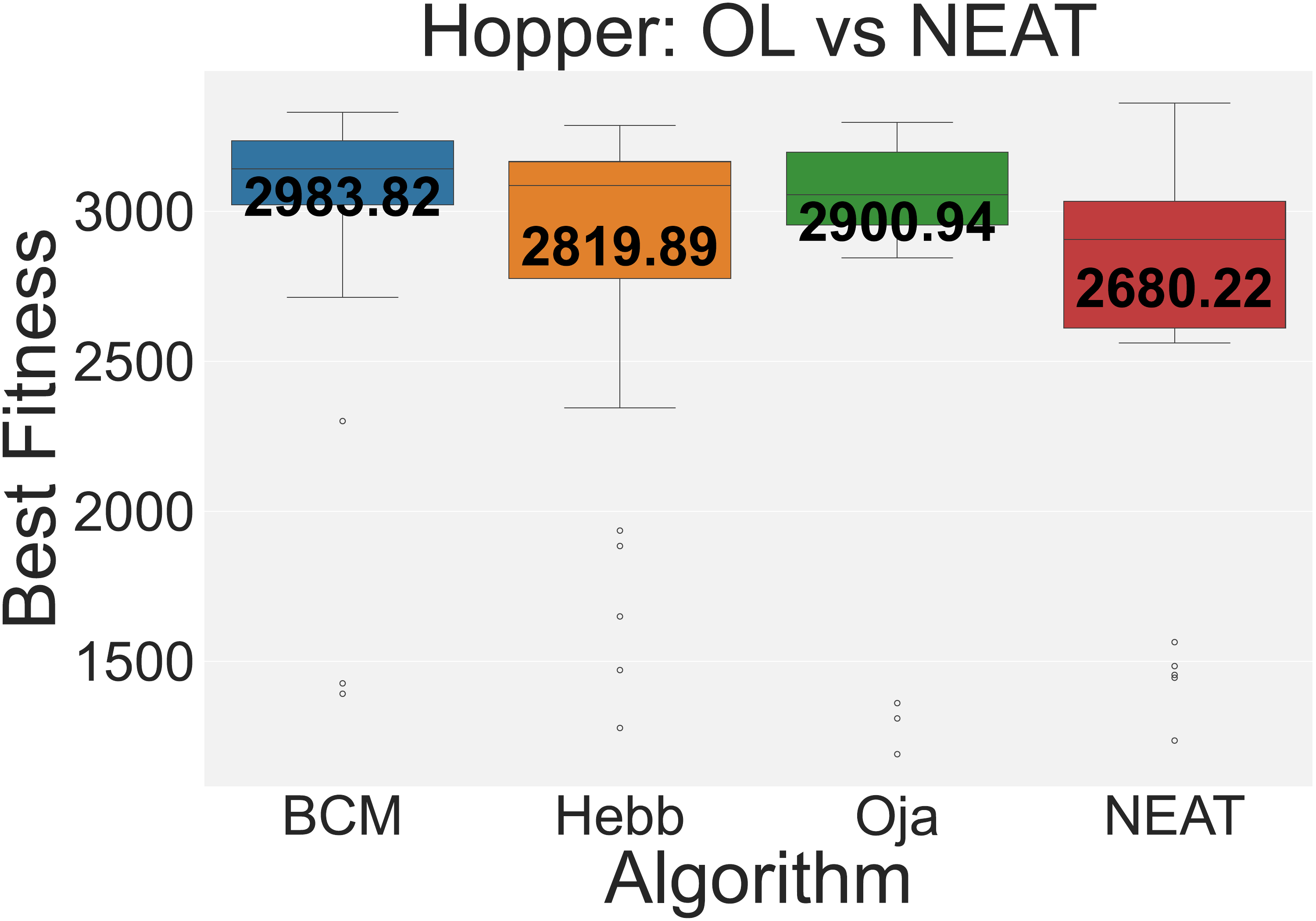}
    \includegraphics[width=0.49\linewidth]{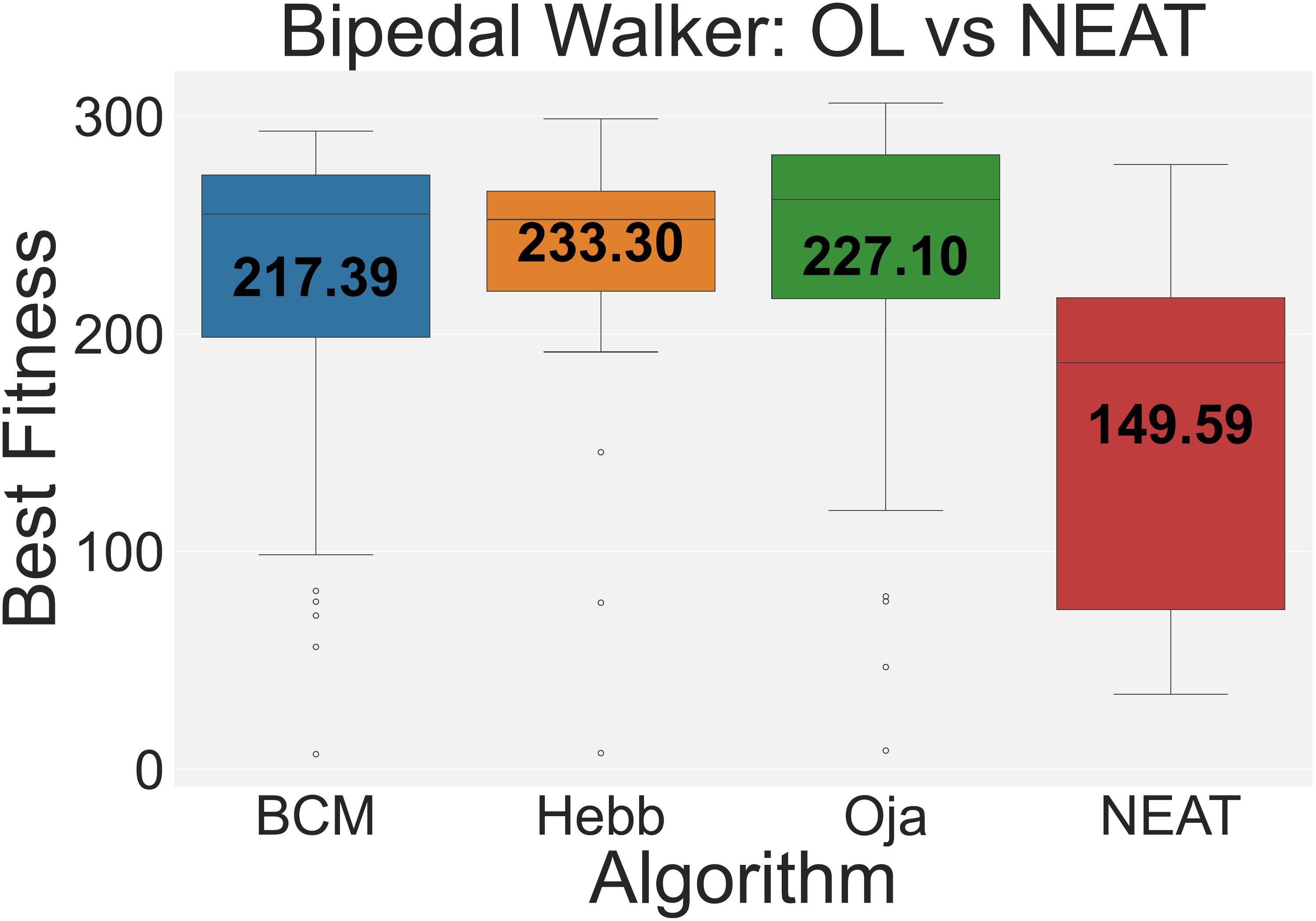}

    \caption{Fitness comparison of NEAT-NEOL with NEAT across four control tasks in 30 independent runs, each task is under a fixed interaction budget $P \times G \times T_{\max}$ environment steps. 
    BCM, Hebb, and Oja learning rules are shown in blue, orange, and green, respectively, while standard NEAT is shown in red.}
    \label{fig:NEOL_comparison_combined}
\end{figure}

\begin{figure}[!ht]
    \centering
    
    \includegraphics[width=0.49\linewidth]{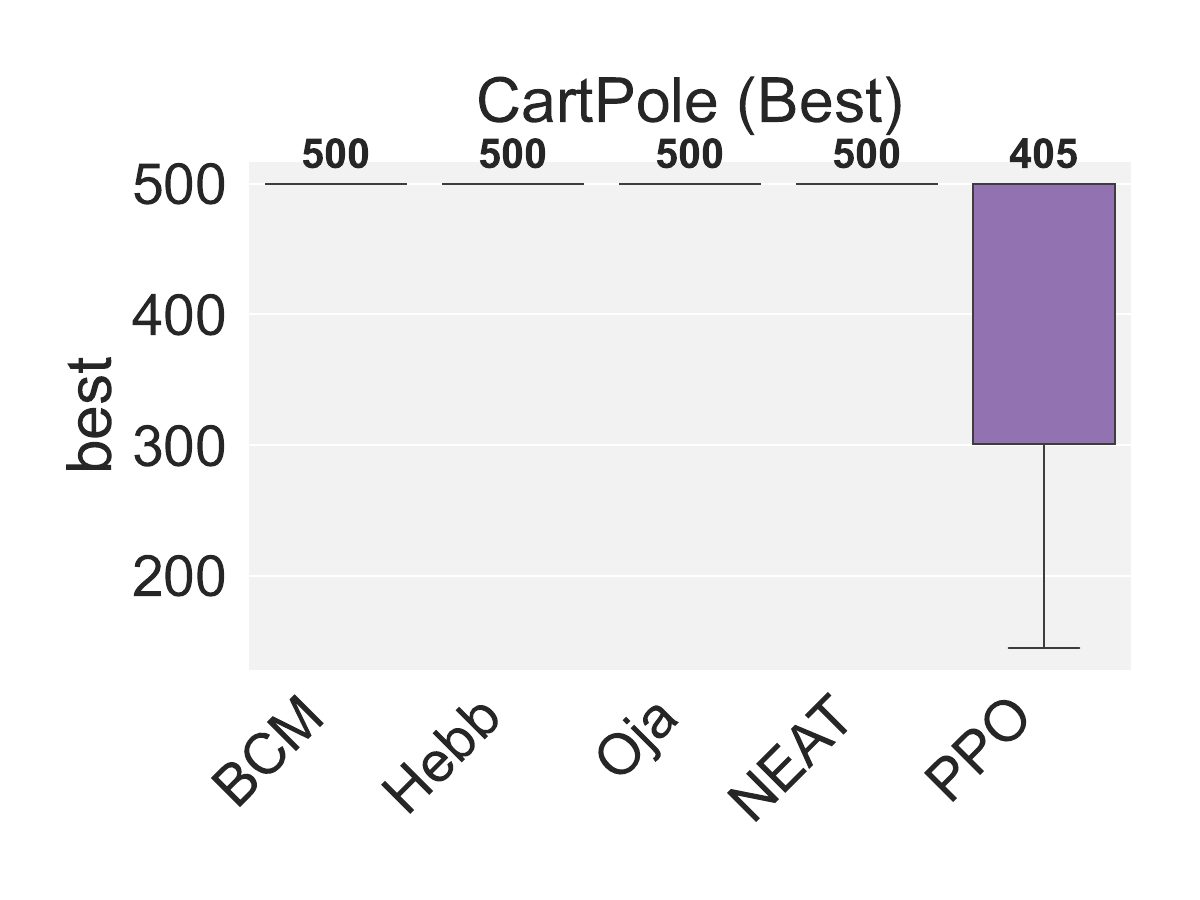}
    \includegraphics[width=0.49\linewidth]{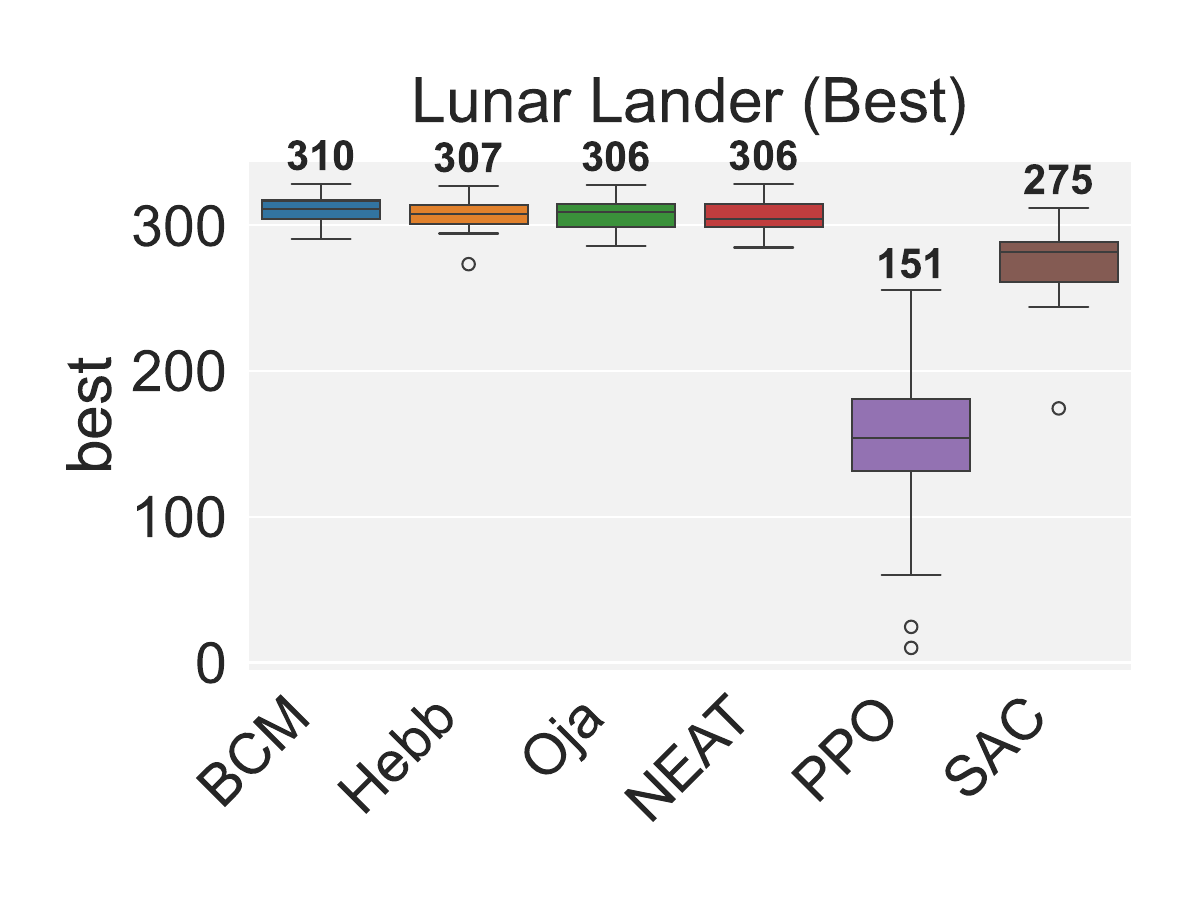}
    \includegraphics[width=0.49\linewidth]{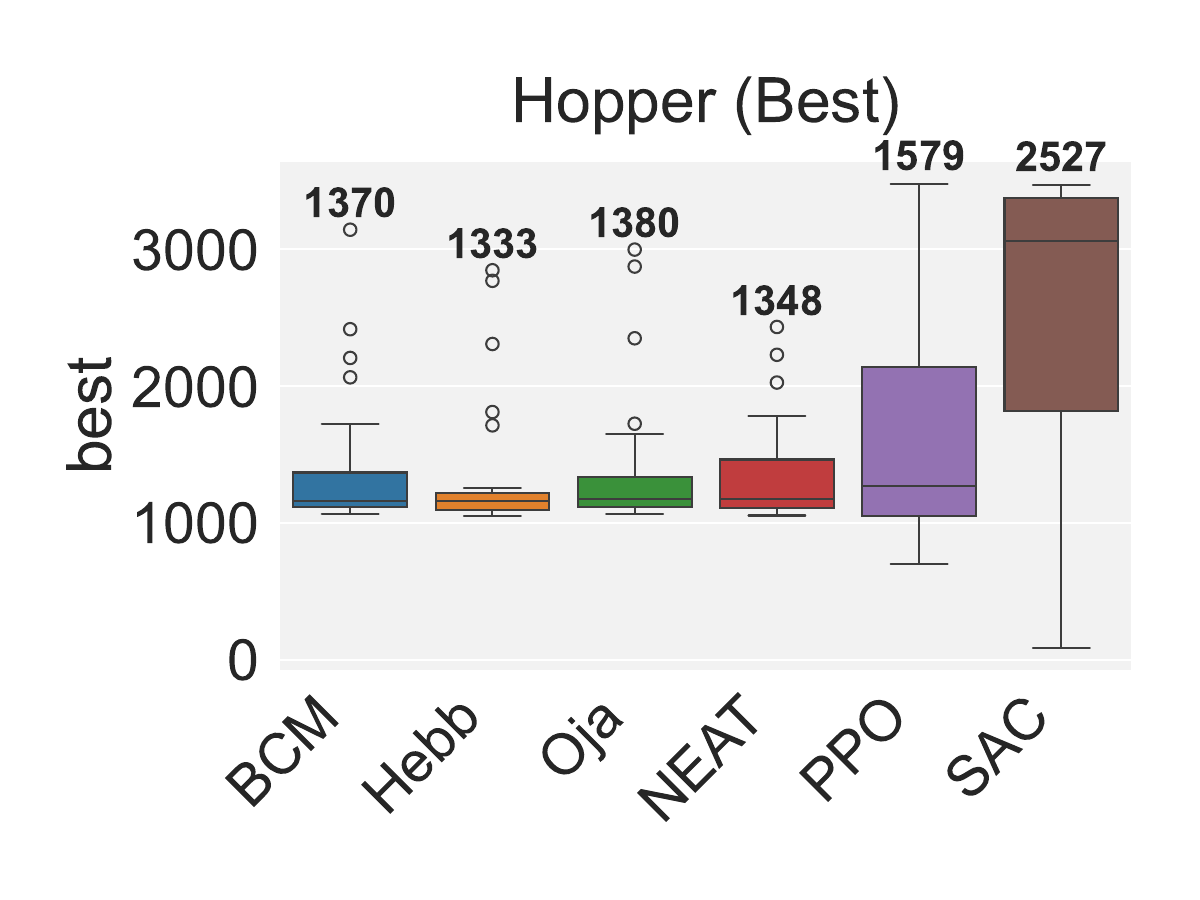}
    \includegraphics[width=0.49\linewidth]{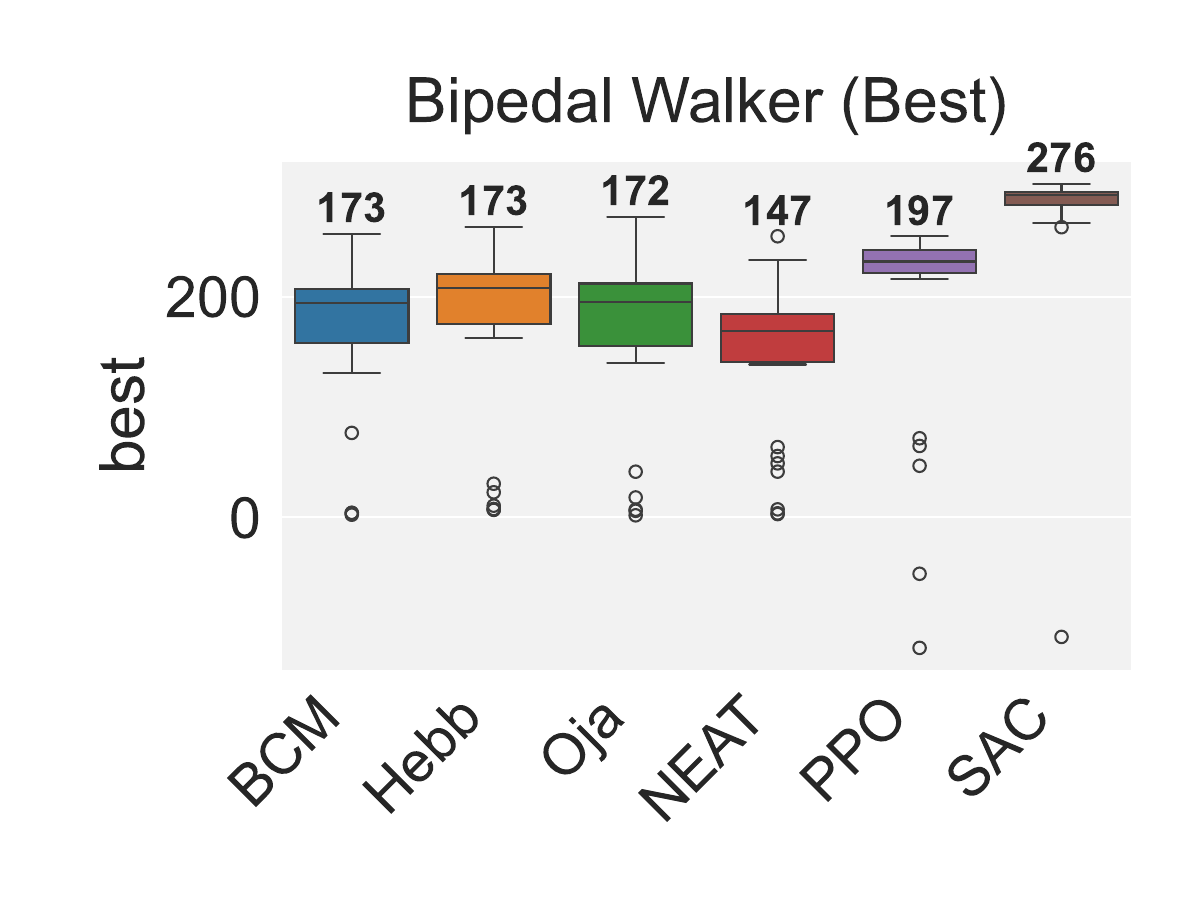}
    
    \caption{Fitness comparison of NEAT-NEOL with NEAT and RLs across four control tasks in 30 independent runs under a fixed interaction budget of 10 million environment steps per seed. 
    For neuroevolution methods, the total budget for each configuration equals $P \times G \times T_{\max}=10^7$; for PPO/SAC, it equals the total training timesteps. 
    Boxplots show episode return aggregated over seeds.}

    \label{fig:NEOL_comparison_combined_RL}
\end{figure}

In the simpler CartPole task, all methods rapidly converge to the maximum fitness, serving as a successful sanity check. 
However, in the more complex environments, significant performance disparities emerge. 
As shown in the convergence plots (Figure~\ref{fig:NEOL_comparison_combined}, top two rows), the NEOL variants consistently achieve higher final fitness scores than standard NEAT. 
The boxplots (Figure~\ref{fig:NEOL_comparison_combined}, bottom two rows) and standard deviations (Table 4) further reveal that while standard NEAT struggles, often resulting in a high variance and numerous low-performing outliers, the NEOL methods achieve a superior median performance. 
Notably, in three continuous control tasks, NEOL not only outperforms NEAT but also exhibits a tighter fitness distribution, indicating higher learning reliability in terms of smaller variance.

The observed performance gains are not coincidental. 
We confirmed their statistical significance using a one-sided Wilcoxon rank-sum test.
The null hypothesis is that the two configurations yield samples from the same distribution; the alternative hypothesis is that the OL method tends to achieve the larger best-fitness than NEAT.
As detailed in Table 5, we reject the null hypothesis at a significance level of $p < 0.05$ for all NEOL variants across all three tasks apart from Cartpole. 
This provides strong evidence that the integration of online learning improves the performance of algorithms.

Next, we compare NEOL with PPO and SAC.
As shown in Figure~\ref{fig:NEOL_comparison_combined_RL}, under the same interaction budget of 10 million (10M) timesteps, NEOL and NEAT outperform PPO and SAC on CartPole and LunarLander, achieving higher fitness with smaller variance, and hence better sample efficiency.
On Hopper and BipedalWalker, PPO and SAC remain strong and outperform both NEAT and NEOL. 
However, NEOL is still more competitive against PPO than standard NEAT.
Overall, these results highlight the potential of NEOL, especially on tasks where fast within-episode adaptation is beneficial.

\subsection{Ablation Studies}

\begin{figure}[!ht]
    \centering
    
    \includegraphics[width=0.48\linewidth]{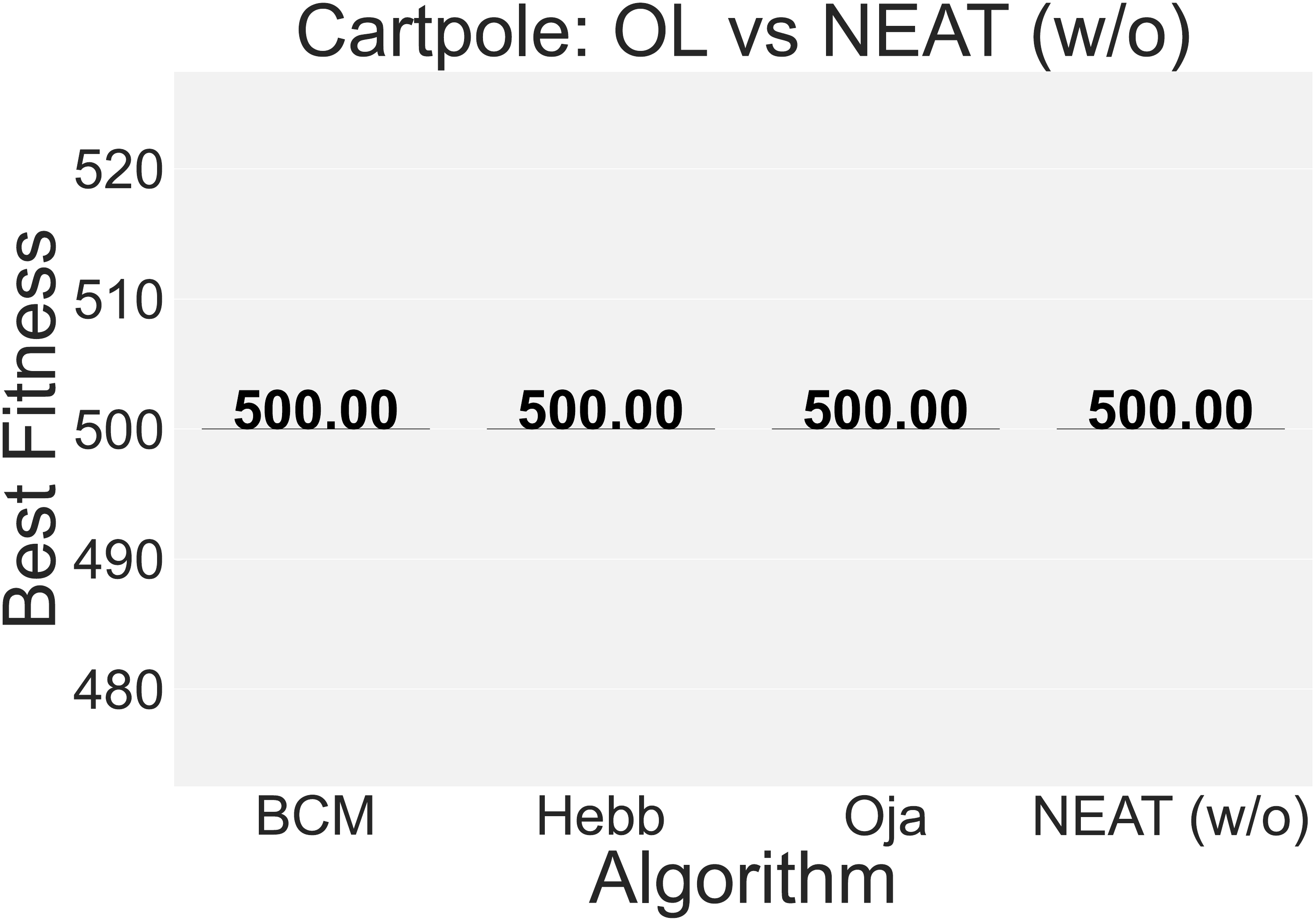}
    \includegraphics[width=0.48\linewidth]{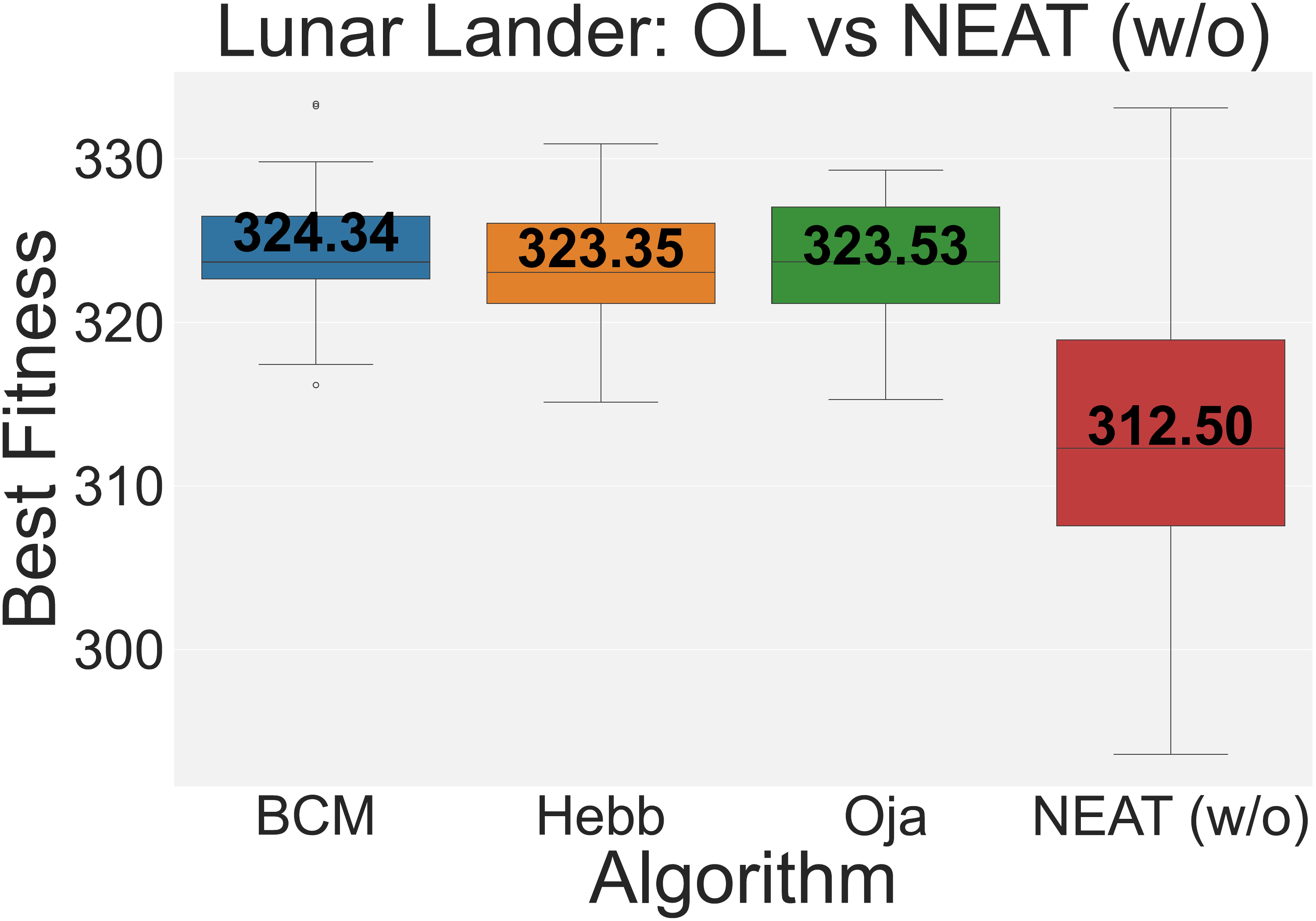}
    \includegraphics[width=0.48\linewidth]{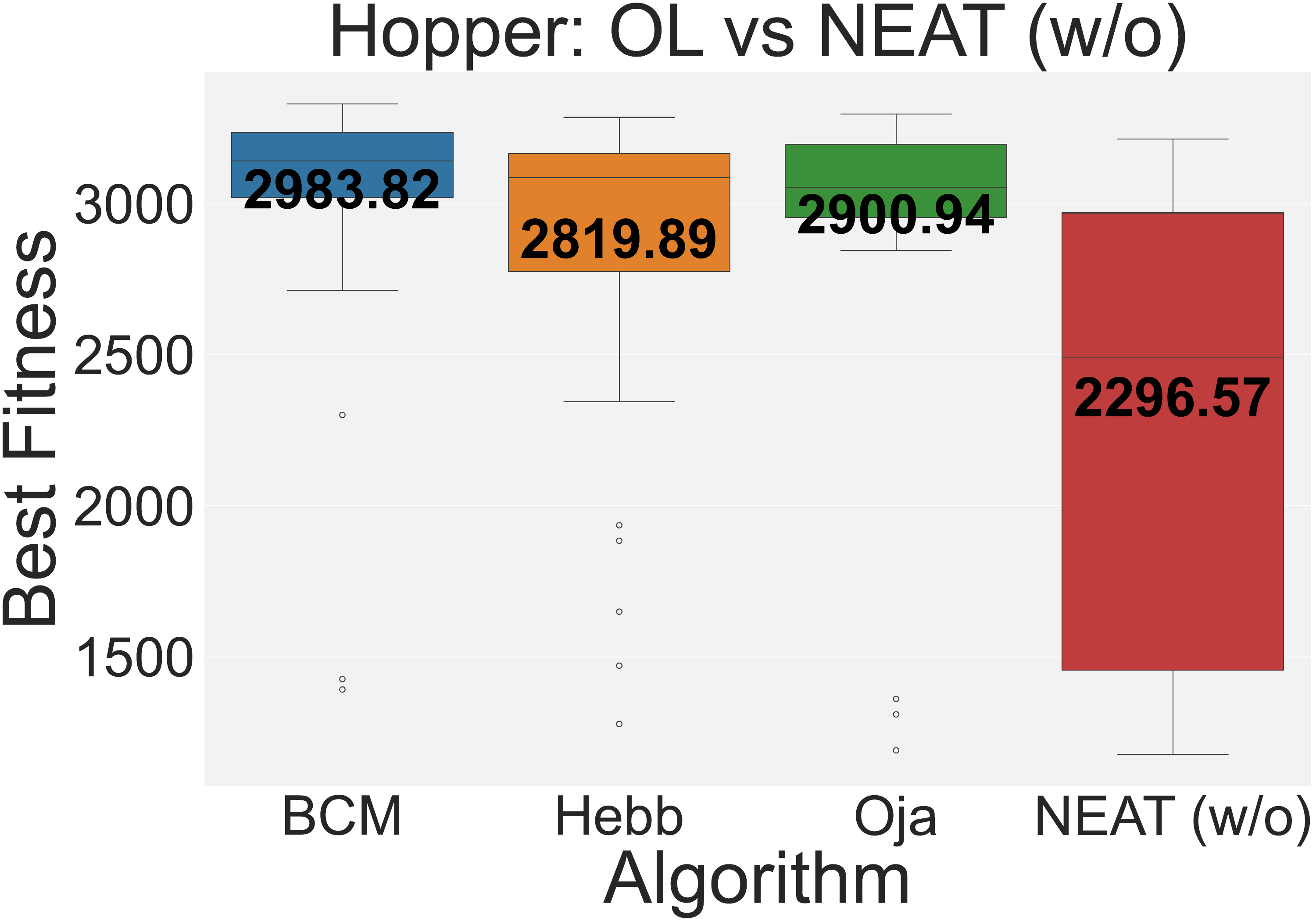}
    \includegraphics[width=0.48\linewidth]{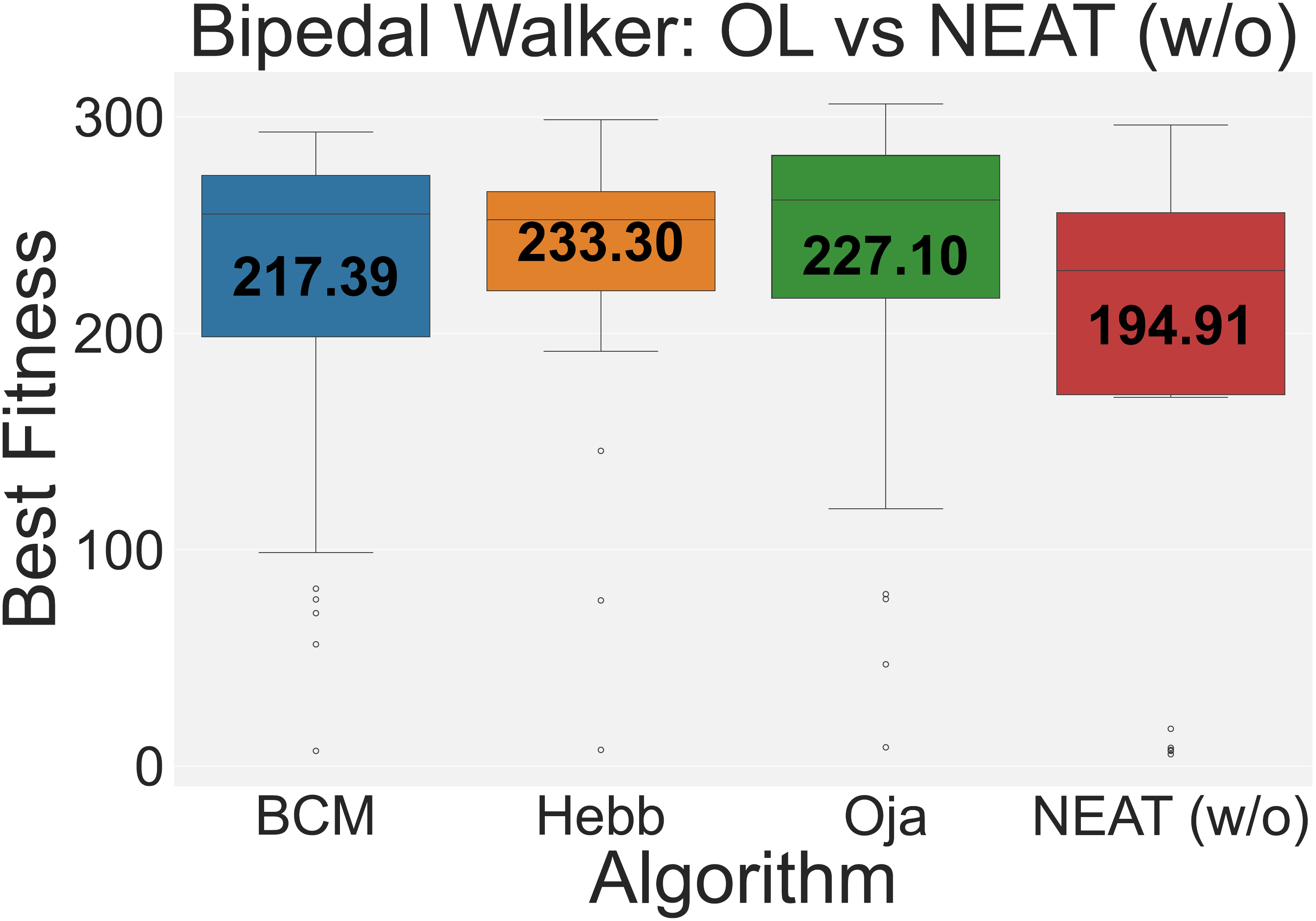}
    
    \caption{Ablation study comparing NEAT-NEOL with NEAT (w/o) across four control tasks in 30 independent runs.
    }
    \label{fig:NEOL_comparison_ablation_combined}
\end{figure}


We ablate the online learning component by disabling weight modulation in NEAT (``NEAT (w/o)''), and all other settings remain unchanged. 
Figure~\ref{fig:NEOL_comparison_ablation_combined} compares the final best-fitness distributions over 30 independent runs on four control tasks.
On CartPole, all methods quickly reach the performance ceiling $500$ as expected, so disabling online learning has no clear effect. 
On Lunar Lander, all NEOL variants (BCM, Hebb, Oja) achieve clearly higher final fitness than NEAT (w/o), with tighter distributions. 
On Hopper, removing online learning results in a significant drop in performance. 
NEAT (w/o) has a much lower median and a wider spread, while the NEOL variants remain stronger overall. 
On Bipedal Walker, NEOL again improves the final fitness compared with NEAT (w/o), and the results are consistent across seeds.
In three continuous control tasks, the boxplots further indicate higher empirical means, higher medians, and tighter interquartile ranges, which support the improvement of online learning via synaptic plasticity.
Overall, these results isolate the contribution of online synaptic plasticity: removing it impairs both the gained fitness and the robustness of learning.



\paragraph{Summary.}

In our two-timescale setup, the outer evolutionary loop searches for good network topologies, while the inner loop uses plasticity rules to adjust weights during online interaction. 
This gives a simple, gradient-free method to assign fitness and fine-tune behaviour within an agent’s lifetime.
The experiment also shows a relatively competitive performance against some strong RL baselines across several tasks.
Moreover, the ablation study supports this across tasks: enabling online plasticity improves final fitness, reduces instability across runs, and leads to better sample efficiency. 
In contrast, removing it (NEAT (w/o)) generally results in weaker final performance and less robust learning. 
Taken together, these findings suggest that online learning via plasticity can improve the performance of NeuroEvolution and make it more robust.

\section{Conclusion and Discussion}

We proposed a general NeuroEvolutionary Online Learning (NEOL) framework, a two-timescale design that separates (i) outer-loop topology/architecture search from (ii) inner-loop, within-episode weight adaptation via reward-modulated local plasticity. 
Using biologically motivated rules (Hebb, Oja, and BCM), NEOL provides a simple, gradient-free mechanism for online weight adaptation. 
Our theoretical analysis shows that the resulting outer-inner loop guarantees an expected sublinear regret under standard boundedness and selection assumptions. 
This serves as a first step towards a more rigorous understanding of these hybrid online learning methods.
Empirically, NEAT-NEOL improves over pure NEAT across four control benchmarks and is relatively competitive with strong RL baselines (PPO/SAC) on several tasks, with notably tighter performance distributions. 
Additional ablation studies and Wilcoxon rank-sum tests further confirm that online plasticity is a key opponent of the improvement. 

Despite these results, NEOL has limitations. 
Theoretically, we only consider the exponential weight update selection for architecture search.
Extending the analysis to broader classes of algorithms presents an exciting and challenging avenue.
Moreover, current theoretical analysis is restricted to the local plasticity rules, including Hebb, Oja and BCM.
We conjecture that sublinear regret is attained for more general local plasticity rules with bounded step sizes. 
Empirically, testing on more diverse environments can strengthen our empirical analysis. 
It is not clear how robust our online weighted adaptation is for more complex environments.
Also, this paper only uses fixed learning rules. 
As the performance algorithm remains sensitive to hyperparameters and task structure, adaptation of learning rules might be the next possible direction.
Together, these future theoretical and empirical extensions would clarify when and why reward-modulated plasticity most effectively enhances neuroevolution, and would move NEOL closer to robust adaptive agents.


\section*{Ethics Statement}
We have no ethical concerns to declare.




\section*{Acknowledgments}
We thank the anonymous reviewers for their helpful reviews.
Shishen is supported by UKRI via grant EP/Y014650/1, as part of the ERC Synergy project OCEAN.
The computations were performed using the University of Warwick high-performance computing (Avon HPC) service.
\section*{Contribution Statement}

Shishen Lin and Yixin Chen contributed equally to this work.
Shishen Lin contributed to the theoretical analysis, experiment design, empirical evaluation and manuscript writing.
Yixin Chen contributed to the algorithm design, implementation, empirical evaluation, and manuscript writing.





\bibliographystyle{named}
\bibliography{ijcai26}

\clearpage
\onecolumn
\section{Appendix}
\appendixtableofcontents

\clearpage
\appendix

\section{Techical Lemmas}

Before proving Lemma~\ref{lem:outer}, we need Hoeffding's lemma. 

\begin{lemma}[Conditional Hoeffding's Lemma~\protect\cite{hoeffding1963probability}]
\label{lem:hoeffding}
Let $\F$ be an $\sigma$-algebra and $Y$ be any random variable such that $Y$ takes values in a bounded interval $[a,b]$. 
Then, for all $\lambda \in \mathbb{R}$, 
\begin{align*}
    \E{e^{\lambda Y} \mid \F} \leq \exp \left(\lambda \E{Y \mid \F}+ \frac{\lambda^2(b-a)^2}{8} \right).
\end{align*}
\end{lemma}

\begin{lemma}[Non-expansiveness property of Euclidean projection~\protect\cite{Zarantonello1971Proj,Bauschke2017}]
\label{lem:nonexpansive}
For any closed and convex set $\mathcal{V}$, we have for any $a, b \in \mathcal{V}$, 
\begin{align*}
    ||\Pi_\mathcal{V}(a)-\Pi_\mathcal{V}(b)||_2 \leq ||a-b||_2.
\end{align*}
\end{lemma}

\section{Theoretical Analysis: Two-Level Regret for the Abstraction of NEOL}
\label{app:toy_theory_neol}

\subsection{Proofs of two standard regret lemmas}
\label{app:toy_lemmas}

\Lemmaone*

Our proof follows a standard regret bound for projected (sub)gradient methods in online convex optimisation \cite{Zinkevich2003OCP}, but instead, we consider a two-timescale setting.
First, we fix $h$ and bound the corresponding inner regret. 

\begin{proof}[Proof of Lemma~\ref{lem:inner}]
Fix any $h\in\mathcal{H}$ and abbreviate
\[
\mathcal{W}:=\mathcal{W}_h,\quad
w_t:=w_t^{(h)},\quad
u_t:=u_t^{(h)},\quad
L_t(\cdot):=L_t^{(h)}(\cdot).
\]
By definition,
\[
L_t(w) = -\langle w, u_t\rangle
\]
is linear (hence convex) in $w$, and a (sub)gradient at any $w$ is
\[
g_t \in \partial L_t(w) = \{-u_t\}.
\]
The projected update \eqref{eq:generic_local_update} can therefore be written as the standard projected subgradient step:
\[
w_{t+1}
= \Pi_{\mathcal{W}}\!\big(w_t + \eta u_t\big)
= \Pi_{\mathcal{W}}\!\big(w_t - \eta g_t\big).
\]

Let
\[
w^\star \in \arg\min_{w\in\mathcal{W}} \sum_{t=1}^T L_t(w)
\]
be an optimal comparator.

Recall Lemma~\ref{lem:nonexpansive} holds on $\mathcal{W}$, which is convex and compact (hence, closed and bounded).

For the first equality, note that from defintion of Euclidean projection, we have for any $w \in \mathcal{W}, \Pi_{\mathcal{W}}(w)=w$ (as the Euclidean projection operator is the identity map in $\mathcal{W}$).
\begin{align}
\|w_{t+1}-w^\star\|_2^2
&= \big\|\Pi_{\mathcal{W}}(w_t+\eta u_t) - \Pi_{\mathcal{W}}(w^\star)\big\|_2^2 \nonumber\\ \intertext{Using Lemma~\ref{lem:nonexpansive} onto a convex and compact set $\mathcal{W}$, for any $t$,}
&\le \| (w_t+\eta u_t) - w^\star\|_2^2 \nonumber\\
&= \|w_t-w^\star\|_2^2 + 2\eta\langle u_t, w_t-w^\star\rangle + \eta^2\|u_t\|_2^2.
\label{eq:proj_expand}
\end{align}
Rearranging \eqref{eq:proj_expand} gives
\begin{equation}
- \langle u_t, w_t-w^\star\rangle
\le
\frac{\|w_t-w^\star\|_2^2-\|w_{t+1}-w^\star\|_2^2}{2\eta}
+
\frac{\eta}{2}\|u_t\|_2^2.
\label{eq:one_step_bound}
\end{equation}
Now observe that
\begin{align*}
L_t(w_t)-L_t(w^\star)
&= -\langle w_t,u_t\rangle + \langle w^\star,u_t\rangle \\ \intertext{Using linearity of inner product gives}
&= \langle w^*-w_t, u_t \rangle  \\ \intertext{Using the fact that the inner product is commutative over $\mathbb{R}$ gives}
&= \langle u_t, w^*-w_t \rangle \\
&= -\langle u_t, w_t-w^\star\rangle.
\end{align*}
Combining this identity with \eqref{eq:one_step_bound} yields, for every $t$,
\[
L_t(w_t)-L_t(w^\star)
\le
\frac{\|w_t-w^\star\|_2^2-\|w_{t+1}-w^\star\|_2^2}{2\eta}
+
\frac{\eta}{2}\|u_t\|_2^2.
\]
Summing over $t=1,\dots,T$ on the left-hand side:
\begin{align*}
\sum_{t=1}^T \big(L_t(w_t)-L_t(w^\star)\big)
&\le
\frac{\|w_1-w^\star\|_2^2-\|w_{T+1}-w^\star\|_2^2}{2\eta}
+
\frac{\eta}{2}\sum_{t=1}^T\|u_t\|_2^2\\
&\le
\frac{\|w_1-w^\star\|_2^2}{2\eta}
+
\frac{\eta}{2}\sum_{t=1}^T\|u_t\|_2^2,
\end{align*}
since $\|w_{T+1}-w^\star\|_2^2\ge 0$. Finally, because $w_1,w^\star\in\mathcal{W}$ and $\mathrm{diam}(\mathcal{W})\le D$ (Assumption~3),
\[
\|w_1-w^\star\|_2 \le D
\quad\Rightarrow\quad
\|w_1-w^\star\|_2^2 \le D^2,
\]
which proves:
\[
\sum_{t=1}^T L_t(w_t) - \min_{w\in\mathcal{W}}\sum_{t=1}^T L_t(w)
\le
\frac{D^2}{2\eta} + \frac{\eta}{2}\sum_{t=1}^T \|u_t\|_2^2.
\]

If additionally $\|u_t\|_2\le U$ for all $t$ (Assumptions 2-4 imply the local plasticity update is bounded by some constant $U$ w.r.t $T$), then $\sum_{t=1}^T\|u_t\|_2^2 \le U^2T$, giving 
\begin{align*}
    \sum_{t=1}^T L_t^{(h)}\!\big(w_t^{(h)}\big)
- \min_{w\in\mathcal{W}_h}\sum_{t=1}^T L_t^{(h)}(w)
\;\le\; \frac{D^2}{2\eta} + \frac{\eta U^2 T}{2}.
\end{align*}
Choosing $\eta = D/(U\sqrt{T})$ yields
\[
\frac{D^2}{2\eta} + \frac{\eta U^2T}{2}
=
\frac{D^2}{2}\cdot\frac{U\sqrt{T}}{D} + \frac{1}{2}\cdot\frac{D}{U\sqrt{T}}\cdot U^2T
=
DU\sqrt{T}.
\]
This completes the proof.
\end{proof}

\Lemmatwo*

The guarantee holds for any realised loss sequence in $[0,1]$.

\begin{proof}[Proof of Lemma~\ref{lem:outer}]
We follow the standard exponential-weights (Hedge) analysis in
\cite{freund1997decision}.
For this lemma we work in the full-information setting (Assumption~7) with bounded losses
$\ell_t(h)\in[0,1]$ (Assumption~2), where we abbreviate
\[
\ell_t(h) := \ell_t\big(h, w_t^{(h)}\big).
\]

Let us define unnormalised weights $q_t(h)\ge 0$ and normaliser $Q_t := \sum_{h\in\mathcal{H}} q_t(h)$.
Define
\[
p_t(h) := \frac{q_t(h)}{Q_t},
\qquad
q_{t+1}(h) := q_t(h)\exp\!\big(-\gamma\,\ell_t(h)\big),
\]
which is equivalent to \eqref{eq:ew_update}.
Assume uniform initialisation $q_1(h)=1$ for all $h$, so $Q_1=\sum_{h\in \mathcal{H}}q_1(h)= M$ (Assumption 1).

By the update,
\begin{align*}
Q_{t+1}
&= \sum_{h} q_{t+1}(h)
= \sum_{h} q_t(h)\exp\!\big(-\gamma\,\ell_t(h)\big)
= Q_t \sum_{h} p_t(h)\exp\!\big(-\gamma\,\ell_t(h)\big),
\end{align*}
hence
\begin{align*}
\log\frac{Q_{t+1}}{Q_t}
= \log\left(\sum_{h} p_t(h)\exp\!\big(-\gamma\,\ell_t(h)\big)\right) = \log \left(\mathbb{E}_{h\sim p_t}[e^{-\gamma\,\ell_t(h)} \mid \F_t] \right).
\end{align*}
Now, we want to bound the log-expectation part.
We use Lemma~\ref{lem:hoeffding} (Hoeffding's lemma) with $Y=-X$ (note $Y=-\ell_t(h) \in [-1,0]$) and take $\log$ on both sides.
\begin{equation}
\label{eq:hoeffding_lemma}
\log\left(\mathbb{E}\left[e^{-\gamma X} \mid \F_t\right]\right)
\le -\gamma\,\mathbb{E}[X\mid \F_t] + \frac{\gamma^2}{8},
\qquad \text{for any random variable } X\in[0,1].
\end{equation}
Applying \eqref{eq:hoeffding_lemma} to the random variable $X=\ell_t(h)$ under $h\sim p_t$
gives
\begin{equation}
\label{eq:upper_log_ratio}
\log\frac{Q_{t+1}}{Q_t}
\le
-\gamma\,\mathbb{E}_{h\sim p_t}[\ell_t(h)\mid \F_t] + \frac{\gamma^2}{8}.
\end{equation}
Summing \eqref{eq:upper_log_ratio} over $t=1,\dots,T$ yields
\begin{equation}
\label{eq:upper_sum_logQ}
\log\frac{Q_{T+1}}{Q_1}
\le
-\gamma\sum_{t=1}^T \mathbb{E}_{h\sim p_t}[\ell_t(h)\mid \F_t]
+ \frac{\gamma^2T}{8}.
\end{equation}

For any fixed $h^\star\in\mathcal{H}$, since $q_{T+1}(h) \geq 0$ for all $h\in \mathcal{H}$, we have
\[
Q_{T+1} = \sum_h q_{T+1}(h) \ge q_{T+1}(h^\star).
\]
Using the multiplicative updates,
\[
q_{T+1}(h^\star)
=
q_1(h^\star)\exp\!\left(-\gamma\sum_{t=1}^T \ell_t(h^\star)\right).
\]
With $q_1(h^\star)=1$ and taking log on both sides, we get
\begin{equation}
\label{eq:lower_logQ}
\log Q_{T+1}
\ge
-\gamma\sum_{t=1}^T \ell_t(h^\star).
\end{equation}

Since $Q_1=M$, we have $\log\frac{Q_{T+1}}{Q_1} = \log Q_{T+1} - \log M$.
Combining \eqref{eq:upper_sum_logQ} and \eqref{eq:lower_logQ} yields
\begin{align*}
-\gamma\sum_{t=1}^T \ell_t(h^\star) - \log M 
\overset{(7)}{\leq} \log Q_{T+1} - \log M = \log \frac{Q_{T+1}}{Q_1} \overset{(8)}{\leq} -\gamma\sum_{t=1}^T \mathbb{E}_{h\sim p_t}[\ell_t(h)\mid \F_t] + \frac{\gamma^2T}{8}.
\end{align*}
In short, we have
\[
-\gamma\sum_{t=1}^T \ell_t(h^\star) - \log M
\le
-\gamma\sum_{t=1}^T \mathbb{E}_{h\sim p_t}[\ell_t(h)\mid \F_t]
+ \frac{\gamma^2T}{8}.
\]
Rearranging gives, for any $h^\star\in\mathcal{H}$,
\[
\sum_{t=1}^T \mathbb{E}_{h\sim p_t}[\ell_t(h)\mid \F_t]
-\sum_{t=1}^T \ell_t(h^\star)
\le
\frac{\log M}{\gamma} + \frac{\gamma T}{8}.
\]
Taking $h^\star \in \arg\min_{h\in\mathcal{H}}\sum_{t=1}^T \ell_t(h)$ proves 

\begin{align*}
\sum_{t=1}^T \mathbb{E}_{h_t\sim p_t}\!\left[\ell_t\big(h_t,w_t^{(h_t)}\big)\mid \F_t\right]- \min_{h\in\mathcal{H}} \sum_{t=1}^T \ell_t\big(h,w_t^{(h)}\big)  \le \frac{\log M}{\gamma} + \frac{\gamma T}{8}.
\end{align*}

Finally, choosing $\gamma=\sqrt{8\log M / T}$ yields
\[
\frac{\log M}{\gamma} + \frac{\gamma T}{8}
=
\sqrt{\frac{T\log M}{8}} + \sqrt{\frac{T\log M}{8}}
=
\sqrt{\frac{T\log M}{2}},
\]
which proves:
\begin{align*}
    \sum_{t=1}^T \mathbb{E}_{h_t\sim p_t}\!\left[\ell_t\big(h_t,w_t^{(h_t)}\big)\mid \F_t\right]
- \min_{h\in\mathcal{H}} \sum_{t=1}^T \ell_t\big(h,w_t^{(h)}\big)
\;\le\; \sqrt{\frac{T\log M}{2}}.
\end{align*}
\end{proof}

\subsection{Proof of Main Theorem: Two-Level Regret Bound}
\label{thm:toy_theorem}

\Mainthm*


\begin{proof}
Let $(h^\star,w^\star)\in\arg\min_{h\in\mathcal{H}}\min_{w\in\mathcal{W}_h}\sum_{t=1}^T \ell_t(h,w)$ be a hindsight-optimal pair.
Recall the regret definition
\[
\mathbb{E}[R_T]
=
\mathbb{E}\Bigg[
\sum_{t=1}^T \ell_t\big(h_t,w_t^{(h_t)}\big)
-\min_{h\in\mathcal{H}}\min_{w\in\mathcal{W}_h}\sum_{t=1}^T \ell_t(h,w)
\Bigg].
\]
Since $\min_{h,w}\sum_{t=1}^T \ell_t(h,w)=\sum_{t=1}^T \ell_t(h^\star,w^\star)$, we can write
\begin{align*}
\mathbb{E}[R_T]
&=
\mathbb{E}\Bigg[
\sum_{t=1}^T \ell_t\big(h_t,w_t^{(h_t)}\big)
-\sum_{t=1}^T \ell_t(h^\star,w^\star)
\Bigg] \\
&=
\mathbb{E}\Bigg[
\underbrace{\sum_{t=1}^T \ell_t\big(h_t,w_t^{(h_t)}\big)
-\sum_{t=1}^T \ell_t\big(h^\star,w_t^{(h^\star)}\big)}_{\text{outer regret}}
+
\underbrace{\sum_{t=1}^T \ell_t\big(h^\star,w_t^{(h^\star)}\big)
-\sum_{t=1}^T \ell_t(h^\star,w^\star)}_{\text{inner regret for $h^*$}}
\Bigg].
\end{align*}

First, we consider the outer regret.
For each round $t$, abbreviate $\ell_t(h):=\ell_t\big(h,w_t^{(h)}\big)$.
Note that
\[
\mathbb{E}\!\left[\ell_t\big(h_t,w_t^{(h_t)}\big)\right]
=
\mathbb{E}\!\left[\ell_t(h_t)\right]
=
\sum_{h\in\mathcal{H}} p_t(h)\,\ell_t(h)
=
\mathbb{E}_{h_t\sim p_t}[\ell_t(h_t)].
\]

Let $\mathcal{F}_t$ denote the $\sigma$-algebra generated by all randomness up to the start of round $t$ (in particular, it contains the realised loss vector $\{\ell_t(h)\}_{h\in\mathcal{H}}$)).
Therefore, by the tower property,
\begin{align*}
\mathbb{E}\Bigg[\sum_{t=1}^T \ell_t\big(h_t,w_t^{(h_t)}\big)\Bigg]
=\mathbb{E}\Bigg[\sum_{t=1}^T \mathbb{E}\!\left[\ell_t\big(h_t,w_t^{(h_t)}\big)\mid \F_t \right] \Bigg] 
=\mathbb{E}\Bigg[\sum_{t=1}^T \mathbb{E}_{h_t\sim p_t}[\ell_t(h_t)\mid \F_t ]\Bigg].
\end{align*}

Applying Lemma~\ref{lem:outer} conditionally on the realised losses yields
\[
\sum_{t=1}^T \mathbb{E}_{h_t\sim p_t}[\ell_t(h_t) \mid \F_t]
-\min_{h\in\mathcal{H}}\sum_{t=1}^T \ell_t(h)
\le \sqrt{\frac{T\log M}{2}},
\]
and taking expectation over $\F_T$ preserves the inequality. 
Hence,
\[
\mathbb{E}\Bigg[\sum_{t=1}^T \ell_t\big(h_t,w_t^{(h_t)}\big)\Bigg]
-\mathbb{E}\Bigg[\min_{h\in\mathcal{H}}\sum_{t=1}^T \ell_t\big(h,w_t^{(h)}\big)\Bigg]
\le
\sqrt{\frac{T\log M}{2}}.
\]
In particular, since $\min_{h}\sum_{t=1}^T \ell_t(h)\le \sum_{t=1}^T \ell_t(h^\star)$ for any fixed $h^*$, we obtain
\[
\mathbb{E}\Bigg[\sum_{t=1}^T \ell_t\big(h_t,w_t^{(h_t)}\big)
-\sum_{t=1}^T \ell_t\big(h^\star,w_t^{(h^\star)}\big)\Bigg]
\le \mathbb{E}\Bigg[\sum_{t=1}^T \ell_t\big(h_t,w_t^{(h_t)}\big)\Bigg]
-\mathbb{E}\Bigg[\min_{h\in\mathcal{H}}\sum_{t=1}^T \ell_t\big(h,w_t^{(h)}\big)\Bigg] \le
\sqrt{\frac{T\log M}{2}}.
\]

Next, we consider the inner regret for $h^*$.
Fix $h^\star$ and define the surrogate loss
$L_t^{(h^\star)}(w):=-\langle w,u_t^{(h^\star)}\rangle$.
By Assumption~A5 (surrogate domination almost surely), for all $w\in\mathcal{W}_{h^\star}$,
\[
\ell_t\big(h^\star,w_t^{(h^\star)}\big)-\ell_t(h^\star,w)
\le
C\Big(L_t^{(h^\star)}\big(w_t^{(h^\star)}\big)-L_t^{(h^\star)}(w)\Big).
\]
Summing over $t=1,\dots,T$ and minimising the right-hand side over $w\in\mathcal{W}_{h^\star}$ yields
\[
\sum_{t=1}^T \ell_t\big(h^\star,w_t^{(h^\star)}\big)
-\min_{w\in\mathcal{W}_{h^\star}}\sum_{t=1}^T \ell_t(h^\star,w)
\le
C\left(
\sum_{t=1}^T L_t^{(h^\star)}\big(w_t^{(h^\star)}\big)
-\min_{w\in\mathcal{W}_{h^\star}}\sum_{t=1}^T L_t^{(h^\star)}(w)
\right).
\]
Applying Lemma~\ref{lem:inner} with $\eta=D/(U\sqrt{T})$ gives
\[
\sum_{t=1}^T \ell_t\big(h^\star,w_t^{(h^\star)}\big)
-\min_{w\in\mathcal{W}_{h^\star}}\sum_{t=1}^T \ell_t(h^\star,w)
\le
C\,DU\sqrt{T} \quad \text{holds almost surely.}
\]
Taking expectation preserves the bound.

Combining the two bounds gives
\[
\mathbb{E}[R_T]
\le
\sqrt{\frac{T\log M}{2}} + C\,DU\sqrt{T},
\]
which is $O(\sqrt{T})$.
\end{proof}

\clearpage
\section{Pseudo-Code and More Experiments}
\label{sec:appendix}
\subsection{The Usage of LLM}
We disclose our concrete use of large language models (LLMs) in line with the IJCAI policy. 
LLMs are not authors; the human authors take full responsibility for all content.
We mainly use LLM for writing assistance.
We used LLM to polish and reorganise author-written text for academic style in academic English. 
Our prompt instructed the model to:
\begin{snugshade}
\begin{verbatim}
Correct spelling, grammar, clarity, concision, and overall 
readability;
First return a polished paragraph, 
then a markdown table enumerating each edit with justification;
Preserve citation strings exactly as written; 
Avoid unnecessary \emph{};
Present the paragraph’s intended logic before rewriting to ensure 
coherence.
\end{verbatim}
\end{snugshade}
These outputs were treated as suggestions; final language and structure were decided by the authors after review. 
No claims, proofs, or empirical results originated from the LLM.
Also, we use LLM for helping with debugging and visualisation assistance.
We used LLM to (i) explain error messages, (ii) suggest small code fixes, and (iii) draft the scripts for result visualisation (e.g., plotting scripts). 
All suggested code was reviewed, adapted or re-implemented where non-trivial, and covered by tests. 
Algorithmic design, hyperparameters, and reported results were chosen by the authors.

\subsection{Algorithm Pseudo-Code}

The online update follows one of the following local rules, where $x$ and $y$ denote pre- and post activities, $w$ is a synaptic weight, and $\theta$ is a slow activity-dependent threshold in BCM:

\begin{algorithm}[!t]
\caption{\text{Decoupling NEAT-based NEOL}}
\label{alg:NEOL_main_loop}
\begin{algorithmic}[1]
\Input Generations: $G \in \mathbb{N}_{>0}$; Population size: $P \in \mathbb{N}_{>0}$; NeuroEvolution config: $\Theta_0$ (Described in Algorithm~\ref{alg:speciate_neat}). Parameters for fitness evaluation: Episodes $N \in \mathbb{N}_{>0}$; Max steps $T_{\max} \in \mathbb{N}_{>0}$; Learning rule $\mathcal{L} \in \{\text{Hebb, Oja, BCM}\}$; Plasticity rate $\eta \in \mathbb{R}_{>0}$; Reward scaling $\beta \in \mathbb{R}_{>0}$; Environment \text{env}.
\Output Best evolved genome $g^*$.

\State $\mathcal{P} = \textsc{InitialisePopulation}(\Theta_0, P)$
\For{gen $\in \{1, \dots, G\}$}
    \ForAll{genome $g_i \in \mathcal{P}$}
        \State $g_i.\text{fitness} = \textsc{Rollout}(g_i, \mathcal{L}, \eta, \beta, N, T_{\max}, \text{env})$ \Comment{See Alg~\ref{alg:NEOL_online}}
    \EndFor
    \State $\mathcal{P} = \textsc{Reproduce}(\mathcal{P})$ \Comment{For NEAT: See Alg~\ref{alg:speciate_neat} and ~\ref{alg:reproduce_neat}}

\EndFor
\State \Return best genome $g^*$ from final population $\mathcal{P}$
\end{algorithmic}
\end{algorithm}

\begin{equation*}
    \textsc{weight\_update}(x, y, w, r, \beta, \eta_w) := 
    \begin{cases} 
        w + \eta_w \cdot \beta r \cdot y \cdot (x - y \cdot w) & \text{if Oja} \\ 
        w + \eta_w \cdot x \cdot y \cdot \beta r & \text{if Hebb} \\ 
        w + \eta_w y(y-\theta) x \beta r & \text{if BCM} \\ 
        w & \text{otherwise} 
    \end{cases}
\end{equation*}
\begin{algorithm}[!t]
\caption{Online Rollout}
\label{alg:NEOL_online}
\begin{algorithmic}[1]
\Input Genome: $g$; Learning rule: $\mathcal{L}$; Plasticity rate: $\eta$; Reward scaling factor: $\beta$; Episodes: $N$; Max steps: $T_{\max}$; Environment: \textit{env}.
(Defined in Algorithm~\ref{alg:NEOL_main_loop})
\Output Average episode reward (fitness).

\State episode\_rewards $= []$
\For{$episode \in \{1, \dots, N\}$}
    \State net $= \textsc{CreateNetworkFromGenome}(g)$
    \State $s = \text{env.reset}()$\Comment{Reset state}
    \State $R_{ep} = 0$
    \For{$t \in \{1, \dots, T_{\max}\}$}
        \State $\hat{a} = \text{net.forward}(s)$ \Comment{Generate the action from policy network}
        \State $a = \operatorname{clip}(\hat{a}, -1, 1)$
        \State $(s', r, \text{done}) = \text{env.step}(a)$\Comment{Get the state and reward after apply the action}
        \State $r_{scaled} = r \cdot \beta$
        \State net.\textsc{weight\_update}($\mathcal{L}, \eta, r_{scaled}$)\Comment{Update weight with specific learning rule}
        \State $R_{ep} = R_{ep} + r$
        \State $s = s'$
        \If{done} \textbf{break} \EndIf
    \EndFor
    \State episode\_rewards.append($R_{ep}$)
\EndFor
\State \Return Average cumulative episode rewards
\end{algorithmic}
\end{algorithm}

\begin{algorithm}[!t]
\caption{Speciate (NEAT)~\protect\cite{Stanley2002neat}}
\label{alg:speciate_neat}
\begin{algorithmic}[1]
\Input Population with raw fitness: $\mathcal{P}$;
NEAT config $\Theta_{\text{NEAT}}$ (compatibility coeffs $c_1,c_2,c_3$, threshold $\delta_t$, survival threshold $\rho$, elitism $E$, stagnation limit $G_{\mathrm{stag}}$, min\_species\_size, etc.).
\Output Species set $\mathcal{S}$; population $\mathcal{P}$ with \emph{adjusted} fitness.

\State $\mathcal{S} = \emptyset$ \Comment{if previous species exist, reuse their representatives}
\For{genome $g \in \mathcal{P}$}
  \State Compute compatibility distance to each representative $r_s$:
  \[
    \delta(g,r_s) = c_1\frac{E + D}{N} + c_3 \cdot \overline{|w_g - w_{r_s}|}
  \]
  \State Assign $g$ to $\arg\min_s \delta(g,r_s)$ if $\min_s \delta \le \delta_t$; else create a new species with $g$.
\EndFor

\For{species $s \in \mathcal{S}$}
  \State Sort members by raw fitness (desc), record champion.
  \State Update $s$’s best-so-far and stagnant counter; mark for removal if $>$ $G_{\mathrm{stag}}$ (optionally keep global best).
\EndFor
\State Remove stagnant species; ensure each remaining species has $\ge 1$ member.

\State \textbf{(Explicit sharing)} For each species $s$ and $g\in s$:
\[
  f^{\mathrm{adj}}(g) = \frac{f(g)}{|s|}, \qquad
  F^{\mathrm{adj}}(s) = \frac{1}{|s|}\sum_{g\in s} f^{\mathrm{adj}}(g).
\]
\State \Return $(\mathcal{S},\mathcal{P})$.
\end{algorithmic}
\end{algorithm}

\begin{algorithm}[!t]
\caption{Reproduce (NEAT)~\protect\cite{Stanley2002neat}}
\label{alg:reproduce_neat}
\begin{algorithmic}[1]
\Require Species set $\mathcal{S}$ and adjusted fitness $F^{\mathrm{adj}}(s)$ computed by Alg.~\ref{alg:speciate_neat}.
\Input Population $\mathcal{P}$ with adjusted fitness; NEAT config $\Theta_0$ (elitism $E$, survival threshold $\rho$, crossover prob. $p_c$, add-connection $m_\ell$, add-node $m_n$, weight-mutation mode \texttt{WM\_MODE}); target size $P$.
\Output Next-generation population $\mathcal{P}_{\mathrm{new}}$.

\State $A =\sum_{s \in \mathcal{S}} F^{\mathrm{adj}}(s)$
\For{each $s \in \mathcal{S}$}
  \State $\mathrm{spawn}(s) =\max(\text{min\_species\_size},~ \mathrm{round}(P \cdot F^{\mathrm{adj}}(s) / A))$
\EndFor
\State Renormalise $\mathrm{spawn}(\cdot)$ so that $\sum_{s} \mathrm{spawn}(s) = P$

\State $\mathcal{P}_{\mathrm{new}} =\emptyset$
\For{each species $s \in \mathcal{S}$}
  \State Sort members of $s$ by raw fitness (descending)
  \State Copy top $E$ elites of $s$ to $\mathcal{P}_{\mathrm{new}}$
  \State $\mathrm{spawn}(s) =\mathrm{spawn}(s) - E$
  \If{$\mathrm{spawn}(s) > 0$}
    \State $K =\lceil \rho \cdot |s| \rceil$
    \State $U =$ top-$K$ members of $s$ \Comment{parent pool}
    \While{$\mathrm{spawn}(s) > 0$}
      \State $\mathrm{spawn}(s) =\mathrm{spawn}(s) - 1$
      \State Sample $p_1 \sim U$
      \State With probability $p_c$, sample $p_2 \sim U$; otherwise set $p_2 =p_1$
      \State $\text{offspring} =\textsc{CrossoverAligned}(p_1, p_2)$ \Comment{align by innovation numbers}
      \If{$\mathrm{rand} < m_\ell$}
        \State \textsc{AddConnection}$(\text{offspring} )$
      \EndIf
      \If{$\mathrm{rand} < m_n$}
        \State \textsc{AddNode}$(\text{offspring} )$
      \EndIf
      \If{\texttt{WM\_MODE} = off}
        \State \textsc{NoWeightMutation}$(\text{offspring} )$
      \ElsIf{\texttt{WM\_MODE} = config}
        \State \textsc{MutateWeightsByConfig}$(\text{offspring} , \Theta_0)$
      \Else
        \State \textsc{MutateWeightsWithProb}$(\text{offspring} , p)$
      \EndIf
      \State Append $\text{offspring}$ to $\mathcal{P}_{\mathrm{new}}$
    \EndWhile
  \EndIf
\EndFor
\State \Return $\mathcal{P}_{\mathrm{new}}$
\end{algorithmic}
\end{algorithm}

Algorithm~\ref{alg:NEOL_main_loop} implements a two-timescale procedure. 
A population of size $P$ is initialised from $\Theta_0$. 
In each generation, every genome $g_i$ is evaluated by Algorithm~\ref{alg:NEOL_online}; its mean episodic return becomes its fitness. After evaluation, NEAT reproduction performs selection and topological variation to produce the next population. 
Across generations, the best genome among the population is tracked; after $G$ generations, the algorithm returns $g^*$.

Algorithm~\ref{alg:NEOL_online} evaluates one genome with online weight adaptation. 
For each episode, a network phenotype is created from $g$, the environment is reset, and the agent interacts for at most $T_{\max}$ steps. 
At step $t$, the network proposes $\hat{a}$, which is clipped to $a$ and applied to obtain $(s', r, \texttt{done})$. 
A reward-scaled signal $r_{\text{scaled}}$ drives plastic updates according to $\mathcal{L}$ at rate $\eta$ using only local pre and post activities together with $r_{\text{scaled}}$. 
The fitness is the mean return over $N$ episodes.

Algorithm~\ref{alg:speciate_neat} performs speciation and fitness adjustment for NEAT. Genomes are assigned to species by compatibility distance with coefficients $c_1,c_2,c_3$ and threshold $\delta_t$. Using excess and disjoint gene counts $E$ and $D$, the number of matched genes $N$, and the mean absolute weight difference, the distance to a species representative $r_s$ is
\[
\delta(g,r_s)=c_1\frac{E+D}{N}+c_3\,\overline{\lvert w_g-w_{r_s}\rvert}.
\]
Within each species, members are ranked by raw fitness, champions are tracked, and species that stagnate beyond $G_{\text{stag}}$ may be removed except for a possible global best safeguard. Explicit fitness sharing is applied:
\[
f^{\mathrm{adj}}(g)=\frac{f(g)}{|s|},\qquad
F^{\mathrm{adj}}(s)=\frac{1}{|s|}\sum_{g\in s} f^{\mathrm{adj}}(g).
\]

Algorithm~\ref{alg:reproduce_neat} generates the next population under speciated reproduction. 
Let $A=\sum_{s} F^{\mathrm{adj}}(s)$ be the sum of adjusted fitness across species. 
Each species receives an offspring budget
\[
\mathrm{spawn}(s)
  = \max\!\left(
      \texttt{min\_species},\,
      \operatorname{round}\!\left(
        P\,\frac{F^{\mathrm{adj}}(s)}{A}
      \right)
    \right),
\]
renormalised so that the counts sum to $P$. 
Elites $E$ are copied unchanged. 
The remaining offspring are bred from the top $\rho$ fraction within each species. 
Parents are selected, crossover is applied with probability $p_c$ using historical innovation numbers for alignment, and structural mutations are applied with probabilities $m_\ell$ (add connection) and $m_n$ (add node). 
Weight mutation is controlled by \texttt{WM\_MODE}, which can disable weight mutation, use configuration defaults, or apply a specified probability $p$. 
The resulting offspring across all species form the next population consumed by Algorithm~\ref{alg:NEOL_main_loop}.


\clearpage

\subsection{Four control tasks from OpenAI gym benchmark suite}

We evaluate on four standard environments from the OpenAI gym benchmark suite~\cite{brockman2016openai}, spanning diverse reward structures and action spaces.

\begin{enumerate}
   \item[(1)] \texttt{CartPole-v1} within Gymnasium~\cite{brockman2016openai}. 
Here, the state $s$ is a 4-dimensional vector
\[
s = [x, \dot{x}, \theta, \dot{\theta}],
\]
where $x$ and $\dot{x}$ are respectively the horizontal position and velocity of the cart, and $\theta$ and $\dot{\theta}$ are respectively the pole angle and angular velocity. 
The action space is discrete, $\mathcal{A}=\{0,1\}$, corresponding to pushing the cart left or right. The agent receives a dense reward of $+1$ per timestep until termination.

\item[(2)] \texttt{LunarLanderContinuous-v3} within Gymnasium. 
Here, the state $s$ is an 8-dimensional vector
\[
s = [x, y, v_x, v_y, \phi, v_{\phi}, c_L, c_R],
\]
where $x,y$ are the lander’s horizontal and vertical position, $v_x,v_y$ its linear velocities, $\phi$ its orientation angle, $v_{\phi}$ its angular velocity, and $c_L,c_R\in\{0,1\}$ indicate whether the left/right leg is in contact with the ground. 
The action space is bounded and 2D, $\mathcal{A}=[-1,1]^2$, where the two components control the throttle of the main engine and the lateral boosters, respectively. Rewards are densely shaped with additional terminal bonuses/penalties for landing safely versus crashing.

\item[(3)] \texttt{BipedalWalker-v3} within Gymnasium. 
Here, the state $s\in\mathbb{R}^{24}$ is a 24-dimensional vector summarising the walker’s torso (hull) orientation and velocities, joint angles and joint angular velocities, binary ground-contact indicators for the legs, and 10 lidar rangefinder readings (no absolute coordinates are included). 
The action space is continuous and four-dimensional,
\[
\mathcal{A}=[-1,1]^4,
\]
where each component specifies a motor command for one of the four actuated joints (hips and knees). The reward is dense, encouraging forward progress and penalising falls and excessive torque.

\item[(4)] \texttt{Hopper-v4} within Gymnasium. 
Here, the state $s\in\mathbb{R}^{11}$ is an 11-dimensional observation consisting of joint configuration variables (excluding the absolute $x$ position by default) together with their corresponding velocities (including forward velocity). Concretely, it includes the torso height and angle, the three joint angles, and the associated translational/rotational velocities. 
The action space is continuous and three-dimensional,
\[
\mathcal{A}=[-1,1]^3,
\]
where each component is a torque applied at one of the three hinge joints (thigh, leg, foot). The reward is dense and combines forward progress, a healthy-alive term, and control cost.

\end{enumerate}
Environment implementations follow the Gymnasium reference; physics backends are Box2D for \texttt{LunarLanderContinuous} and \texttt{BipedalWalker}, and MuJoCo for \texttt{Hopper}.

\subsection{Configuration for experiments}\label{app:configs}


\begin{table*}[!ht]
\centering
\caption{Best configurations used for the fitness results in Figure~\ref{fig:NEOL_comparison_combined}, Table~\ref{tab:fitness_table} and Table~\ref{tab:wilcoxon_vs_neat}. For plastic methods (BCM, Hebb, Oja), we report population size (pop) and learning rate (lr); NEAT has no learning rate.}
\label{tab:configs_best}
\small
\setlength{\tabcolsep}{6pt}
\renewcommand{\arraystretch}{1.15}
\begin{tabular}{lcccc}
\toprule
\textbf{Task} & \textbf{BCM} & \textbf{Hebb} & \textbf{Oja} & \textbf{NEAT} \\
\midrule
\emph{Cartpole}        & pop=50, lr=0.00025  & pop=50, lr=0.00025  & pop=50, lr=0.00025  & pop=50 \\
\emph{Lunar Lander}    & pop=300, lr=0.25    & pop=300, lr=0.00025 & pop=300, lr=0.0025  & pop=300 \\
\emph{Hopper}          & pop=300, lr=0.025   & pop=300, lr=0.00025 & pop=300, lr=0.0025  & pop=300 \\
\emph{Bipedal Walker}  & pop=300, lr=0.00025 & pop=200, lr=0.00025 & pop=300, lr=0.00025 & pop=100 \\
\bottomrule
\end{tabular}
\end{table*}

\begin{table*}[!ht]
\centering
\caption{Best configurations used for the fitness results in Figure~\ref{fig:NEOL_comparison_combined_RL}. 
For plastic methods (BCM, Hebb, Oja), we report population size (pop), generation (G) and learning rate (lr); NEAT has no learning rate.}
\label{tab:configs_best_general}
\small
\setlength{\tabcolsep}{6pt}
\renewcommand{\arraystretch}{1.15}
\begin{tabular}{lcccc}
\toprule
\textbf{Task} & \textbf{BCM} & \textbf{Hebb} & \textbf{Oja} & \textbf{NEAT} \\
\midrule
\emph{Cartpole}        & pop=50, G=200, lr=0.00025  & pop=50, G=200, lr=0.00025  & pop=50, G=200, lr=0.00025  & pop=50, G=200 \\
\emph{Lunar Lander}    &  pop=100,G=100, lr=0.00025    &  pop=100,G=100, lr=0.00025 &  pop=100,G=100, lr=0.00025  & pop=100,G=100 \\
\emph{Hopper}          & pop=100,G=100, lr=0.00025   & pop=100,G=100, lr=0.00025 & pop=100,G=100, lr=0.00025  & pop=100,G=100 \\
\emph{Bipedal Walker}  & pop=50, G=200, lr=0.00025 & pop=100,G=100, lr=0.00025 & pop=50, G=200, lr=0.00025 & pop=50, G=200 \\
\bottomrule
\end{tabular}
\end{table*}

\begin{table*}[!ht]
\centering
\caption{Configurations used for the ablation fitness results in Table~\ref{tab:fitness_ablation}. 
NEAT (w/o) disables weight modulation by setting the learning rate to zero.}
\label{tab:configs_ablation}
\small
\setlength{\tabcolsep}{6pt}
\renewcommand{\arraystretch}{1.15}
\begin{tabular}{lcccc}
\toprule
\textbf{Task} & \textbf{BCM} & \textbf{Hebb} & \textbf{Oja} & \textbf{NEAT (w/o)} \\
\midrule
\emph{Cartpole}        & pop=50, lr=0.00025  & pop=50, lr=0.00025  & pop=50, lr=0.00025  & pop=50 \\
\emph{Lunar Lander}    & pop=300, lr=0.25    & pop=300, lr=0.00025 & pop=300, lr=0.0025  & pop=300 \\
\emph{Hopper}          & pop=300, lr=0.025   & pop=300, lr=0.00025 & pop=300, lr=0.0025  & pop=300 \\
\emph{Bipedal Walker}  & pop=300, lr=0.00025 & pop=300, lr=0.00025 & pop=300, lr=0.00025 & pop=300 \\
\bottomrule
\end{tabular}
\end{table*}

\clearpage
\subsection{Statistical Tests for Comparison on Best Fitness}

\begin{table*}[!ht]
\centering
\caption{Best final-generation fitness (mean $\pm$ standard deviation (SD)),
\textbf{Bold} marks the highest fitness per task; \underline{underline} marks the runner-up.}
\label{tab:fitness_table}
\small
\setlength{\tabcolsep}{6pt}
\renewcommand{\arraystretch}{1.15}
\begin{tabular}{lcccc}
\toprule
\text{Task} & \text{BCM} & \text{Hebb} & \text{Oja} & \text{NEAT} \\
\midrule
\emph{CartPole}        & \best{500.00 \pm 0.00} & \best{500.00 \pm 0.00} & \best{500.00 \pm 0.00} & \best{500.00 \pm 0.00} \\
\emph{Lunar Lander}    & \best{324.34 \pm 3.86} & $323.35 \pm 3.67$ & \runner{323.53 \pm 4.19} & $311.77 \pm 8.18$ \\
\emph{Hopper}          & \best{2983.82 \pm 479.79} & $2819.89 \pm 577.33$ & \runner{2900.94 \pm 562.16} & $2680.22 \pm 603.53$ \\
\emph{Bipedal Walker}  & $217.39 \pm 82.82$ & \best{233.30 \pm 62.42} & \runner{227.10 \pm 81.96} & $149.59 \pm 79.36$ \\
\bottomrule
\end{tabular}
\end{table*}

\begin{table*}[!ht]
\centering
\caption{Wilcoxon rank-sum test: $p$-values for Bipedal Walker, Hopper, and Lunar Lander, comparing each OL method (BCM, Hebb, Oja) against NEAT. 
For each task and method, we use the best configuration selected by means of final-generation best across $30$ seeds, and test the final best values. 
The null hypothesis is that the two configurations yield samples from the same distribution; the alternative hypothesis is that the OL method tends to achieve the larger best-fitness than NEAT. 
\textbf{Bold} entries indicate $p<0.05$ (rejecting the null at the 5\% level). 
Cartpole is omitted because all methods achieve $500.0$ exactly in the end.
Values shown with three significant figures ($a\,\mathrm{e}b \equiv a\times10^{b}$).
}
\label{tab:wilcoxon_vs_neat}
\small
\setlength{\tabcolsep}{6pt}
\renewcommand{\arraystretch}{1.15}
\begin{tabular}{lccc}
\toprule
\text{Task} & \text{BCM vs NEAT} & \text{Hebb vs NEAT} & \text{Oja vs NEAT} \\
\midrule
\emph{Lunar Lander} & \best{5.76\,e{-9}} & \best{1.61\,e{-8}} & \best{2.25\,e{-8}} \\
\emph{Hopper}       & \best{8.19\,e{-4}} & \best{2.82\,e{-2}} & \best{7.36\,e{-3}} \\
\emph{Bipedal Walker} & \best{1.38\,e{-4}} & \best{9.04\,e{-6}} & \best{3.49\,e{-5}} \\
\bottomrule
\end{tabular}
\end{table*}


\subsection{Additional Tables for Ablation Studies}
In this ablation study, we keep NEOL unchanged and ablate only the pure NEAT baseline by disabling genetic weight mutation. 
The ablated counterpart is denoted NEAT (w/o). 
This setting is equivalent to setting the online rate of online learning to zero (fixing the weight mutation of NEAT means there is only evolutionary topology search remaining), and isolates whether online plasticity inside NEOL can substitute for, or complement, evolutionary weight mutation.

From Table~\ref{tab:fitness_ablation}, we compare standard NEOL (BCM, Hebb, Oja) with the ablated pure NEAT baseline (NEAT (w/o)). 
On Cartpole, all methods reach the optimum and are indistinguishable. 
On Lunar Lander, every NEOL rule yields a higher mean than NEAT (w/o), with BCM at {324.34}, Oja at 323.53, and Hebb at 323.35, versus 312.50 for NEAT (w/o); the standard deviation for NEOL is small, indicating reliable convergence. 
On Hopper, the gap is significant: BCM, Oja, and Hebb achieve 2983.82, 2900.94, and 2819.89, respectively, compared with 2296.57 for NEAT (w/o), with difference shown and well-separated means. 
On Bipedal Walker, NEOL still leads in the mean (Hebb 233.30, Oja 227.10, BCM 217.39) over NEAT (w/o) (194.91), but standard deviations are wide for all configurations, suggesting that although Bipedal Walker is a harder task for all the algorithms, every NEOL is still robust.
With weight mutation disabled in the counterpart, reward-modulated online plasticity is sufficient to recover and surpass final optimisation on Lunar Lander and Hopper tasks.


\begin{table*}[!ht]
\centering
\caption{Ablation study comparing NEOL with NEAT (w/o) on best final-generation fitness (mean $\pm$ SD) for each task and method using each method’s best hyperparameters. 
NEAT (w/o) corresponds to disabling the weight-modulation mechanism by setting the learning rate $\text{lr}=0$.
\textbf{Bold} marks the highest fitness per task; \underline{underline} marks the runner-up.}
\label{tab:fitness_ablation}
\small
\setlength{\tabcolsep}{6pt}
\renewcommand{\arraystretch}{1.15}
\begin{tabular}{lcccc}
\toprule
\text{Task} & \text{BCM} & \text{Hebb} & \text{Oja} & \text{NEAT (w/o)} \\
\midrule
\emph{Cartpole}        & \best{500.00 \pm 0.00} & \best{500.00 \pm 0.00} & \best{500.00 \pm 0.00} & \best{500.00 \pm 0.00} \\
\emph{Lunar Lander}    & \best{324.34 \pm 3.86} & $323.35 \pm 3.67$ & \runner{323.53 \pm 4.19} & $312.50 \pm 8.97$ \\
\emph{Hopper}          & \best{2983.82 \pm 479.79} & $2819.89 \pm 577.33$ & \runner{2900.94 \pm 562.16} & $2296.57 \pm 740.56$ \\
\emph{Bipedal Walker}  & $217.39 \pm 82.82$ & \best{233.30 \pm 62.42} & \runner{227.10 \pm 81.96} & $194.91 \pm 92.83$ \\
\bottomrule
\end{tabular}
\end{table*}

\clearpage
\subsection{More Experiment Results on the Final Best Fitness}

\subsection{Summary of More Experiment Results}
\paragraph{Parameter Setting.}
The heatmaps in \autoref{fig:NEOL_comparison_combined_heatmap_cartpole},\autoref{fig:NEOL_comparison_combined_heatmap_Lunar},\autoref{fig:NEOL_comparison_combined_heatmap_Hopp},\autoref{fig:NEOL_comparison_combined_heatmap_BW}, 
report a parameter sweep over population size $\{50, 100, 200, 300\}$ and learning rates $\{2.5\times10^{-4},\,2.5\times10^{-3},\,2.5\times10^{-2},\,2.5\times10^{-1}\}$ for the three NEAT-NEOL variants (BCM, Hebb, Oja) alongside the standard NEAT baseline.
Each pixel shows the empirical mean of the final best fitness over 30 seeds; the vertical axis is the plasticity learning rate, and the horizontal axis is the population size. 
Because NEAT has no learning rate, it appears as a single row per environment.

\paragraph{Results}
On {CartPole-v1 (\autoref{fig:NEOL_comparison_combined_heatmap_cartpole}).} 
All methods reach the task ceiling and remain flat across most settings in best fitness. 
BCM and Hebb saturate at $500$ for every population size and plasticity rate on our heatmap, indicating that structural search alone or in combination with offline updates is sufficient on this easy, discrete–action benchmark. 
Oja exhibits a mild instability only at the largest plasticity rates and smallest populations (top row, leftmost columns), where the mean final best fitness drops below the ceiling, but recovers as the rate is reduced or the population increases. The NEAT baseline sits at the ceiling for all population sizes.

On {LunarLander-v3 (\autoref{fig:NEOL_comparison_combined_heatmap_Lunar}).} 
The three NEOL variants consistently dominate the NEAT row and exhibit a smooth improvement with population size. 
BCM is the most robust to the plasticity learning rate: means above 320 are obtained for populations 200–300 across a wide range of rates, and the best cell is attained at population 300. 
Hebb is more sensitive to the learning rate: very large learning rates combined with small populations depress performance, while learning rates in $[2.5\times 10^{-4},\,2.5\times 10^{-3}]$ recover and surpass 320 as the population grows. 
Oja shows a similar pattern, with its best region again in the bottom half of the grid and larger populations. 
The NEAT row improves slightly with population, but remains roughly 10–15 points below the strongest NEOL settings.

On {Hopper-4v (\autoref{fig:NEOL_comparison_combined_heatmap_Hopp}).}
For BCM and Oja the surface rises sharply with population and peaks at intermediate plasticity learning rates ($2.5\times 10^{-3}$ or $2.5\times 10^{-2}$), reaching mean final best fitness near or above 2.9k at population 300. 
Very small learning rates underfit and very large rates overfit or destabilise, producing a characteristic ridge across the middle rows. 
Hebb benefits from the same scaling trends but remains below BCM and Oja over most of the grid, particularly at small populations or extreme rates. 
Standard NEAT also scales with population but plateaus several hundred points below the best NEOL cells, indicating that online weight adaptation contributes materially beyond structural search in this domain. 
Continuous control displays a pronounced interaction between population size and plasticity rate. 

On {BipedalWalker-v3 (\autoref{fig:NEOL_comparison_combined_heatmap_BW}).}
NEOL improves upon NEAT across broad regions. 
Hebb attains the highest cell in the grid at population 200 with the smallest plasticity learning rate, and degrades rapidly as the learning rate increases, especially at small populations. 
Oja and BCM display more gradual trends: performance climbs with population size and is best in the lowest–rate row, with Oja’s peak at population 300 and BCM’s at population 300 as well. 
The NEAT row is comparatively flat and non–monotonic in population, with means concentrated around 130–180 and no configuration matching the top NEOL cells. 
These results suggest that modest, reward–gated plasticity combined with sufficient population–level exploration is beneficial, whereas aggressive learning rates are detrimental in this environment.


\subsection{Lamarckian vs. Darwinian in NEAT-NEOL}

In NEAT-NEOL, online synaptic plasticity modifies weights during each evaluation rollout. A natural design choice is whether these within-lifetime weight changes should be inherited by the next generation. We therefore consider two inheritance mechanisms.

In the Darwinian (Baldwinian) variant, plasticity updates affect the fitness obtained during a rollout but are not inherited: each evaluation starts from genotype-specified parameters and the adapted weights are discarded after evaluation. 
In the Lamarckian variant, the post-adaptation parameters are written back to the genotype, so acquired within-lifetime changes are inherited, and selection operates on the updated parameters.

Empirically, we find that the two variants yield highly similar learning curves and final performance across the benchmarks considered (differences are within the run-to-run variance and do not change the qualitative conclusions). This resemblance suggests that, under our training budget and reward-modulated plasticity rules, the primary benefit comes from within-episode adaptation improving immediate fitness, rather than from inter-generational accumulation of plasticity-induced weight changes. Given the comparable performance and for simplicity of implementation and reporting, we adopt the Lamarckian write-back strategy in the remaining experiments.

\begin{figure}[!ht]
    \centering
    \includegraphics[width=0.49\linewidth]{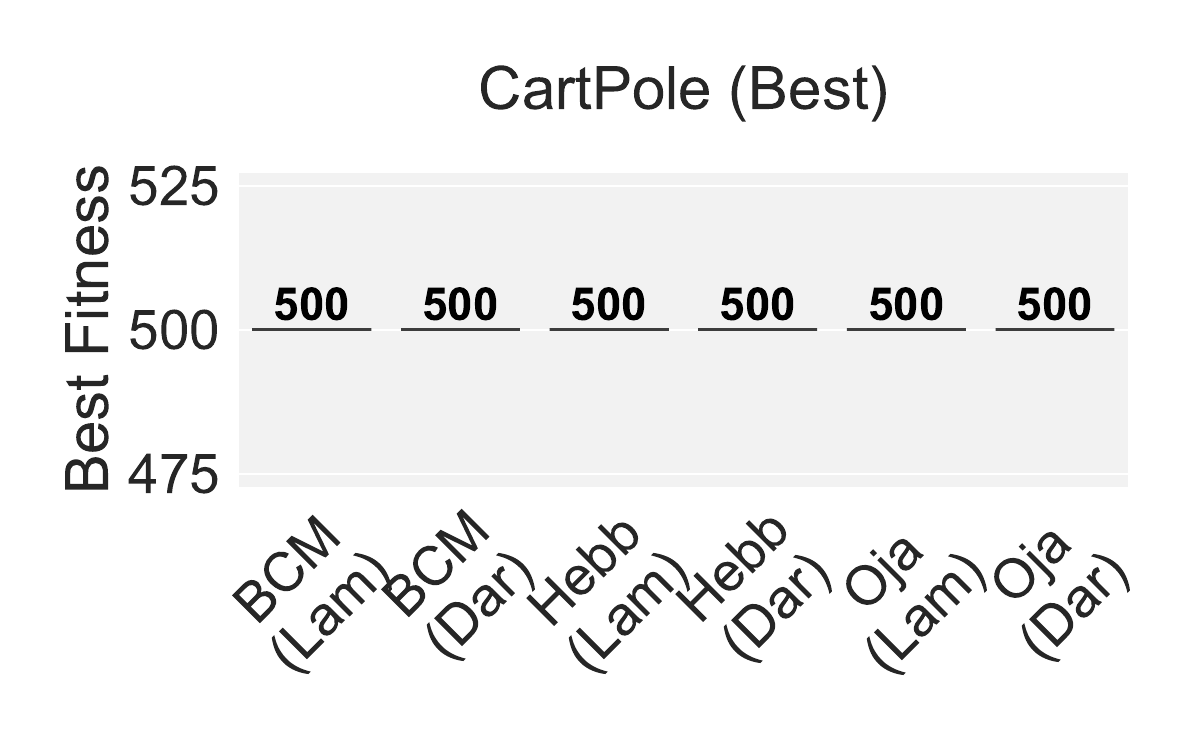}
    \includegraphics[width=0.49\linewidth]{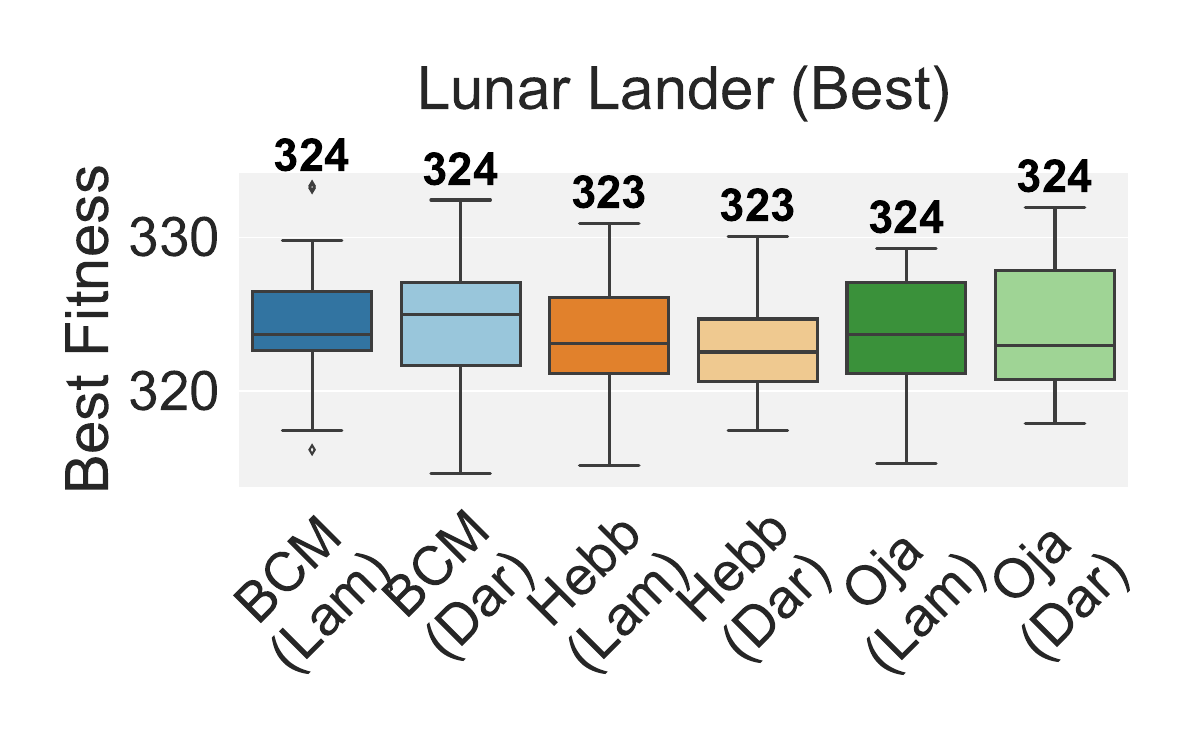}
    \includegraphics[width=0.49\linewidth]{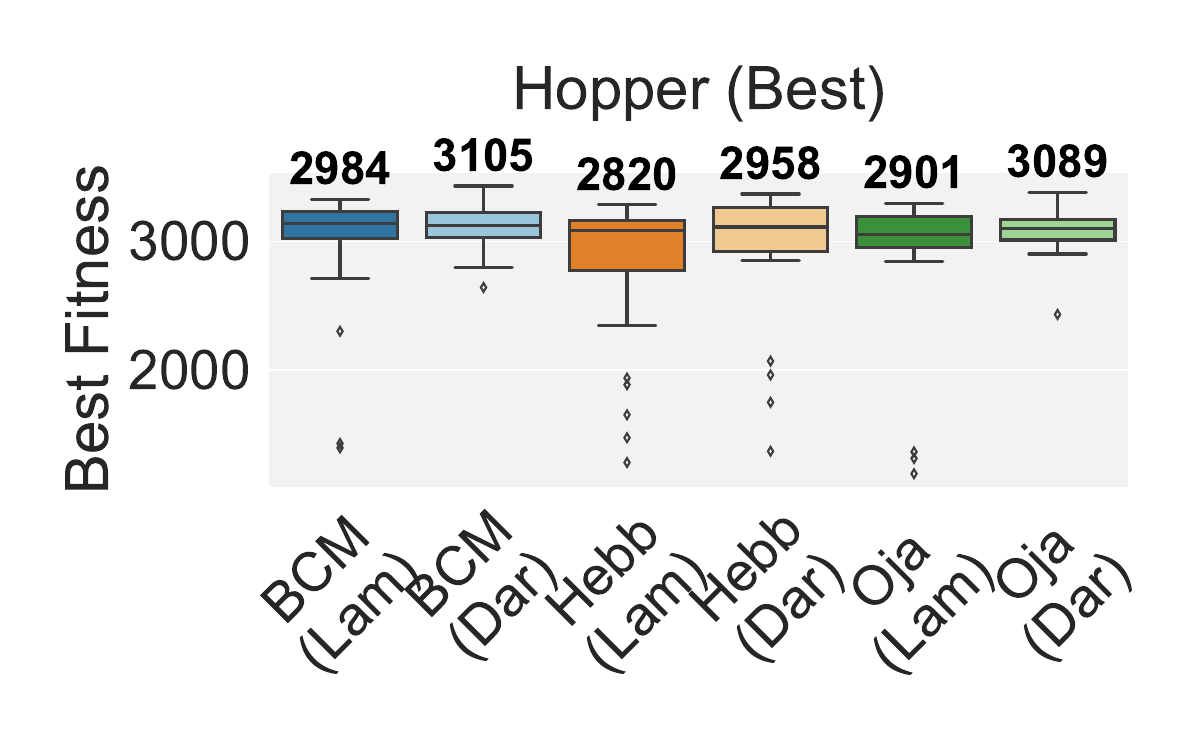}
    \includegraphics[width=0.49\linewidth]{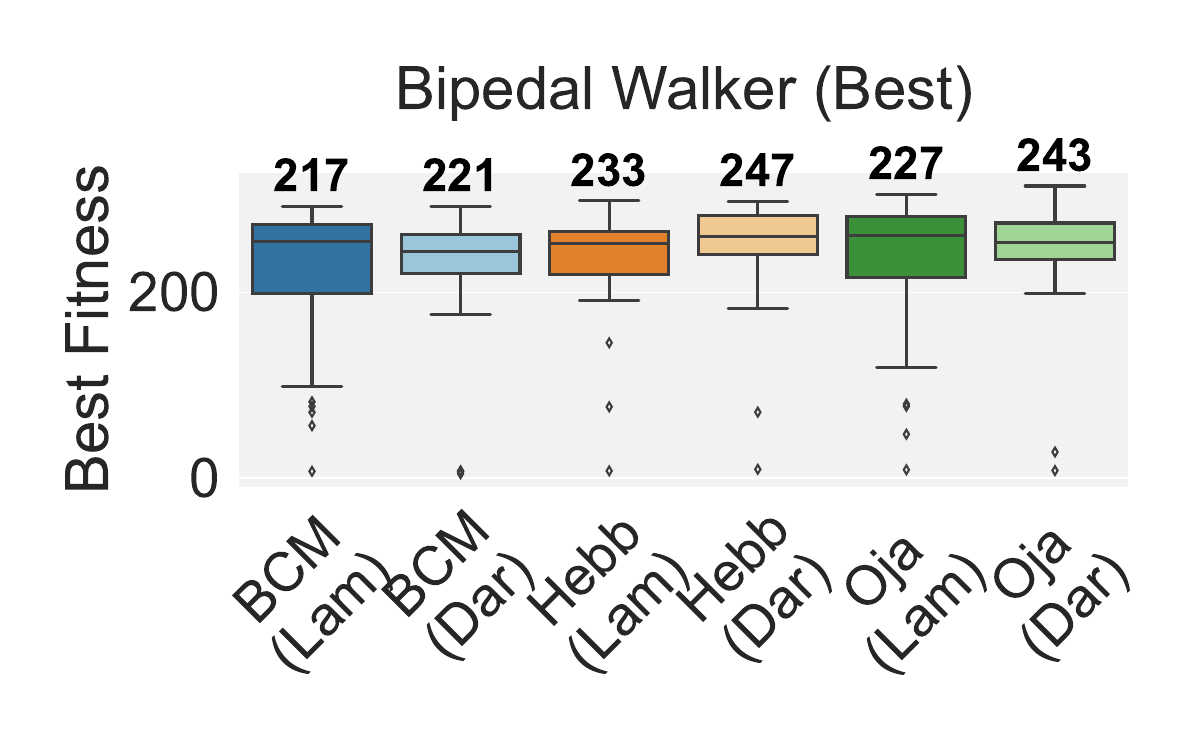}
    
    \caption{Performance comparison of NEAT-NEOL with the standard NEAT across four environments (\texttt{CartPole-v1}, \texttt{LunarLander-v3}, \texttt{Hopper-v4}, and \texttt{BipedalWalker-v3}). 
    {\bf Top row}: convergence plots showing fitness over generations. 
    {\bf Bottom row}: boxplots of final-generation fitness distributions. 
    BCM, Hebb, and Oja learning rules are shown in blue, orange, and green, respectively, while standard NEAT is shown in red.}
    \label{fig:NEOL_comparison_LamarvsDar}
\end{figure}

\begin{figure}[!ht]
    \centering
    \begin{subfigure}[t]{0.38\linewidth}
        \centering
        \subcaption{BCM}
        \includegraphics[width=\linewidth]{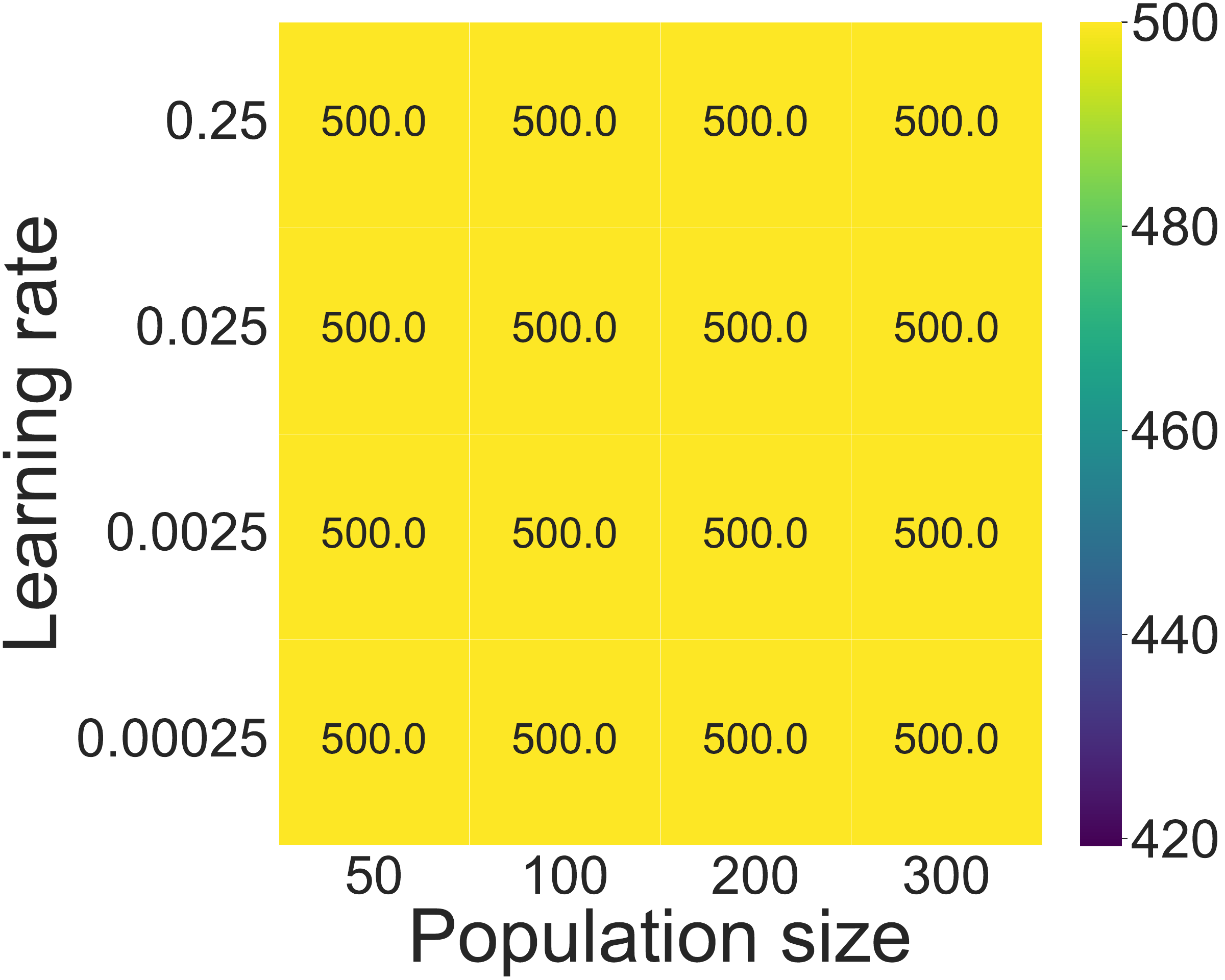}
    \end{subfigure}
    \hfill
    \begin{subfigure}[t]{0.38\linewidth}
        \centering
        \subcaption{Hebb}
        \includegraphics[width=\linewidth]{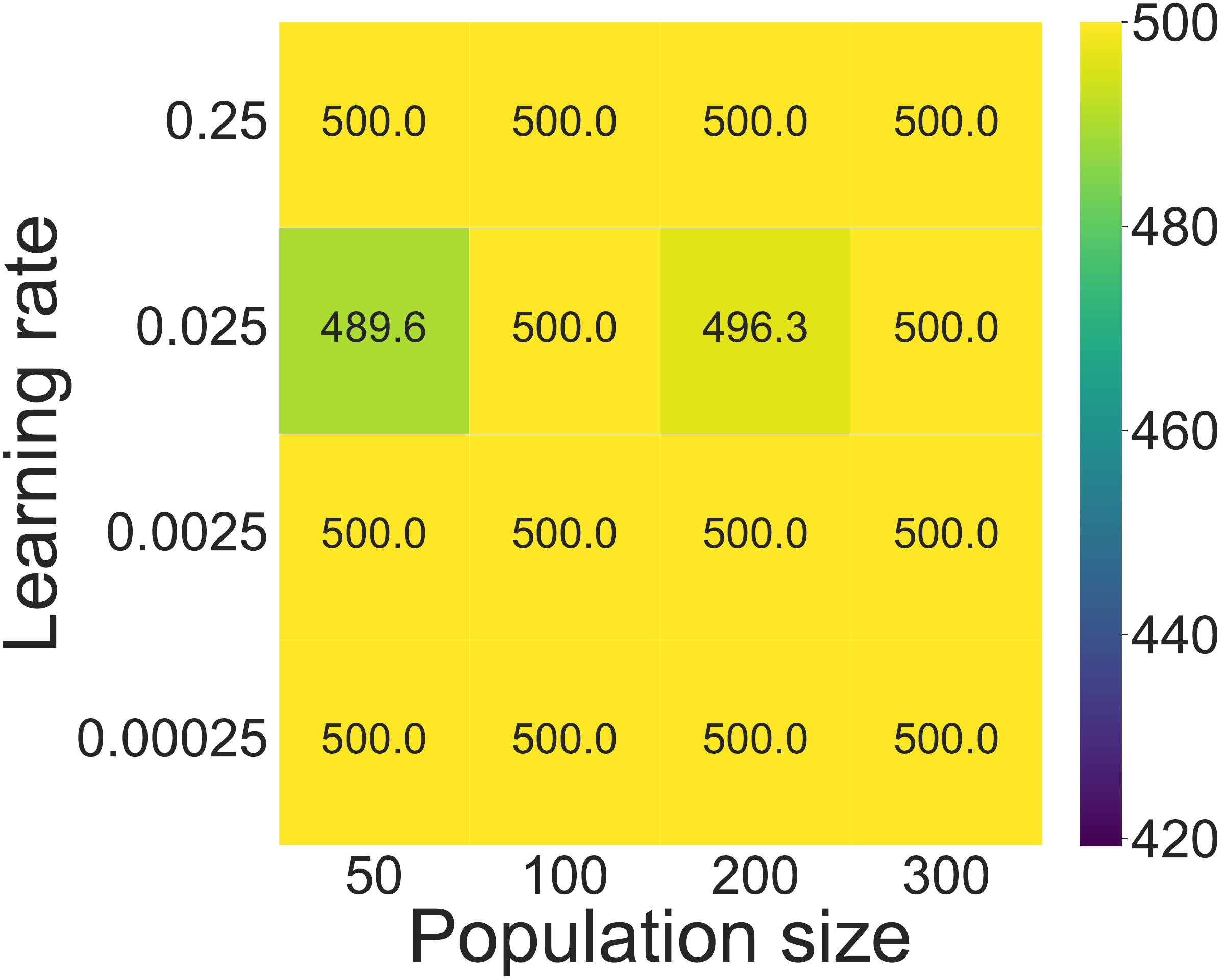}
    \end{subfigure}

    \vspace{0.5em}
    \begin{subfigure}[t]{0.38\linewidth}
        \centering
        \subcaption{Oja}
        \includegraphics[width=\linewidth]{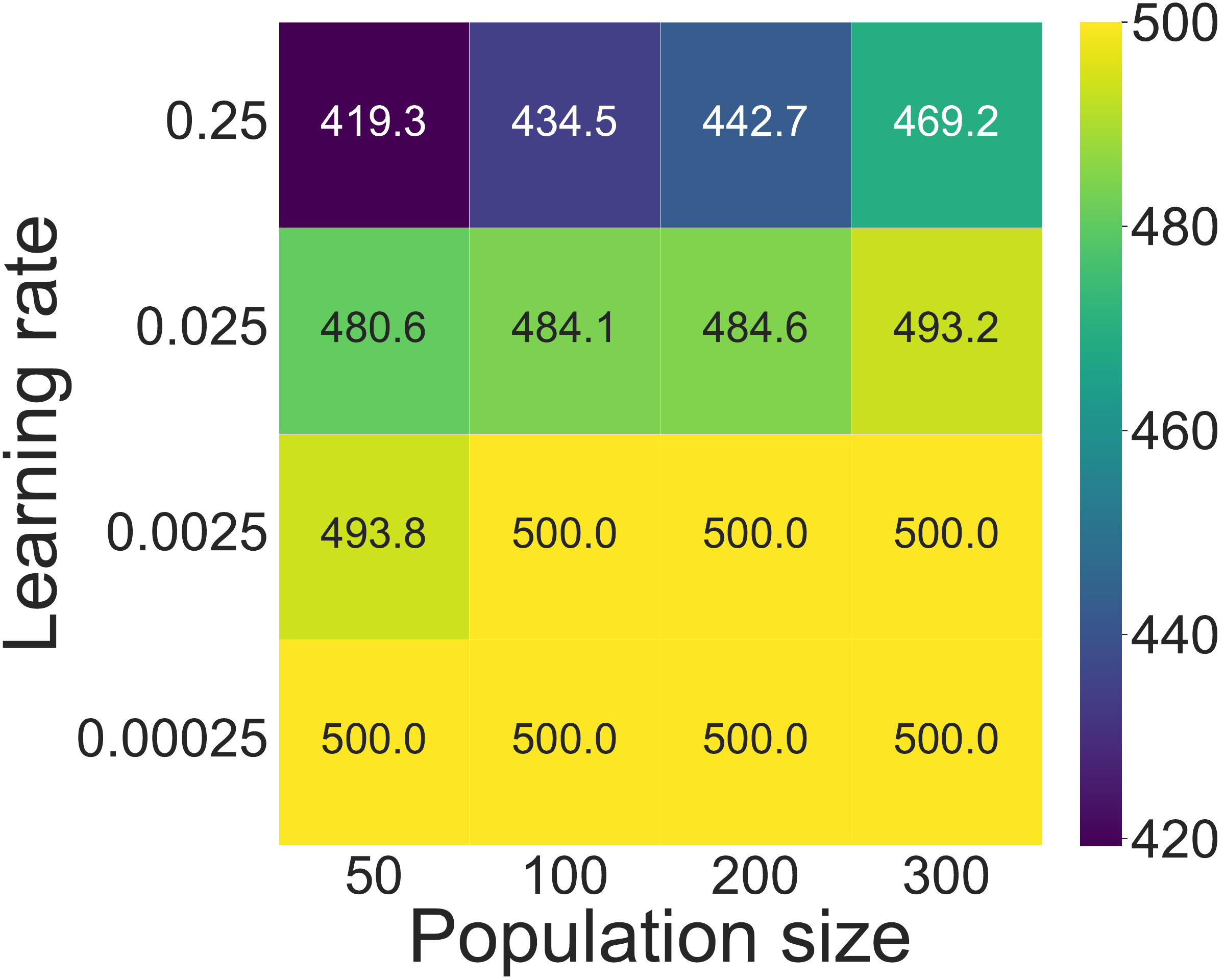}
    \end{subfigure}
        \hfill
    \begin{subfigure}[t]{0.38\linewidth}
        \centering
        \subcaption{Standard NEAT}
        \includegraphics[width=\linewidth]{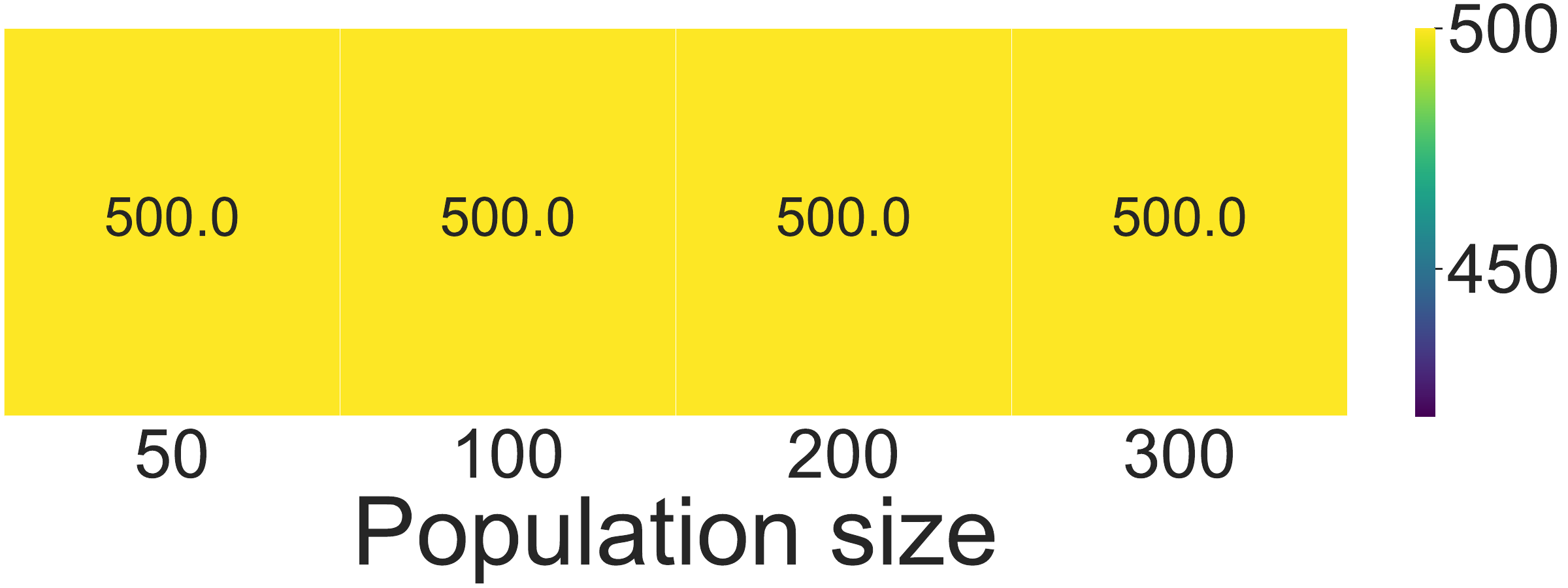}
    \end{subfigure}
    
    \caption{Heatmap comparison of NEOL algorithms (BCM, Hebb, Oja) against standard NEAT in \texttt{CartPole-v1}. 
    Values represent the empirical mean of the final best fitness over $30$ seeds.}
    \label{fig:NEOL_comparison_combined_heatmap_cartpole}
\end{figure}

\begin{figure}[!ht]
    \centering
    \begin{subfigure}[t]{0.38\linewidth}
        \centering
        \subcaption{BCM}
        \includegraphics[width=\linewidth]{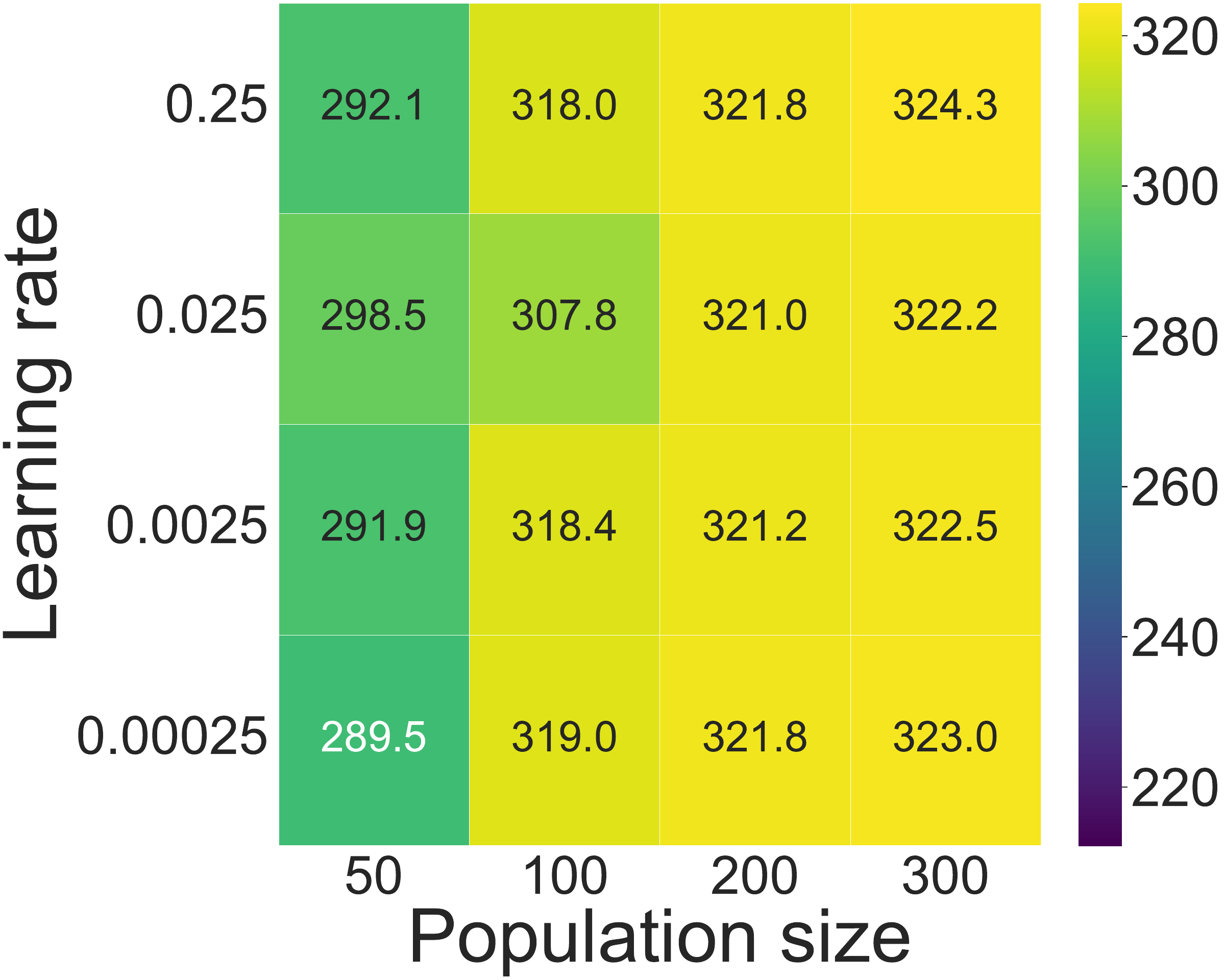}
    \end{subfigure}
    \hfill
    \begin{subfigure}[t]{0.38\linewidth}
        \centering
        \subcaption{Hebb}
        \includegraphics[width=\linewidth]{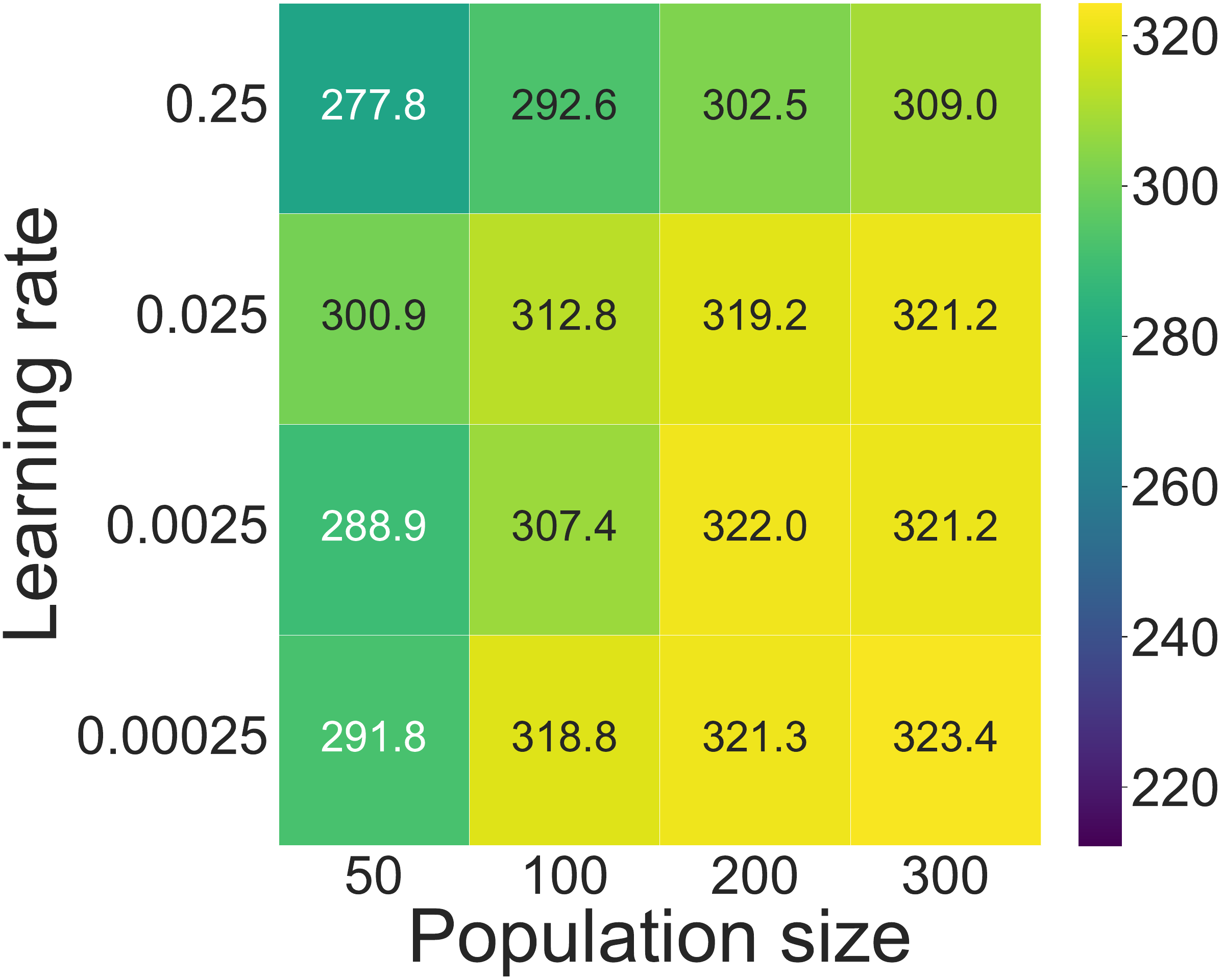}
    \end{subfigure}

    \vspace{0.5em}
    \begin{subfigure}[t]{0.38\linewidth}
        \centering
        \subcaption{Oja}
        \includegraphics[width=\linewidth]{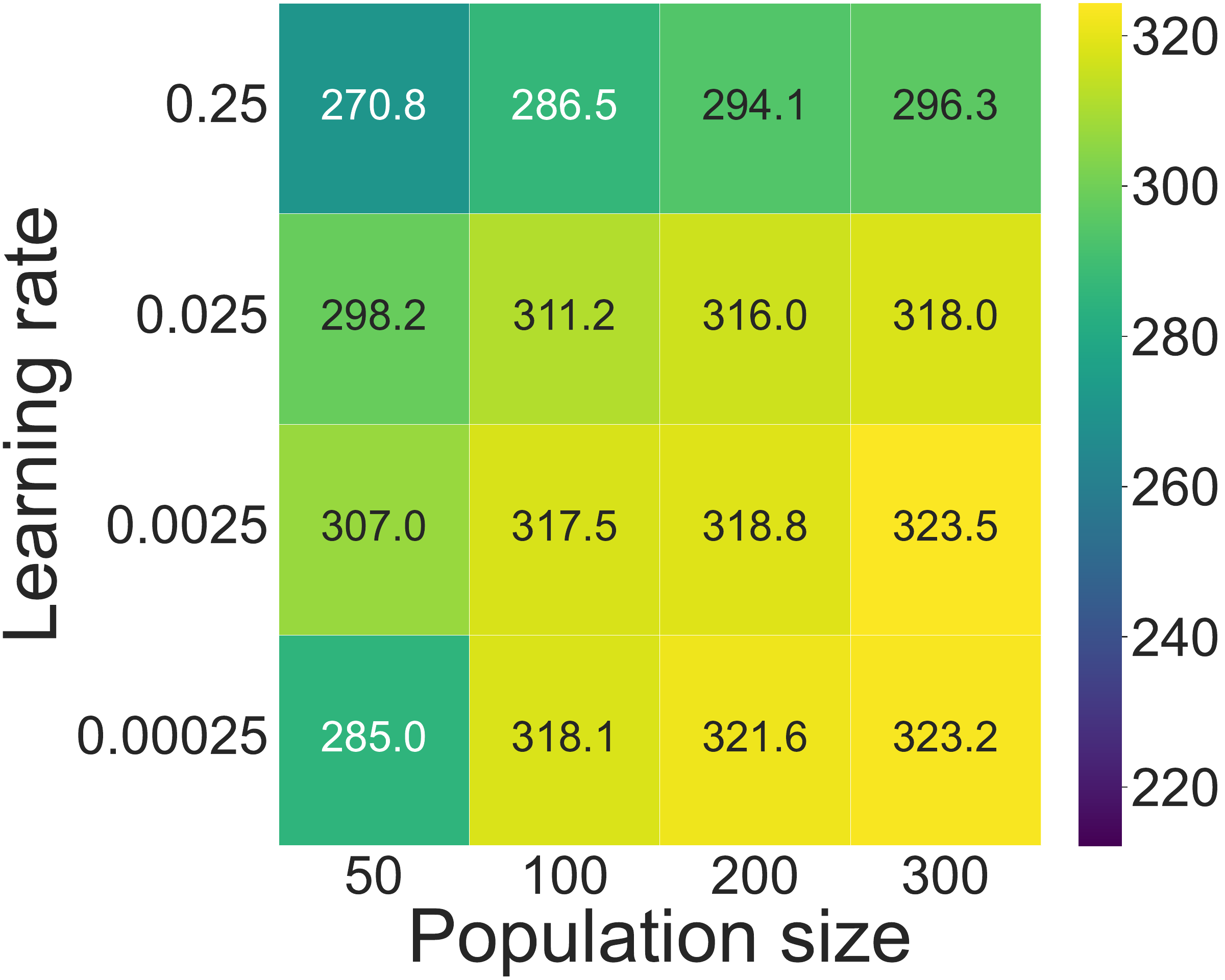}
    \end{subfigure}
    \hfill
    \begin{subfigure}[t]{0.38\linewidth}
        \centering
        \subcaption{Standard NEAT}
        \includegraphics[width=\linewidth]{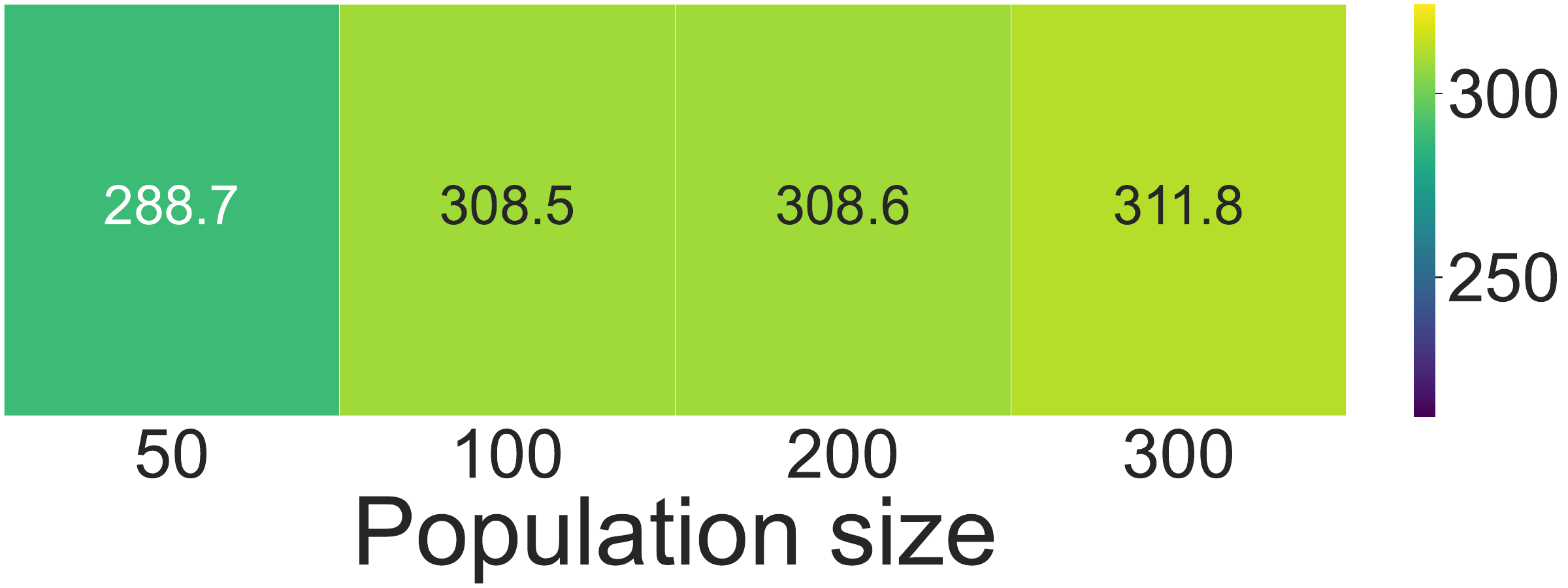}
    \end{subfigure}
    
    \caption{Heatmap comparison of NEOL algorithms (BCM, Hebb, Oja) against standard NEAT in \texttt{LunarLander-v3}. 
    Values represent the empirical mean of the final best fitness over $30$ seeds.}
    \label{fig:NEOL_comparison_combined_heatmap_Lunar}
\end{figure}

\begin{figure}[!ht]
    \centering
    \begin{subfigure}[t]{0.38\linewidth}
        \centering
        \subcaption{BCM}
        \includegraphics[width=\linewidth]{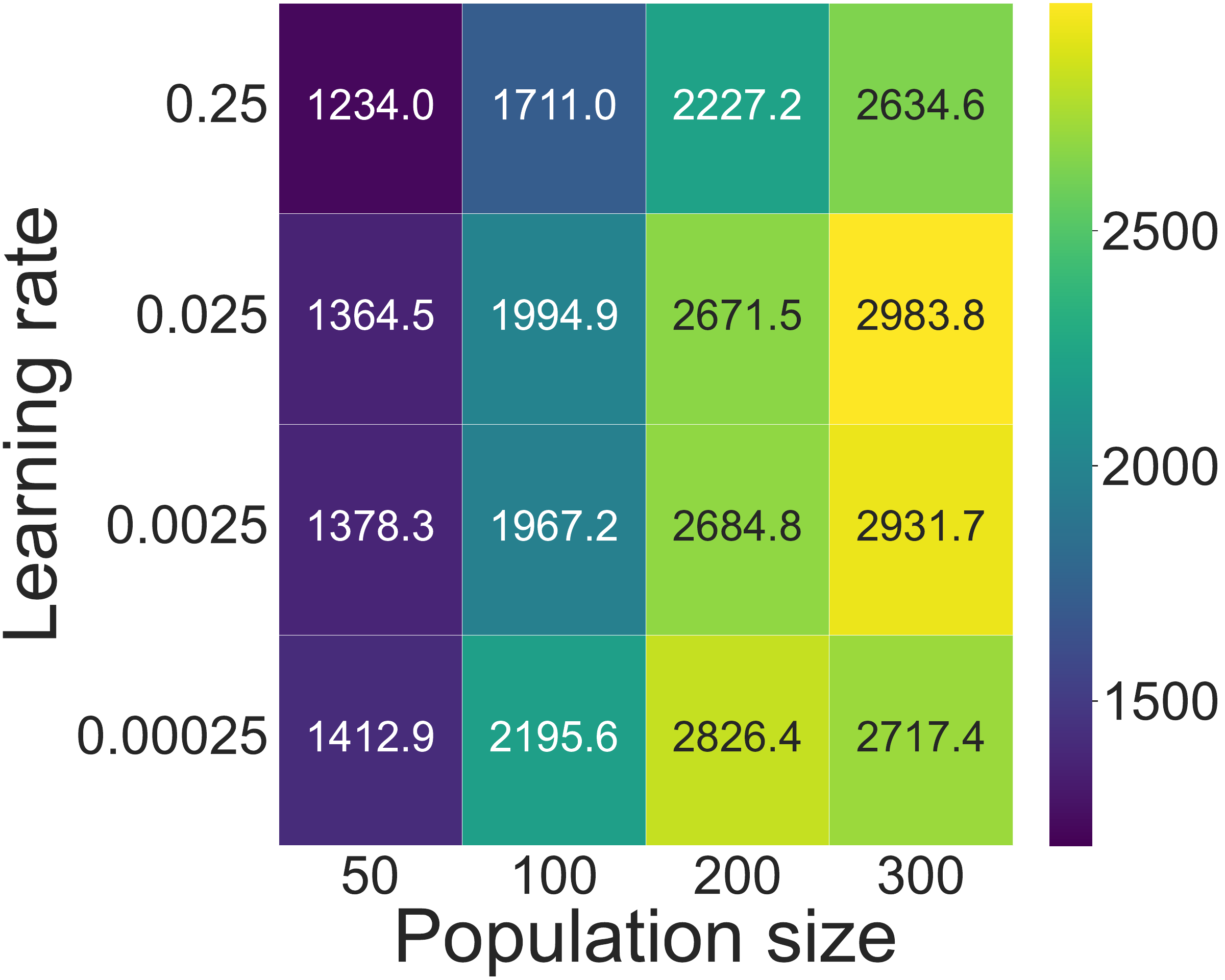}
    \end{subfigure}
    \hfill
    \begin{subfigure}[t]{0.38\linewidth}
        \centering
        \subcaption{Hebb}
        \includegraphics[width=\linewidth]{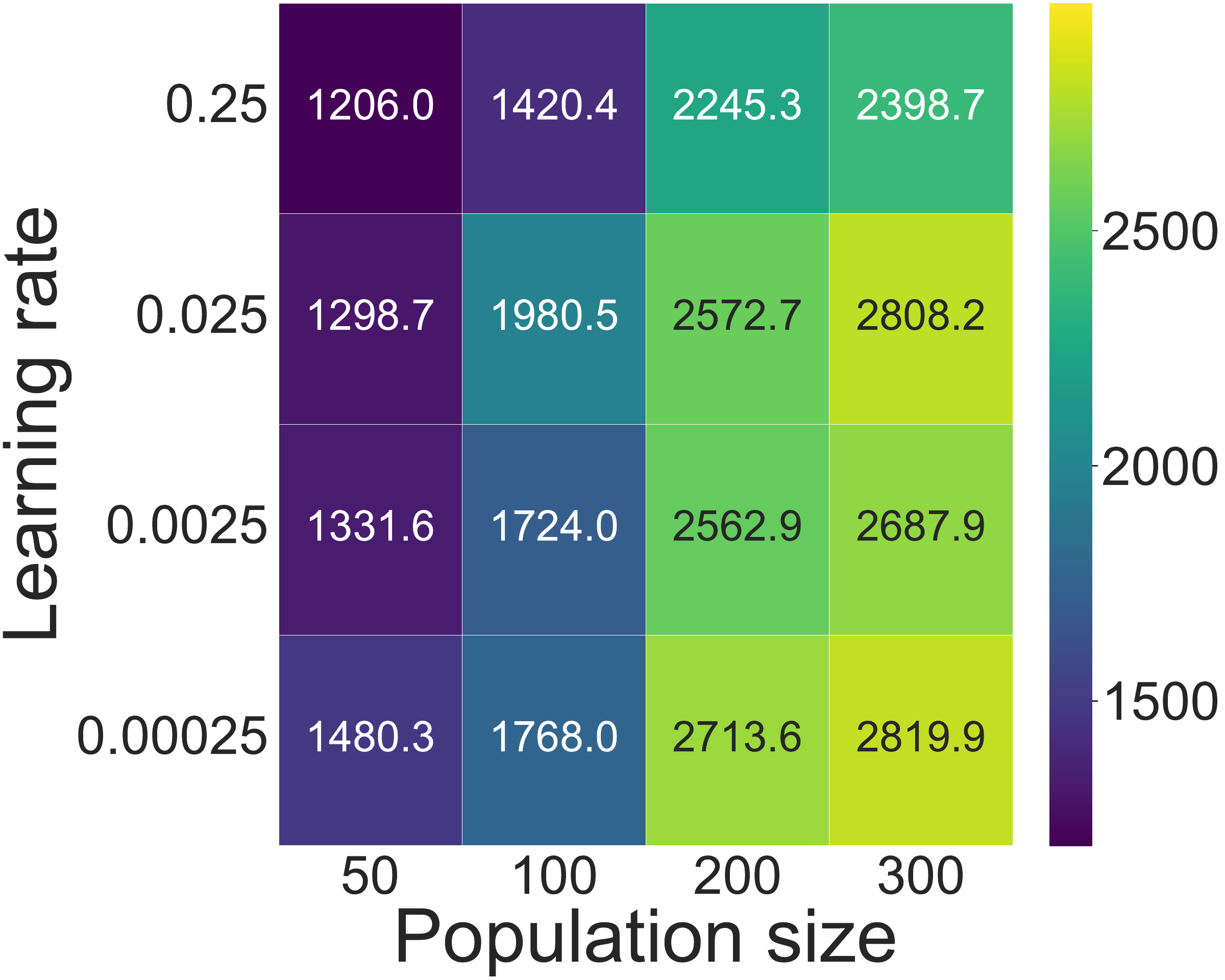}
    \end{subfigure}
    \vspace{0.5em}
    \begin{subfigure}[t]{0.38\linewidth}
        \centering
        \subcaption{Oja}
        \includegraphics[width=\linewidth]{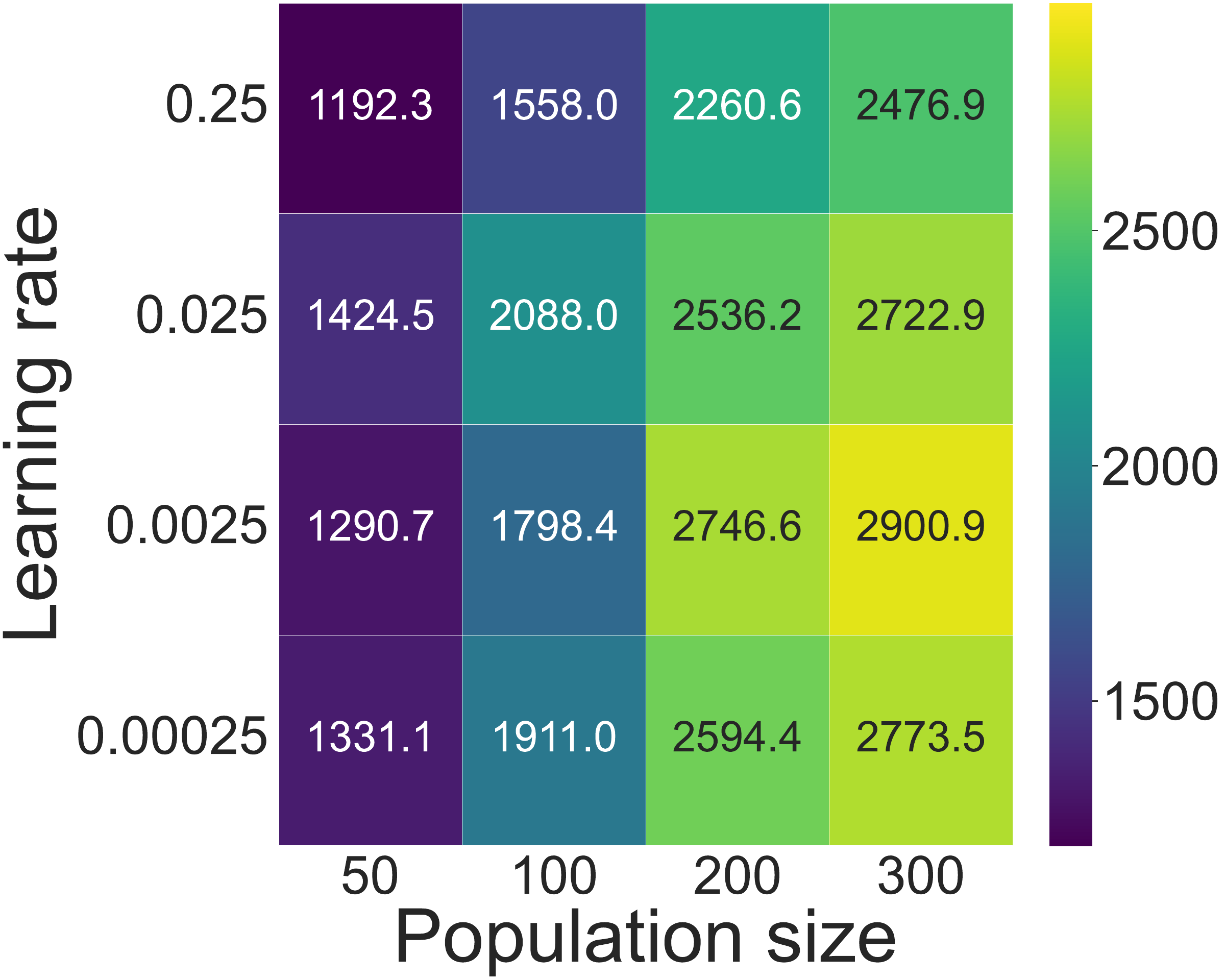}
    \end{subfigure}    
    \hfill
    \begin{subfigure}[t]{0.38\linewidth}
        \centering
        \subcaption{Standard NEAT}
        \includegraphics[width=\linewidth]{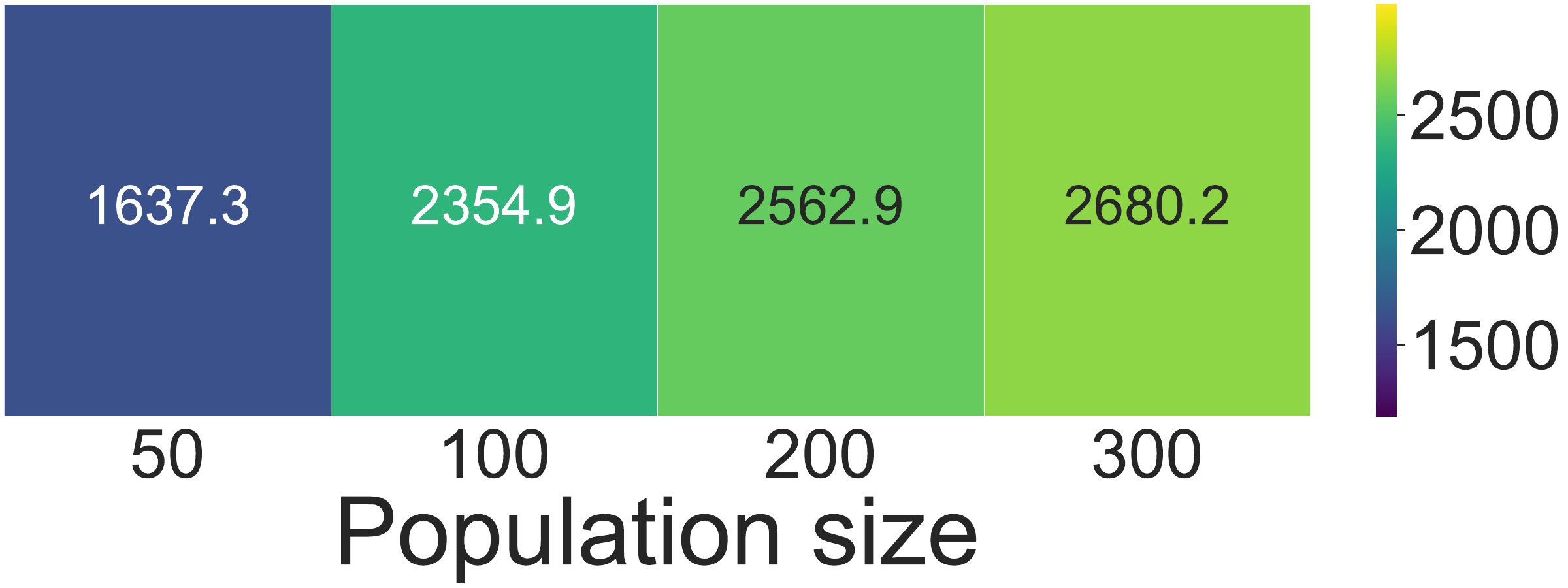}
    \end{subfigure}
    
    \caption{Heatmap comparison of NEOL algorithms (BCM, Hebb, Oja) against standard NEAT in \texttt{Hopper-v4}. 
    Values represent the empirical mean of the final best fitness over $30$ seeds.}
    \label{fig:NEOL_comparison_combined_heatmap_Hopp}
\end{figure}

\begin{figure}[!ht]
    \centering
    \begin{subfigure}[t]{0.38\linewidth}
        \centering
        \subcaption{BCM}
        \includegraphics[width=\linewidth]{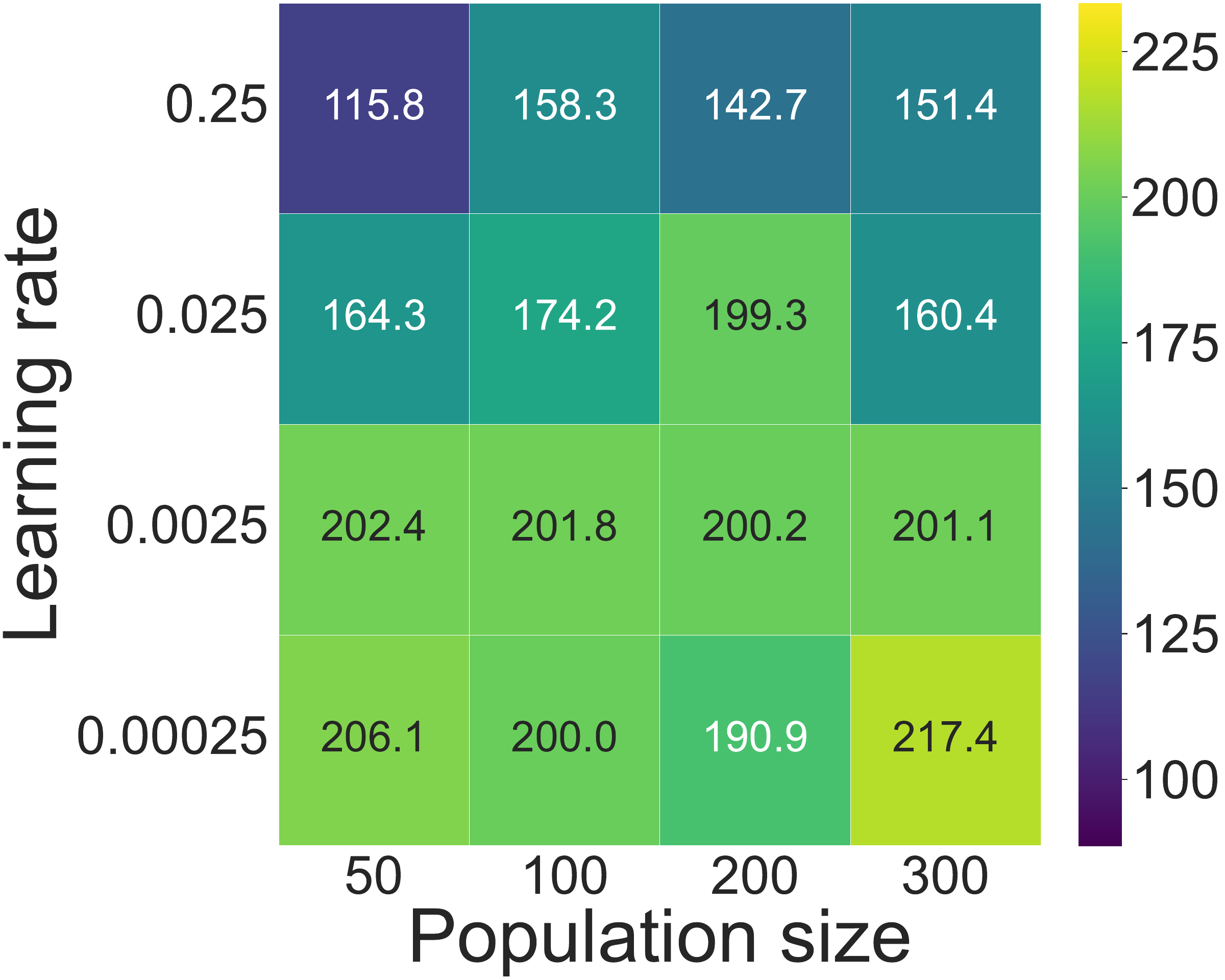}
    \end{subfigure}
    \hfill
    \begin{subfigure}[t]{0.38\linewidth}
        \centering
        \subcaption{Hebb}
        \includegraphics[width=\linewidth]{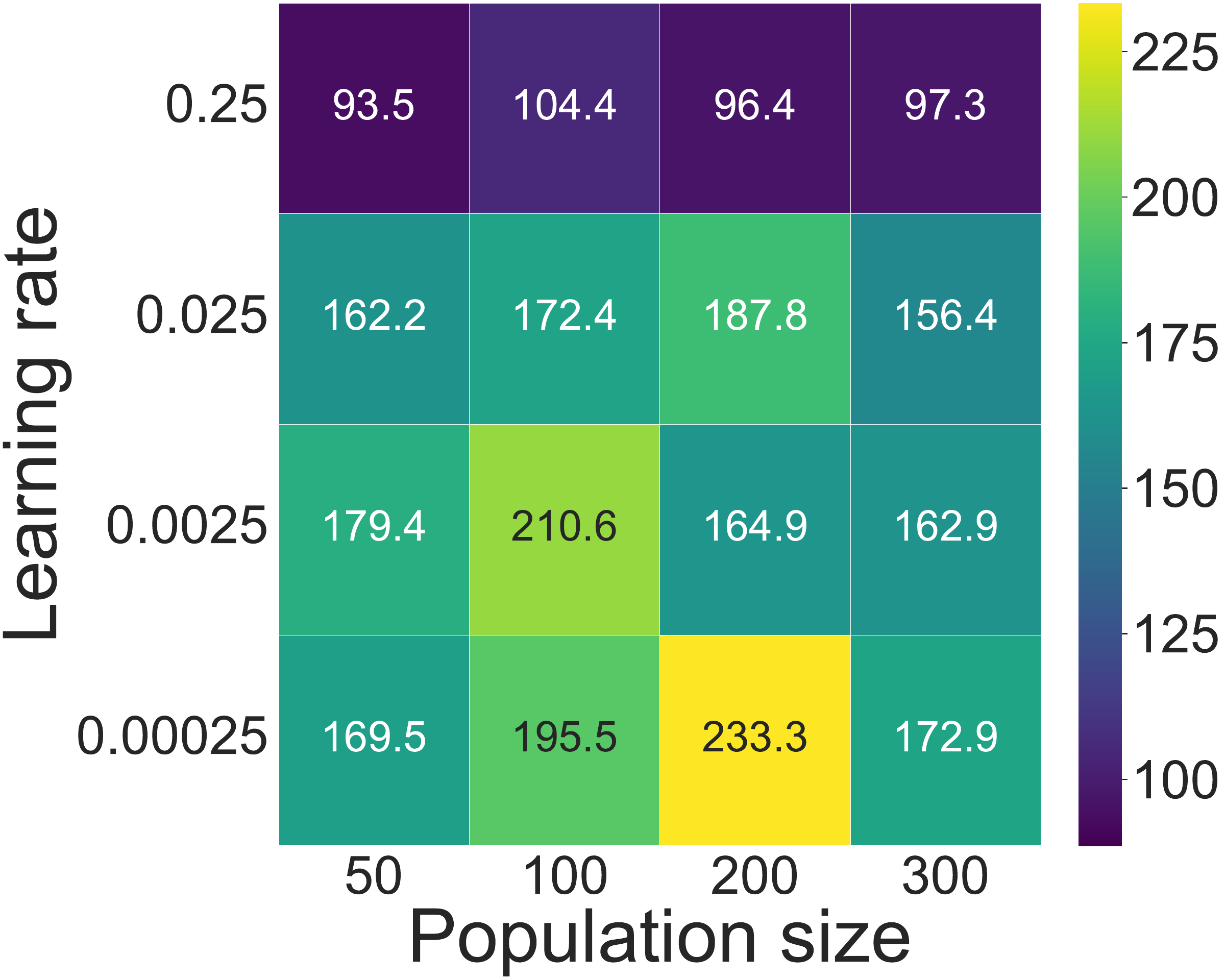}
    \end{subfigure}

    \vspace{0.5em}

    \begin{subfigure}[t]{0.38\linewidth}
        \centering
        \subcaption{Oja}
        \includegraphics[width=\linewidth]{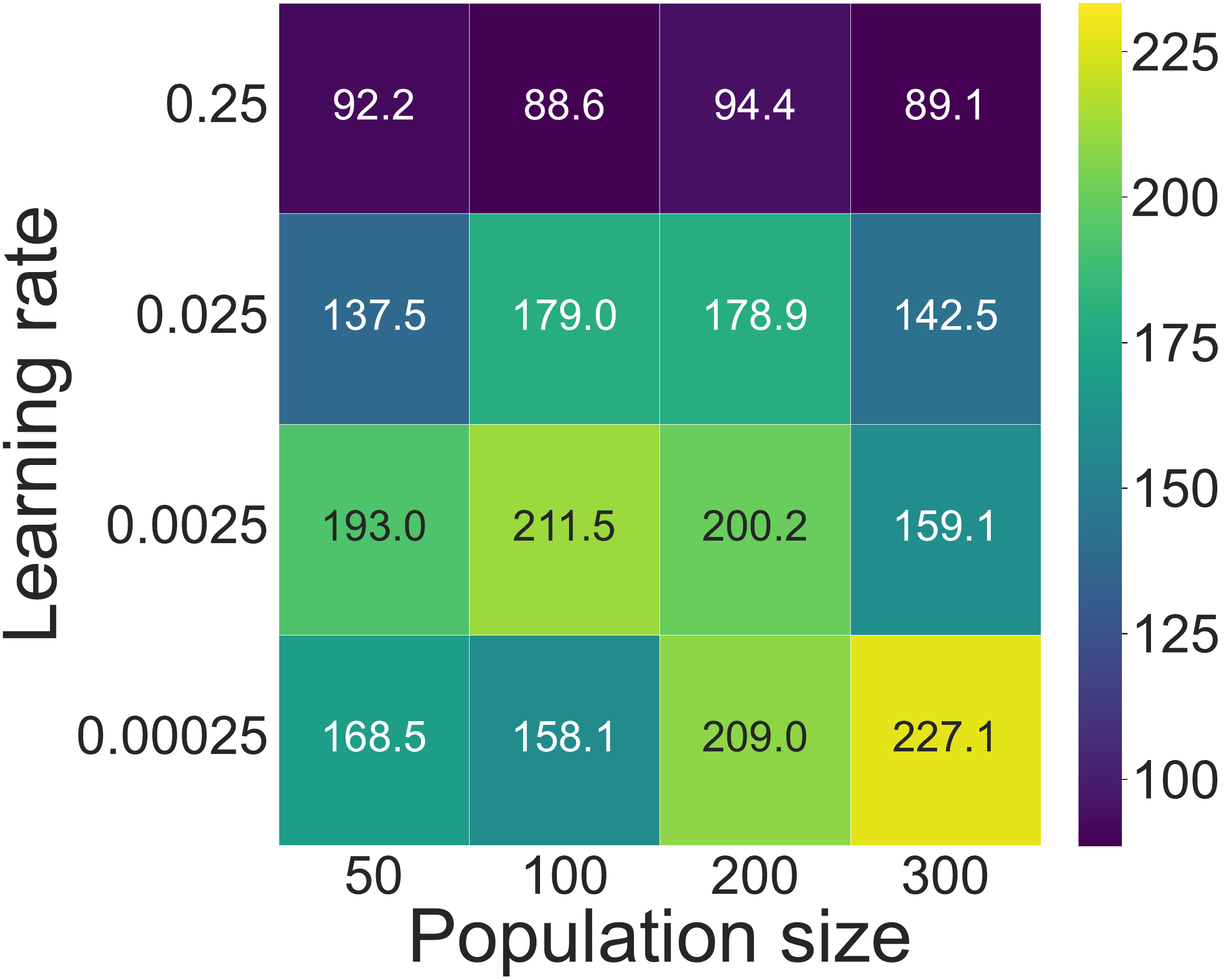}
    \end{subfigure}
    \hfill
    \begin{subfigure}[t]{0.38\linewidth}
        \centering
        \subcaption{Standard NEAT}
        \includegraphics[width=\linewidth]{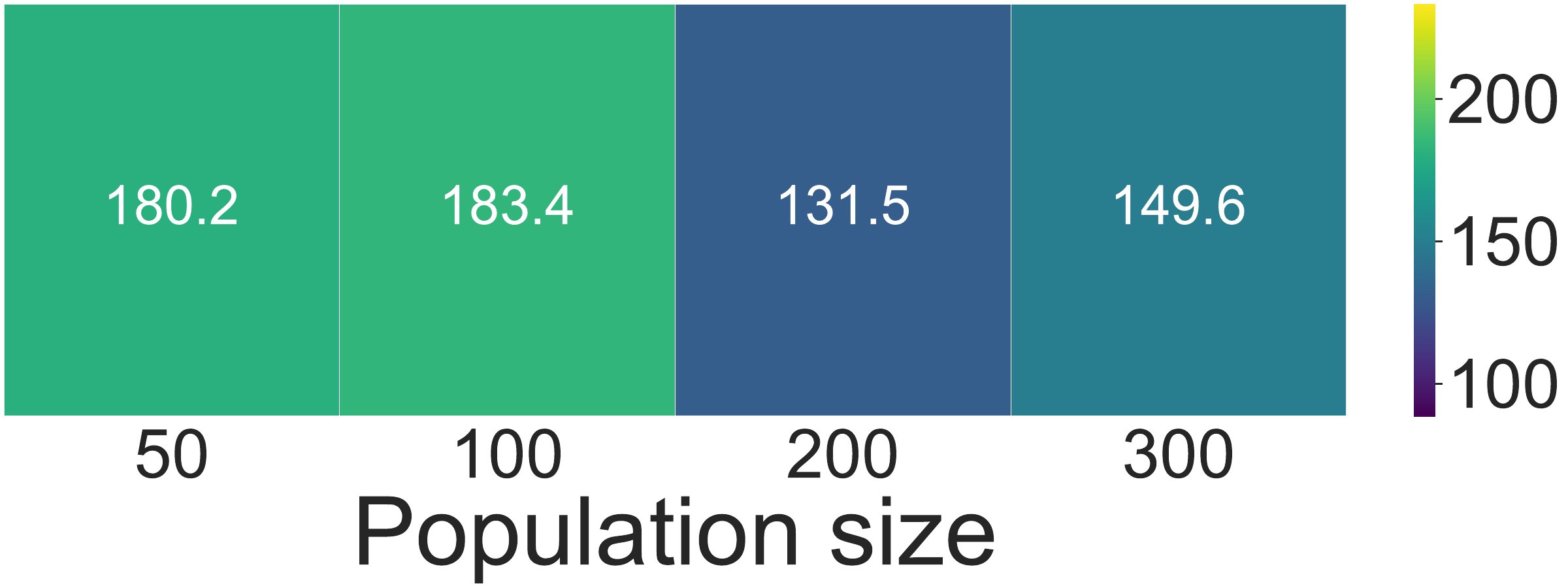}
    \end{subfigure}
    
    \caption{Heatmap comparison of NEOL algorithms (BCM, Hebb, Oja) against standard NEAT in \texttt{BipedalWalker-v3}. 
    Values represent the empirical mean of the final best fitness over $30$ seeds.}
    \label{fig:NEOL_comparison_combined_heatmap_BW}
\end{figure}

\clearpage

\subsection{Reproducibility}
We provide an anonymous GitHub repository at \url{https://anonymous.4open.science/r/NeuroEvolution_Online_Learning_NEOL-41F7/Tasks/BW/ojaNEATRL.py}, containing all source code and environment specifications required to run our experiments from scratch. 
To preserve double blind anonymity, we do not include raw configuration files in the submission. 
Instead, all model/algorithm settings and training protocols (including NEOL components, optimiser choices, schedules, population size, mutation/crossover rates, selection pressure, etc.) are fully enumerated in the paper, see Appendix. 
Readers can launch from scratch with an arbitrary random seed and reproduce our tables/figures within expected stochastic variation. 

\noindent \textbf{Computational Consideration:} We provide sufficient information on the computer resources used for each experiment, specifically describing the use of a general-purpose computing cluster with Ice Lake and Cascade Lake nodes. 
For each job, it specifies the type of compute worker (e.g., every single CPU with 60+cores and 100+G memory), and the execution time (up to 908.73 hours).
\section{More Related Works}

\subsection{NEAT Family Work.}
Existing hybrids explore different aspects of this idea. 
HyperNEAT~\cite{stanley2009hypercube} exploits geometric regularities by indirectly specifying connectivity, shifting part of the optimisation burden from dimensionality to problem structure~\cite{stanley2009hypercube}. 
However, subsequent studies report that HyperNEAT can underperform on some large-scale problems (e.g., Atari) and can struggle when the state--action mapping is highly discontinuous (the ``fracture'' issue)~\cite{hausknecht2014neuroevolution,kohl2009evolving}. 
Other works integrate topology evolution with value-based or policy-gradient RL. 
For example, \cite{whiteson2005automatic} combined Q-learning with evolving topologies but evaluated on a single control domain, leaving scalability unclear. 
\cite{Peng2018neatpg} proposed NEAT with Policy Gradient Search (NEAT-PGS), where RL is typically used to pre-train policy networks, NEAT evolves a feature network, and the policy is further trained given the evolved features.

\subsection{Difference Between Prior Plasticity Frameworks}
It is crucial to distinguish NEOL from two prominent predecessors in plastic neuroevolution to clarify our novelty:
\subsubsection{AdaptiveNEAT} \cite{Stanley2003AdaptiveSynapses} concluded that evolving fixed recurrent networks is more effective than evolving local plasticity rules. However, their study relied on two-factor Hebbian rules (Pre $\times$ Post), which are purely correlation-based and unsupervised. In contrast, NEOL employs three-factor plasticity (Pre $\times$ Post $\times$ Reward). By gating the synaptic update with the environment's reward signal, we transform the inner loop from chaotic self-organisation into directed online reinforcement learning.

\subsubsection{Metalearning-through-Hebbian-in-random-networks:} \cite{Najarro2020HebbianMetaLearning} utilised Evolution Strategies (ES) to meta-learn complex plasticity coefficients for fixed, randomly initialised networks. Their approach assumes the architecture is secondary to the learning rule. NEOL tests the inverse hypothesis: we fix the plasticity rules to standard, biologically grounded formulations (Oja, BCM) and evolve the topology via NEAT. Our results suggest that structural evolution can discover architectures that naturally exploit simple learning rules, removing the need to meta-learn the rules themselves.
\subsubsection{Differentiable Plasticity (Backpropamine):} \cite{Miconi2018DifferentiablePlasticity}
Modern meta-learning approaches use gradient descent to optimise plasticity coefficients (e.g., Backpropamine). While powerful, these methods require differentiable reward signals and typically operate on fixed architectures. NEOL is distinct in being gradient-free and focusing on \emph{topological evolution}. We demonstrate that evolving the architecture to exploit fixed, standard plasticity rules is a viable alternative to meta-learning parameters via backpropagation.

\subsubsection{Evolved Neuromodulation:} 
Prior work in \cite{Soltoggio2008NeuromodulatedPlasticity} evolves complex, dedicated modulatory neurons to regulate plasticity locally. NEOL simplifies this by using the global environmental reward directly as the neuromodulatory signal (Three-Factor Rule). Our results show that this simpler, "global-to-local" signal is sufficient to achieve state-of-the-art performance in continuous control without the overhead of evolving complex modulatory circuitry.

\end{document}